\documentclass[twoside,11pt,preprint]{article}

\usepackage{jmlr2e}
\usepackage{lastpage}
\usepackage[T1]{fontenc}
\usepackage{microtype}
\usepackage{amsmath}
\usepackage{tikz}
\usepackage{algorithm}
\usepackage{algorithmic}
\usepackage{tabularx}
\usepackage{booktabs}
\usepackage{subcaption}
\usepackage{xcolor}

\definecolor{npg1}{HTML}{E64B35}
\definecolor{npg2}{HTML}{00A087}

\newtheorem{assumption}{Assumption}

\usetikzlibrary{arrows.meta,positioning,fit,backgrounds,calc}

\def\ndot{n_{\scriptscriptstyle\bullet}}

\hypersetup{colorlinks=true,citecolor=blue,urlcolor=blue}

\jmlrheading{23}{2026}{1-\pageref{LastPage}}{1/21; Revised 5/22}{9/22}{21-0000}{Ryan Thompson, Matt P. Wand, and Veerabhadran Baladandayuthapani}

\ShortHeadings{Structure Learning on Clustered Data}{Thompson, Wand, and Baladandayuthapani}
\firstpageno{1}

\begin{document}

\title{Structure Learning on Clustered Data}

\author{%
\name Ryan Thompson \email ryan.thompson-1@uts.edu.au \\
\addr School of Mathematical and Physical Sciences \\
University of Technology Sydney \\
Ultimo, NSW 2007, Australia
\AND
\name Matt P. Wand \email matt.wand@uts.edu.au \\
\addr School of Mathematical and Physical Sciences \\
University of Technology Sydney \\
Ultimo, NSW 2007, Australia
\AND
\name Veerabhadran Baladandayuthapani \email veerab@umich.edu \\
\addr Department of Biostatistics \\
University of Michigan \\
Ann Arbor, MI 48105, USA
}

\editor{My editor}

\maketitle

\begin{abstract}
Recent algorithmic advances have made directed acyclic graph (DAG) structure learning scalable for causal discovery. Yet, the currently available techniques assume a completely homogeneous population, precluding their application to clustered data where cluster-specific variations (e.g., patient-specific effects) are common. We address this issue by introducing a new approach that estimates a global structure while accounting for local cluster-level effects. The key idea is to extend the fixed- and random-effects framework of classical mixed models to the structure learning setting. Towards this end, we present a differentiable graph coupling mechanism that guarantees the union of the fixed- and random-effects graphs remains acyclic. Computationally, we provide a provably convergent first-order method and leverage efficient batched updates across clusters. Statistically, we establish identifiability of the model and show that our approach recovers the true structure asymptotically. In experiments on real and synthetic data, our proposal detects dependencies missed by alternative estimators, underscoring its value for structure learning in clustered settings.
\end{abstract}


\section{Introduction}
\label{sec:introduction}

Directed acyclic graphs (DAGs)---graphs with directed edges and no cycles---are a fundamental tool for probabilistic modeling and causal inference \citep{Spirtes2000,Pearl2009}. Their capacity to compactly encode conditional independence relations makes them useful in a wide range of domains, e.g., psychology \citep{Foster2010}, economics \citep{Imbens2020}, and epidemiology \citep{Tennant2021}. The first formal treatments in statistics and machine learning appeared in the late 1980s \citep{Lauritzen1988,Pearl1988}. Three decades later, \citet{Zheng2018} made a pivotal advance that transformed structure learning from a notoriously difficult combinatorial optimization problem to a tractable continuous optimization problem, opening the way to scalably learn graphs with hundreds of nodes using first-order algorithms. Surveys by \citet{Vowels2022} and \citet{Kitson2023} provide comprehensive overviews of this shift and its impact on modern structure learning.

\begin{figure}[t]
\centering
\begin{tikzpicture}[
  >=Latex,
  var/.style={circle, draw, inner sep=1pt, minimum size=9pt},
  dagedge/.style={-Latex, line width=0.6pt},
  Bedge/.style={dagedge},
  UedgeOne/.style={dagedge, loosely dashed, draw=npg1},
  WedgeOne/.style={dagedge, dashed, draw=npg1},
  UedgeThr/.style={dagedge, dashed, draw=npg2},
  WedgeThr/.style={dagedge, dashed, draw=npg2},
  WedgeMixOne/.style={dagedge, dotted, draw=npg1},
  WedgeMixThr/.style={dagedge, dotted, draw=npg2},
  symbol/.style={font=\Large, inner sep=1pt},
  graphtitle/.style={font=\small, anchor=east, align=center},
  graphbox/.style={draw=black!20, rounded corners=2pt, inner sep=3pt}
]

\def\xB{0}
\def\xPlus{2}
\def\xU{5.5}
\def\xEq{8}
\def\xW{11}

\def\yTop{1.4}
\def\yBot{-1.4}
\def\yEll{0.1} 

\newcommand{\drawDAG}[4]{%
  \node[var] (#11) at (#2-0.65, #3+0.55) {\scriptsize $X_1$};
  \node[var] (#12) at (#2+0.65, #3+0.55) {\scriptsize $X_2$};
  \node[var] (#13) at (#2-0.65, #3-0.55) {\scriptsize $X_3$};
  \node[var] (#14) at (#2+0.65, #3-0.55) {\scriptsize $X_4$};
  \node[inner sep=0pt, fit=(#11)(#12)(#13)(#14)] (#1box) {};
  \node[graphtitle] at ($(#1box.west)+(-2mm,0)$) {#4};
}

\drawDAG{B}{\xB}{0}{$\mathbf{B}$}
\draw[Bedge] (B1) -- (B2);
\draw[Bedge] (B1) -- (B3);
\draw[Bedge] (B2) -- (B4);
\draw[Bedge] (B3) -- (B4);

\node[symbol] at (\xPlus,\yTop) {+};
\drawDAG{Uone}{\xU}{\yTop}{$\mathbf{U}^{(1)}$}
\draw[UedgeOne] (Uone1) -- (Uone2);
\draw[UedgeOne] (Uone2) -- (Uone3);
\node[symbol] at (\xEq,\yTop) {=};
\drawDAG{Wone}{\xW}{\yTop}{$\mathbf{W}^{(1)}$}
\draw[dagedge] (Wone1) -- (Wone3);
\draw[dagedge] (Wone2) -- (Wone4);
\draw[dagedge] (Wone3) -- (Wone4);
\draw[WedgeMixOne] (Wone1) -- (Wone2);
\draw[WedgeOne]    (Wone2) -- (Wone3);

\node at (\xU,\yEll) {$\vdots$};
\node at (\xPlus,\yEll) {$\vdots$};
\node at (\xW,\yEll) {$\vdots$};
\node at (\xEq,\yEll) {$\vdots$};

\node[symbol] at (\xPlus,\yBot) {+};
\drawDAG{Uthr}{\xU}{\yBot}{$\mathbf{U}^{(m)}$}
\draw[UedgeThr] (Uthr1) -- (Uthr2);
\draw[UedgeThr] (Uthr2) -- (Uthr3);
\node[symbol] at (\xEq,\yBot) {=};
\drawDAG{Wthr}{\xW}{\yBot}{$\mathbf{W}^{(m)}$}
\draw[dagedge] (Wthr1) -- (Wthr3);
\draw[dagedge] (Wthr2) -- (Wthr4);
\draw[dagedge] (Wthr3) -- (Wthr4);
\draw[WedgeMixThr] (Wthr1) -- (Wthr2);
\draw[WedgeThr]    (Wthr2) -- (Wthr3);

\begin{scope}[on background layer]
  \node[graphbox, fit=(Bbox)] {};
  \node[graphbox, fit=(Uonebox)] {};
  \node[graphbox, fit=(Uthrbox)] {};
  \node[graphbox, fit=(Wonebox)] {};
  \node[graphbox, fit=(Wthrbox)] {};
\end{scope}

\end{tikzpicture}
\caption{Schematic of mixed DAGs. The fixed-effects graph $\mathbf{B}$, shown with solid black arrows, is added to cluster-specific random-effect deviations $\mathbf{U}^{(i)}$, shown with colored dashed arrows, to form $\mathbf{W}^{(i)}=\mathbf{B}+\mathbf{U}^{(i)}$. Dotted colored arrows indicate edges with both fixed- and random-effect contributions. Different colors correspond to different clusters.}
\label{fig:illustration}
\end{figure}
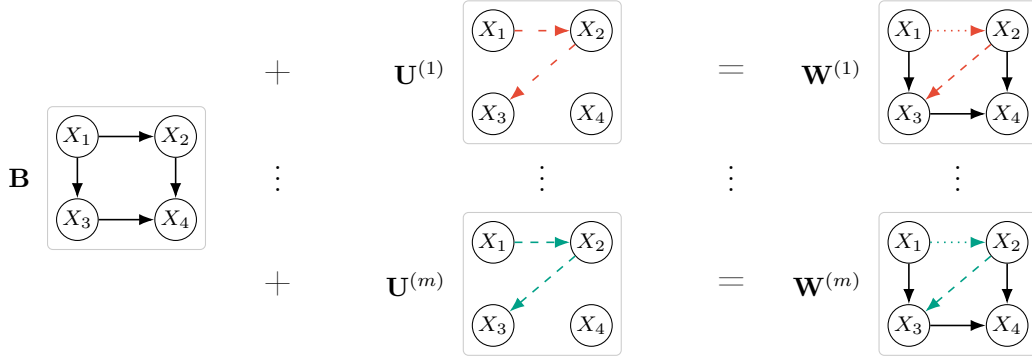

While recent advances have made structure learning scalable, most methods are confined to the iid setting, where the population is assumed to follow a single homogeneous causal model. This assumption is convenient but fails when causal effects vary within the population. Several strands of work have begun to relax homogeneity, including mixtures of DAGs \citep{Saeed2020}, DAGs adapted to distributional shifts \citep{Huang2020}, and DAGs that vary with modifying covariates \citep{Thompson2024}. However, one of the most typical sources of heterogeneity addressed by non-graphical models, clustered data, remains largely unexplored for DAGs. In clustered data, observations come with known cluster memberships, and while clusters may share a common causal structure, the strength or polarity of effects can differ across them. For example, in proteomic studies, repeated measures on the same patients form clusters. These clusters may exhibit protein interactions that broadly align across the population yet have effects that differ in detail across individuals or groups.

Drawing on the machinery of mixed-effects models \citep[e.g.,][]{Jiang2021}, we develop a scalable framework for learning DAGs on clustered data. Formally, let $\mathbf{X}^{(i)}\in\mathbb{R}^{n_i\times p}$ be $n_i$ observations on $p$ variables belonging to cluster $i=1,\dots,m$. A DAG on these variables can be represented by a \emph{sparse} matrix $\mathbf{W}^{(i)}\in\mathbb{R}^{p\times p}$, known as a weighted adjacency matrix, with $w_{jk}^{(i)}\neq 0$ if and only if a directed edge exists from node $j$ to node $k$. Hence, learning the DAG reduces to estimating an acyclic $\mathbf{W}^{(i)}$, which we decompose into population-level and cluster-specific components. This decomposition leads to a structural equation model:
\begin{equation}
\label{eq:sem}
\mathbf{X}^{(i)}=\mathbf{X}^{(i)}\mathbf{W}^{(i)}+\boldsymbol{\varepsilon}^{(i)},\quad\mathbf{W}^{(i)}=\mathbf{B}+\mathbf{U}^{(i)},\quad i=1,\dots,m,
\end{equation}
where $\boldsymbol{\varepsilon}^{(i)}\in\mathbb{R}^{n_i\times p}$ is stochastic noise. Here, the matrix $\mathbf{B}\in\mathbb{R}^{p\times p}$ represents \emph{fixed effects}, traditionally understood in linear mixed models as coefficients common to all clusters. Meanwhile, the matrices $\mathbf{U}^{(i)}\in\mathbb{R}^{p\times p}$ contain \emph{random effects}, which capture cluster-specific deviations. In analogy with linear mixed models, the random effects associated with node $k$ may be taken to follow $\mathbf{u}_k^{(i)}\sim\mathrm{N}(\mathbf{0},\operatorname{diag}(\boldsymbol{\gamma}_k))$, where $\boldsymbol{\gamma}_k$ is the $k$th column of a matrix $\boldsymbol{\Gamma}\in\mathbb{R}_+^{p\times p}$ collecting all random-effect variances. As Figure~\ref{fig:illustration} illustrates, unlike standard (purely fixed-effects) DAGs, where a single $\mathbf{W}$ applies uniformly across the population, mixed DAGs allow specific effects to vary with cluster membership under a global structure.

Under model \eqref{eq:sem}, the overall DAG is determined by the union of two sparse matrices: the fixed-effects matrix $\mathbf{B}$, whose nonzero entries capture population-level effects, and the random-effects variance matrix $\boldsymbol{\Gamma}$, whose nonzero entries indicate which cluster-specific deviations are active across clusters. Learning the model therefore amounts to estimating $\mathbf{B}$ and $\boldsymbol{\Gamma}$ subject to the constraint that their union remains acyclic. We enforce this constraint through a graph coupling mechanism based on existing differentiable characterizations of acyclicity, guaranteeing that the combined fixed- and random-effects graphs are cycle-free. On the computational side, we develop a first-order optimization method with provable convergence and leverage efficient batched updates across clusters, making it feasible to learn graphs with up to hundreds of nodes. We further derive statistical results establishing identifiability of the model and correct recovery of the true structure asymptotically. Synthetic experiments and an application to proteomics confirm that our proposed approach recovers both population-level and cluster-specific effects that alternative approaches fail to detect.

The remainder of the paper is organized as follows. Section~\ref{sec:related} situates our proposal in the context of the broader literature. Section~\ref{sec:mixed} introduces our framework for learning mixed DAGs. Section~\ref{sec:scalable} describes algorithms and strategies for scalable computation. Section~\ref{sec:statistical} establishes an asymptotic statistical guarantee. Section~\ref{sec:experiments} reports experiments on synthetic, semi-synthetic, and real data. Section~\ref{sec:final} closes the paper with some final remarks.

\section{Related Work}
\label{sec:related}

The breakthrough work by \citet{Zheng2018} reformulates the combinatorial acyclicity constraint as a continuous, differentiable constraint by exploiting the correspondence between cycles in a graph and traces of matrix powers. Subsequent work by \citet{Zhang2022} and \citet{Bello2022}, among others, develops alternative continuous acyclicity characterizations with improved numerical behavior, including polynomial and log-determinant formulations. \citet{Zhang2025} provide a unified analytic perspective of these characterizations. \citet{Zheng2020} devise extensions to non-parametric DAGs, \citet{Thompson2025a} extend the framework to variational inference, while \citet{Deng2023} and \citet{Deng2024} establish theoretical guarantees. While these approaches are nonconvex, \citet{Deng2023a} show how discrete combinatorics can be used to escape local minima if desired.

Though the above differentiable approaches assume iid observations, a few classical methods can handle clustered data. \citet{Bae2016} incorporate node-level random intercepts into DAGs, capable of capturing limited forms of cluster-level variation, while \citet{Li2018} allow edge strengths to consist of both fixed and random components. Both, however, rely on traditional discrete search and hence remain fundamentally combinatorial. More recent work considers DAGs that vary as a function of observed covariates rather than across clusters: \citet{Ni2019} use spline-based mapping and \citet{Thompson2024} employ a neural mapping with differentiable acyclicity. \citet{Oates2016} take a different route, jointly estimating observation-specific DAGs that are linked via a dependency network. These methods capture certain forms of heterogeneity, but none are designed specifically for clustered data.

Our paper also contributes to the growing literature on mixed-effects models adapted for machine learning. \citet{Xiong2019} and \citet{Simchoni2021,Simchoni2023,Simchoni2025} integrate random effects into neural networks---including feedforward, convolutional, and recurrent variants---to handle clustered data, with high-cardinality categorical features as an important special case. \citet{Simchoni2024} extend this perspective to variational autoencoders, introducing random effects directly into the latent representation. \citet{Sholokhov2024} and \citet{Thompson2025} develop algorithms for learning linear mixed models on high-dimensional data, enabling the selection of nonzero fixed and random effects across thousands of candidate predictors. These works, however, are confined to non-graphical models and so do not address multivariate relationships, let alone DAG structure learning.

\section{Mixed DAGs}
\label{sec:mixed}

\subsection{Objective Function}

To formulate our estimator, we begin by narrowing our focus to a single node $k$ within a single cluster $i$. The mixed-effects structural equation model \eqref{eq:sem} at this level takes the form
\begin{equation*}
\mathbf{x}_k^{(i)}=\mathbf{X}^{(i)}\boldsymbol{\beta}_k+\mathbf{X}^{(i)}\mathbf{u}_k^{(i)}+\boldsymbol{\varepsilon}_k^{(i)},
\end{equation*}
where $\mathbf{x}_k^{(i)}\in\mathbb{R}^{n_i}$ denotes the $k$th column of $\mathbf{X}^{(i)}$, the vectors $\boldsymbol{\beta}_k\in\mathbb{R}^p$ and $\mathbf{u}_k^{(i)}\in\mathbb{R}^p$ represent the $k$th columns of the fixed-effects matrix $\mathbf{B}$ and random-effects matrix $\mathbf{U}^{(i)}$, respectively, and $\boldsymbol{\varepsilon}_k^{(i)}\in\mathbb{R}^{n_i}$ represents the $k$th column of the noise matrix $\boldsymbol{\varepsilon}^{(i)}$. It is common in the linear mixed model literature to assume Gaussian noise and random effects, so that $\boldsymbol{\varepsilon}_k^{(i)}\sim\mathrm{N}(\mathbf{0},\mathbf{I})$ and $\mathbf{u}_k^{(i)}\sim\mathrm{N}(\mathbf{0},\operatorname{diag}(\boldsymbol{\gamma}_k))$, where $\boldsymbol{\gamma}_k\in\mathbb{R}_+^p$ is the $k$th column of $\boldsymbol{\Gamma}$. More generally, node-specific noise variances can be accommodated by taking $\boldsymbol{\varepsilon}_k^{(i)}\sim\mathrm{N}(\mathbf{0},\sigma_k^2\mathbf{I})$ and replacing $\mathbf{I}$ below by $\sigma_k^2\mathbf{I}$, while here we set $\sigma_k^2=1$ to simplify notation. The diagonal covariance for the random effects is a common simplification in high-dimensional settings, where allowing unrestricted covariance is impractical. Throughout, we impose $\operatorname{diag}(\mathbf{B})=\operatorname{diag}(\boldsymbol{\Gamma})=\mathbf{0}$, excluding self-loops. Conditional on the parent variables of node $k$, the vector $\mathbf{x}_k^{(i)}$ follows
\begin{equation*}
\mathbf{x}_k^{(i)}\sim\mathrm{N}(\mathbf{X}^{(i)}\boldsymbol{\beta}_k,\mathbf{V}^{(i)}(\boldsymbol{\gamma}_k)),
\end{equation*}
where the matrix $\mathbf{V}^{(i)}(\boldsymbol{\gamma}_k)$ captures correlations between observations within the $i$th cluster:
\begin{equation*}
\mathbf{V}^{(i)}(\boldsymbol{\gamma}_k):=\mathbf{I}+\mathbf{X}^{(i)}\operatorname{diag}(\boldsymbol{\gamma}_k)\mathbf{X}^{(i)\top}.
\end{equation*}

We now move from this single-node description to the full model by applying the same construction across nodes $k=1,\dots,p$ and clusters $i=1,\dots,m$. On the feasible DAG set, after integrating out the random effects, the resulting nodewise conditional densities factorize over nodes and independent clusters, yielding a negative log-likelihood of the form
\begin{equation}
\label{eq:likelihood}
l(\mathbf{B},\boldsymbol{\Gamma})=\sum_{k=1}^p\sum_{i=1}^m\left\{\log\det\{\mathbf{V}^{(i)}(\boldsymbol{\gamma}_k)\}+(\mathbf{x}_k^{(i)}-\mathbf{X}^{(i)}\boldsymbol{\beta}_k)^\top\mathbf{V}^{-(i)}(\boldsymbol{\gamma}_k)(\mathbf{x}_k^{(i)}-\mathbf{X}^{(i)}\boldsymbol{\beta}_k)\right\},
\end{equation}
where $\mathbf{V}^{-(i)}(\boldsymbol{\gamma}_k):=\{\mathbf{V}^{(i)}(\boldsymbol{\gamma}_k)\}^{-1}$, and additive and multiplicative constants are omitted to simplify exposition. The negative log-likelihood \eqref{eq:likelihood} has explicit gradients with respect to the fixed-effects $\mathbf{B}$ and the random-effects variances $\boldsymbol{\Gamma}$. For the fixed-effect $\beta_{jk}$, the gradient is
\begin{equation*}
\frac{\partial}{\partial\beta_{jk}}l(\mathbf{B},\boldsymbol{\Gamma})=-2\sum_{i=1}^m\mathbf{x}_j^{(i)\top}\mathbf{V}^{-(i)}(\boldsymbol{\gamma}_k)(\mathbf{x}_k^{(i)}-\mathbf{X}^{(i)}\boldsymbol{\beta}_k).
\end{equation*}
Likewise, the gradient with respect to the random-effect variance $\gamma_{jk}$ is of the form
\begin{equation*}
\frac{\partial}{\partial\gamma_{jk}}l(\mathbf{B},\boldsymbol{\Gamma})=\sum_{i=1}^m\left[\mathbf{x}_j^{(i)\top}\mathbf{V}^{-(i)}(\boldsymbol{\gamma}_k)\mathbf{x}_j^{(i)}-\left\{\mathbf{x}_j^{(i)\top}\mathbf{V}^{-(i)}(\boldsymbol{\gamma}_k)(\mathbf{x}_k^{(i)}-\mathbf{X}^{(i)}\boldsymbol{\beta}_k)\right\}^2\right].
\end{equation*}
These gradient expressions are useful for the algorithmic developments that follow later.

We take our estimator as a minimizer of the negative log-likelihood \eqref{eq:likelihood} over the fixed- and random-effects parameters while enforcing two key structural constraints. First, the graphs $\mathcal{G}(\mathbf{B})$ and $\mathcal{G}(\boldsymbol{\Gamma})$ induced by $\mathbf{B}$ and $\boldsymbol{\Gamma}$ must jointly remain acyclic. Here, $\mathcal{G}(\mathbf{W})$ denotes the directed graph on $\{1,\dots,p\}$ having an edge $j\to k$ if and only if $w_{jk}\neq 0$. In other words, the graph $\mathcal{G}(\mathbf{B})\cup\mathcal{G}(\boldsymbol{\Gamma})$ formed by the union of the individual edge sets must be a DAG. Second, sparsity is imposed to encourage parsimonious graphs with fewer edges and improve structure recovery. Together, these requirements lead to the constrained optimization problem
\begin{equation}
\label{eq:optcomb}
\underset{\mathbf{B}\in\mathbb{R}^{p\times p},\,\boldsymbol{\Gamma}\in\mathbb{R}_+^{p\times p}}{\min}\;l(\mathbf{B},\boldsymbol{\Gamma})+\lambda_1\|\mathbf{B}\|_1+\lambda_2\|\boldsymbol{\Gamma}\|_1\quad\operatorname{s.t.}\;\mathcal{G}(\mathbf{B})\cup\mathcal{G}(\boldsymbol{\Gamma})\in\mathrm{DAG}_p,
\end{equation}
where $\mathrm{DAG}_p$ is the set of all DAGs on nodes $\{1,\dots,p\}$. Here, $\lambda_1,\lambda_2\geq0$ are sparsity penalty parameters, and $\|\cdot\|_1$ is the sum of the absolute values of the entries of its matrix argument. The optimization problem \eqref{eq:optcomb} is nonconvex, owing both to the form of the negative log-likelihood and to the DAG constraint. We denote its minimizers by $\hat{\mathbf{B}}$ and $\hat{\boldsymbol{\Gamma}}$.

To report cluster-specific graphs, we recover the random-effects $\mathbf{U}^{(i)}$ via their so-called best linear unbiased predictors (BLUPs) $\hat{\mathbf{U}}^{(i)}$. The BLUP for the $k$th column $\mathbf{u}_k^{(i)}$ of $\mathbf{U}^{(i)}$ is
\begin{equation*}
\hat{\mathbf{u}}_k^{(i)}=\operatorname{diag}(\hat{\boldsymbol{\gamma}}_k)\mathbf{X}^{(i)\top}\mathbf{V}^{-(i)}(\hat{\boldsymbol{\gamma}}_k)(\mathbf{x}_k^{(i)}-\mathbf{X}^{(i)}\hat{\boldsymbol{\beta}}_k).
\end{equation*}
Combining $\hat{\mathbf{U}}^{(i)}$ with the estimated fixed-effects $\hat{\mathbf{B}}$ produces a DAG for cluster $i$ as
\begin{equation*}
\hat{\mathbf{W}}^{(i)}=\hat{\mathbf{B}}+\hat{\mathbf{U}}^{(i)}.
\end{equation*}

\subsection{Differentiable Reformulation}

As written, the DAG constraint in \eqref{eq:optcomb} constitutes a non-differentiable combinatorial restriction that precludes the use of scalable first-order optimization methods like gradient descent. Following the differentiable structure learning literature \citep[e.g.,][]{Zheng2018,Bello2022}, we reformulate the combinatorial DAG constraint as an \emph{equivalent} differentiable DAG constraint. This line of work established the existence of smooth functions $h(\mathbf{W})$ such that
\begin{equation*}
h(\mathbf{W})=0\iff\mathcal{G}(\mathbf{W})\in\mathrm{DAG}_p.
\end{equation*}
In other words, the differentiable function $h(\mathbf{W})$ vanishes exactly on the set of DAGs, so enforcing $h(\mathbf{W})=0$ is lossless relative to the original DAG constraint. We adopt the log-determinant form of $h$ first proposed by \citet{Bello2022}. For matrices $\mathbf{W}\in\mathbb{W}:=\{\mathbf{W}\in\mathbb{R}_+^{p\times p}:\rho(\mathbf{W})<1\}$, where $\rho(\mathbf{W})$ denotes the spectral radius of $\mathbf{W}$, define the function
\begin{equation*}
h(\mathbf{W}):=-\log\det(\mathbf{I}-\mathbf{W}).
\end{equation*}
Under the stated nonnegativity and spectral radius conditions, \citet{Bello2022} showed that $h(\mathbf{W})=0$ if and only if $\mathbf{W}$ is acyclic. In the real-valued case, $\mathbf{W}$ can be mapped to a nonnegative surrogate via the Hadamard product $\mathbf{W}\odot\mathbf{W}$ before using $h$, preserving the zeros and hence the graph. Crucially, $h$ is continuously differentiable on $\mathbb{W}$ with gradient
\begin{equation*}
\nabla h(\mathbf{W})=(\mathbf{I}-\mathbf{W})^{-\top}.
\end{equation*}

Our setting requires a stronger constraint than above: the \emph{union} of the graphs induced by the fixed-effects matrix $\mathbf{B}$ and random-effects matrix $\boldsymbol{\Gamma}$ must be acyclic. It is therefore insufficient for $\mathbf{B}$ and $\boldsymbol{\Gamma}$ to be individually acyclic, as their union can still contain cycles, leading to cluster-specific matrices $\hat{\mathbf{W}}^{(i)}$ that are not valid DAGs. Since $\boldsymbol{\Gamma}$ has nonnegative entries, while $\mathbf{B}$ has real-valued entries, we enforce acyclicity of the union by applying the log-determinant characterization to the combined nonnegative matrix $\mathbf{B}\odot\mathbf{B}+\boldsymbol{\Gamma}$:
\begin{equation}
\label{eq:acyclicity}
h(\mathbf{B},\boldsymbol{\Gamma}):=-\log\det(\mathbf{I}-\mathbf{B}\odot\mathbf{B}-\boldsymbol{\Gamma}).
\end{equation}
The following proposition formalizes the equivalence between acyclicity of the union graph induced by the fixed and random effects and the level-set of the log-determinant \eqref{eq:acyclicity} at zero.
\begin{proposition}
\label{prop:union_acyclicity}
Let $h(\mathbf{B},\boldsymbol{\Gamma})$ be given in \eqref{eq:acyclicity} and define the domain
\begin{equation}
\label{eq:acyclicitydomain}
\mathbb{D}:=\{(\mathbf{B},\boldsymbol{\Gamma})\in\mathbb{R}^{p\times p}\times\mathbb{R}_+^{p\times p}:\rho(\mathbf{B}\odot\mathbf{B}+\boldsymbol{\Gamma})<1\}.
\end{equation}
Then, for any $(\mathbf{B},\boldsymbol{\Gamma})\in\mathbb{D}$, it holds
\begin{equation*}
\mathcal{G}(\mathbf{B})\cup\mathcal{G}(\boldsymbol{\Gamma})\in\mathrm{DAG}_p\iff h(\mathbf{B},\boldsymbol{\Gamma})=0.
\end{equation*}
\end{proposition}
\begin{proof}
See Appendix~\ref{app:union_acyclicity}.
\end{proof}
Proposition \ref{prop:union_acyclicity} shows that, on the domain $\mathbb{D}$ defined in \eqref{eq:acyclicitydomain}, enforcing $h(\mathbf{B},\boldsymbol{\Gamma})=0$ is exactly equivalent to requiring that the union graph $\mathcal{G}(\mathbf{B})\cup\mathcal{G}(\boldsymbol{\Gamma})$ be acyclic. Consequently, the differentiable constraint in \eqref{eq:acyclicity} is lossless relative to the original combinatorial DAG restriction. The proof carries the argument of \citet{Bello2022} over to the mixed-effects setting.


The gradients of $h(\mathbf{B},\boldsymbol{\Gamma})$ are
\begin{equation*}
\nabla_\mathbf{B} h(\mathbf{B},\boldsymbol{\Gamma})=2(\mathbf{I}-\mathbf{B}\odot\mathbf{B}-\boldsymbol{\Gamma})^{-\top}\odot\mathbf{B}
\end{equation*}
and
\begin{equation*}
\nabla_{\boldsymbol{\Gamma}} h(\mathbf{B},\boldsymbol{\Gamma})=(\mathbf{I}-\mathbf{B}\odot\mathbf{B}-\boldsymbol{\Gamma})^{-\top}.
\end{equation*}
With this characterization, we can rewrite \eqref{eq:optcomb} without the combinatorial DAG constraint:
\begin{equation*}
\underset{(\mathbf{B},\boldsymbol{\Gamma})\in\mathbb{D}}{\min}\;l(\mathbf{B},\boldsymbol{\Gamma})+\lambda_1\|\mathbf{B}\|_1+\lambda_2\|\boldsymbol{\Gamma}\|_1\quad\operatorname{s.t.}\;h(\mathbf{B},\boldsymbol{\Gamma})=0.
\end{equation*}
As discussed in the next section, it is possible to solve this optimization problem using first-order (gradient) methods, paving the way for scalable mixed-effects structure learning.

\section{Scalable Computation}
\label{sec:scalable}

\subsection{Path-Following Algorithm}

Standard first-order algorithms do not readily accommodate hard constraints like $h(\mathbf{B},\boldsymbol{\Gamma})=0$ directly. We therefore adopt a path-following strategy in the spirit of \citet{Bello2022}. The key idea is to relax the constraint by incorporating it into the objective via a smooth barrier:
\begin{equation}
\label{eq:inner}
\underset{(\mathbf{B},\boldsymbol{\Gamma})\in\mathbb{D}}{\min}\;f_\mu(\mathbf{B},\boldsymbol{\Gamma};\lambda_1,\lambda_2):=\mu\left\{l(\mathbf{B},\boldsymbol{\Gamma})+\lambda_1\|\mathbf{B}\|_1+\lambda_2\|\boldsymbol{\Gamma}\|_1\right\}+h(\mathbf{B},\boldsymbol{\Gamma}).
\end{equation}
Here, $\mu\geq0$ is a continuation parameter that we decrease along a sequence $\mu^{(1)}>\mu^{(2)}>\cdots>\mu^{(T)}\geq0$. At iteration $t+1$ of our strategy, we solve \eqref{eq:inner} with $\mu=\mu^{(t+1)}$ using the solution obtained from $\mu=\mu^{(t)}$ as an initialization point. Along this path, the log-determinant term $h(\mathbf{B},\boldsymbol{\Gamma})$ acts as a smooth barrier over $\mathbb{D}$, keeping $\rho(\mathbf{B}\odot\mathbf{B}+\boldsymbol{\Gamma})<1$ and progressively steering iterates toward an acyclic union as $\mu$ decreases. In the limit as $\mu\to0$, the barrier dominates and the resulting solution satisfies the acyclicity constraint $h(\mathbf{B},\boldsymbol{\Gamma})=0$. Algorithm~\ref{alg:path} summarizes the full procedure.
\begin{algorithm}
\caption{Path-following algorithm}
\label{alg:path}
\begin{algorithmic}
\STATE \textbf{Input:} Initializer $(\mathbf{B}^{(0)},\boldsymbol{\Gamma}^{(0)})\in\mathbb{D}$ and parameters $\lambda_1,\lambda_2>0$ and $\mu^{(1)}>\cdots>\mu^{(T)}\geq0$
\FOR{$t=0,1,\dots,T-1$}
\STATE Initialize $(\mathbf{B},\boldsymbol{\Gamma})$ at $(\mathbf{B}^{(t)},\boldsymbol{\Gamma}^{(t)})$ and set
\begin{equation*}
(\mathbf{B}^{(t+1)},\boldsymbol{\Gamma}^{(t+1)})\gets\underset{(\mathbf{B},\boldsymbol{\Gamma})\in\mathbb{D}}{\arg\,\min}\,f_{\mu^{(t+1)}}(\mathbf{B},\boldsymbol{\Gamma};\lambda_1,\lambda_2)
\end{equation*}
\ENDFOR
\STATE \textbf{Output:} Solution $(\hat{\mathbf{B}},\hat{\boldsymbol{\Gamma}})=(\mathbf{B}^{(T)},\boldsymbol{\Gamma}^{(T)})$
\end{algorithmic}
\end{algorithm}
In practice, we set the initial continuation parameter $\mu^{(1)}=1$ and take $\mu^{(t+1)}=0.1\mu^{(t)}$ for $T=5$ iterations. The initialization point $(\mathbf{B}^{(0)},\boldsymbol{\Gamma}^{(0)})$ is set to matrices of zeros, i.e., $(\mathbf{B}^{(0)},\boldsymbol{\Gamma}^{(0)})=(\mathbf{0},\mathbf{0})$, which trivially lies in $\mathbb{D}$.

\subsection{Proximal Gradient Algorithm}

To solve the inner problem \eqref{eq:inner} at each iteration of Algorithm~\ref{alg:path}, we design a proximal gradient algorithm \citep[see, e.g.,][]{Polson2015}. Towards this end, we use the fact that the objective function in \eqref{eq:inner} naturally decomposes into the sum of smooth and nonsmooth components as
\begin{equation*}
f_\mu(\mathbf{B},\boldsymbol{\Gamma}):=g_\mu(\mathbf{B},\boldsymbol{\Gamma})+r_\mu(\mathbf{B},\boldsymbol{\Gamma}),
\end{equation*}
where the smooth component and nonsmooth component are defined respectively as
\begin{equation}
\label{eq:smoothnonsmooth}
g_\mu(\mathbf{B},\boldsymbol{\Gamma}):=\mu\{l(\mathbf{B},\boldsymbol{\Gamma})\}+h(\mathbf{B},\boldsymbol{\Gamma}),\qquad r_\mu(\mathbf{B},\boldsymbol{\Gamma}):=\mu\{\lambda_1\|\mathbf{B}\|_1+\lambda_2\|\boldsymbol{\Gamma}\|_1\}+\iota_{\mathbb{R}_+^{p\times p}}(\boldsymbol{\Gamma}).
\end{equation}
We absorb the nonnegativity constraint on $\boldsymbol{\Gamma}$ into $r_\mu$ via the indicator $\iota_{\mathbb{R}_+^{p\times p}}(\boldsymbol{\Gamma})$, which equals $0$ if $\boldsymbol{\Gamma}\in\mathbb{R}_+^{p\times p}$ and $+\infty$ otherwise. The smooth component $g_\mu$ admits gradients with respect to $\mathbf{B}$ and $\boldsymbol{\Gamma}$, while the nonsmooth term $r_\mu$ is amenable to proximal operators. This structure leads to simple updates: given current iterates $(\mathbf{B}^{(k)},\boldsymbol{\Gamma}^{(k)})$ we take a gradient step with step size $\alpha>0$ on the smooth component and then apply proximal operators to the result, thus handling the nonsmooth penalties. The corresponding elementwise operators are
\begin{equation*}
\operatorname{soft}_\lambda(z):=\operatorname{sign}(z)\max(|z|-\lambda,0),\qquad\operatorname{soft}_\lambda^+(z):=\max(z-\lambda,0).
\end{equation*}
The first operator $\operatorname{soft}_\lambda(z)$ is the well-known proximal operator for the $\ell_1$-norm, applied to induce sparsity in the fixed-effects $\mathbf{B}$. The second operator $\operatorname{soft}_\lambda^+(z)$ is the proximal operator for an $\ell_1$-norm combined with a projection onto the nonnegative reals, used to produce sparse random-effects variances in $\boldsymbol{\Gamma}$. Algorithm~\ref{alg:proximal} summarizes this proximal gradient procedure.
\begin{algorithm}
\caption{Proximal gradient algorithm}
\label{alg:proximal}
\begin{algorithmic}
\STATE \textbf{Input:} Initializer $(\mathbf{B}^{(0)},\boldsymbol{\Gamma}^{(0)})\in\mathbb{D}$, parameters $\lambda_1,\lambda_2>0$ and $\mu\geq0$, and step size $\alpha>0$
\STATE Initialize iterator $k\gets0$
\WHILE{Not converged}
\STATE Perform proximal gradient updates
\begin{equation*}
\mathbf{B}^{(k+1)}\gets\operatorname{soft}_{\alpha\mu\lambda_1}\left(\mathbf{B}^{(k)}-\alpha\nabla_\mathbf{B}g_\mu(\mathbf{B}^{(k)},\boldsymbol{\Gamma}^{(k)})\right)
\end{equation*}
\begin{equation*}
\boldsymbol{\Gamma}^{(k+1)}\gets\operatorname{soft}_{\alpha\mu\lambda_2}^+\left(\boldsymbol{\Gamma}^{(k)}-\alpha\nabla_{\boldsymbol{\Gamma}}g_\mu(\mathbf{B}^{(k)},\boldsymbol{\Gamma}^{(k)})\right)
\end{equation*}
\STATE Increment iterator $k\gets k+1$
\ENDWHILE
\STATE \textbf{Output:} Solution $(\mathbf{B}^\star,\boldsymbol{\Gamma}^\star)=(\mathbf{B}^{(k)},\boldsymbol{\Gamma}^{(k)})$
\end{algorithmic}
\end{algorithm}
In practice, we declare convergence when the relative change in $g_\mu$ falls below $10^{-6}$.

To determine a suitable step size $\alpha$ in Algorithm~\ref{alg:proximal}, we derive a quadratic upper bound on the local change of the smooth component $g_\mu$. This upper bound ensures that updates yield a controlled decrease in $g_\mu$, thereby providing a principled basis for step size selection. The following proposition establishes this bound on any compact subset $\Omega\subset\mathbb{D}$.
\begin{proposition} 
\label{prop:descent}
Let $g_\mu(\mathbf{B},\boldsymbol{\Gamma})$ be the smooth component \eqref{eq:smoothnonsmooth} of the objective function, defined on the domain $\mathbb{D}$ given in \eqref{eq:acyclicitydomain}. Fix any compact set $\Omega\subset\mathbb{D}$. Then there exists a constant $L<\infty$ such that for all $(\mathbf{B},\boldsymbol{\Gamma}),(\mathbf{B}^+,\boldsymbol{\Gamma}^+)\in\Omega$ it holds
\begin{equation*}
g_\mu(\mathbf{B}^+,\boldsymbol{\Gamma}^+)\leq g_\mu(\mathbf{B},\boldsymbol{\Gamma})+\langle\nabla g_\mu(\mathbf{B},\boldsymbol{\Gamma}),(\Delta\mathbf{B},\Delta\boldsymbol{\Gamma})\rangle+\frac{L}{2}\|(\Delta\mathbf{B},\Delta\boldsymbol{\Gamma})\|_F^2,
\end{equation*}
where $\Delta\mathbf{B}:=\mathbf{B}^+-\mathbf{B}$, $\Delta\boldsymbol{\Gamma}:=\boldsymbol{\Gamma}^+-\boldsymbol{\Gamma}$, and $\|(\Delta\mathbf{B},\Delta\boldsymbol{\Gamma})\|_F^2:=\|\Delta\mathbf{B}\|_F^2+\|\Delta\boldsymbol{\Gamma}\|_F^2$. Here, $\|\cdot\|_F$ denotes the Frobenius norm and $\langle\cdot,\cdot\rangle$ denotes the Frobenius inner product.
\end{proposition}
\begin{proof}
See Appendix~\ref{app:descent}.
\end{proof}
Proposition~\ref{prop:descent} motivates the step size choice $\alpha\leq1/L$, under which each proximal gradient descent update is guaranteed to be a descent step. In practice, we do not compute $L$ explicitly when choosing the step size. Instead, we start from an initial step size and apply backtracking line search until the inequality in Proposition~\ref{prop:descent} is satisfied by the trial update. This backtracking procedure can also be used to check that the trial update remains in $\mathbb{D}$.

With the step size condition above in hand, we now characterize the convergence properties of Algorithm~\ref{alg:proximal} for fixed $\mu\geq0$. The following theorem presents these properties.
\begin{theorem}
\label{thrm:convergence}
Let $\mu\geq0$. Fix any compact set $\Omega\subset\mathbb{D}$. Then there exists a constant $L<\infty$ such that, if Algorithm~\ref{alg:proximal} is run with step size $\alpha<1/L$ and the iterates $(\mathbf{B}^{(k)},\boldsymbol{\Gamma}^{(k)})$ remain in $\Omega$, the sequence of objective values $\{f_\mu(\mathbf{B}^{(k)},\boldsymbol{\Gamma}^{(k)})\}$ is decreasing and converges to a finite limit $f_\mu^\star$. Moreover, as $k\to\infty$, the iterates satisfy
\begin{equation*}
\|(\mathbf{B}^{(k+1)},\boldsymbol{\Gamma}^{(k+1)})-(\mathbf{B}^{(k)},\boldsymbol{\Gamma}^{(k)})\|_F\to0.
\end{equation*}
\end{theorem}
\begin{proof}
See Appendix~\ref{app:convergence}.
\end{proof}
Theorem~\ref{thrm:convergence} shows that the proximal gradient descent iterates are well-behaved. In particular, the objective values decrease monotonically and the successive updates vanish asymptotically. Thus, under the stated step size condition, the objective sequence converges to a finite limit and the change from one iterate to the next becomes arbitrarily small as the algorithm progresses. The requirement that the iterates remain in a compact set is a technical condition used in the proof. Although this condition can be enforced explicitly if desired, our numerical experience is that it is not necessary to impose in practice for stable convergence.

\subsection{Batched Updates}

Direct evaluation of the negative log-likelihood in Algorithm~\ref{alg:proximal} requires computing, for each node $k$ and cluster $i$, the covariance matrix $\mathbf{V}^{(i)}(\boldsymbol{\gamma}_k)$ together with its log-determinant and the quadratic form involving its inverse. Both of these quantities depend on a Cholesky factorization of $\mathbf{V}^{(i)}(\boldsymbol{\gamma}_k)$. Looping over all $m$ clusters and $p$ nodes, therefore, leads to $mp$ separate factorizations at every gradient step, each costing $O(n_i^3)$, which is computationally prohibitive when performed serially. To avoid this bottleneck, we perform the entire likelihood evaluation in parallel using batched linear algebra routines. Prior to running Algorithm~\ref{alg:proximal}, the cluster matrices $\mathbf{X}^{(i)}$ are padded into a block-aligned tensor $\mathbf{X}\in\mathbb{R}^{m\times\max_i n_i\times p}$ with a row mask to drop padded entries. Then, within each iteration of Algorithm~\ref{alg:proximal}, the random-effects variances $\boldsymbol{\gamma}_k$ are broadcast across clusters, producing the collection $\{\mathbf{V}^{(i)}(\boldsymbol{\gamma}_k)\}_{ik}$ as a four-dimensional tensor. Next, batched routines compute the factorizations concurrently, and the resulting triangular factors are used (again in batch) to extract the log-determinant terms from their diagonals and perform triangular solves for the quadratic terms.

In the setting $n_i\geq p$, the log-determinant and quadratic terms can be evaluated via Sylvester's determinant identity and the Woodbury identity using $p\times p$ batched factorizations, reducing the per-factorization cost from $O(n_i^3)$ to $O(p^3)$ \citep[see, e.g.,][]{Harville1997}.

\subsection{Computational Complexity}

Each iteration of Algorithm~\ref{alg:proximal} evaluates the negative log-likelihood term $l(\mathbf{B},\boldsymbol{\Gamma})$ and acyclicity term $h(\mathbf{B},\boldsymbol{\Gamma})$, along with their gradients. The acyclicity term, which involves the log-determinant of a $p\times p$ matrix, requires $O(p^3)$ operations. The negative log-likelihood term requires $O(p\sum_{i=1}^m\min\{n_i^2p+n_i^3,p^3\})$ operations. When the number of clusters $m$ grows with the total sample size $\ndot:=\sum_{i=1}^mn_i$ and the cluster sizes $n_i$ remain uniformly bounded by a constant $n$, the negative log-likelihood term reduces to $O(p^2\ndot)$ operations. The gradients of both terms incur the same asymptotic costs. Hence, the cost per iteration of Algorithm~\ref{alg:proximal} is $O(p\sum_{i=1}^m\min\{n_i^2p+n_i^3,p^3\}+p^3)$ in general and $O(p^2\ndot+p^3)$ in the bounded-cluster regime.

To examine how the run time scales with the number of nodes, we time Algorithm~\ref{alg:path} as $p$ varies. Figure~\ref{fig:timings} reports the results.
\begin{figure}[t]
\centering
\includegraphics[width=0.5\textwidth]{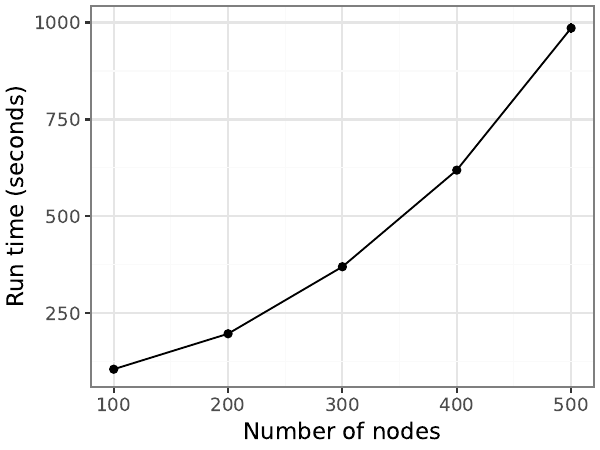}
\caption{Run time in seconds as a function of the number of nodes $p$ with $\ndot=1{,}000$ observations and $m=20$ clusters, measured on an NVIDIA RTX 4090. Averages (solid points) and standard errors (error bars) are measured over 30 datasets. All runs use regularization parameters $\lambda_1=\lambda_2=0.001$.}
\label{fig:timings}
\end{figure}
As expected, the run time grows superlinearly in $p$, though not quite at a cubic rate (cubic complexity in $p$ is worst-case complexity for computing the log-determinant). Overall, the times remain practical for many applications, with training taking about a minute for $p=100$ and about 15 minutes for $p=500$.

\section{Statistical Properties}
\label{sec:statistical}

\subsection{Setup}

We now study some statistical properties of mixed DAGs. In particular, we first establish an identifiability result for the model and then present a statistical guarantee for our estimator in the asymptotic regime where the number of clusters $m\to\infty$. The guarantee addresses both (i) parameter estimation consistency and (ii) structure recovery consistency. 

We assume the data-generating process is a structural equation model of the form \eqref{eq:sem}:
\begin{equation*}
\mathbf{X}^{(i)}=\mathbf{X}^{(i)}\mathbf{W}_0^{(i)}+\boldsymbol{\varepsilon}^{(i)},\qquad\mathbf{W}_0^{(i)}=\mathbf{B}_0+\mathbf{U}_0^{(i)},\qquad i=1,\dots,m,
\end{equation*}
where the noise $\boldsymbol{\varepsilon}^{(i)}\sim\mathrm{N}(\mathbf{0},\mathbf{I})$ and the random-effects matrix $\mathbf{U}_0^{(i)}\in\mathbb{R}^{p\times p}$ has columns satisfying $\mathbf{u}_{k,0}^{(i)}\sim\mathrm{N}(\mathbf{0},\operatorname{diag}(\boldsymbol{\gamma}_{k,0}))$, with all random effects mutually independent and independent of the noise. The model parameters $\mathbf{B}_0\in\mathbb{R}^{p\times p}$ and $\boldsymbol{\Gamma}_0\in\mathbb{R}_+^{p\times p}$, representing the true fixed- and random-effects parameters, respectively, constitute a causal model satisfying
\begin{equation*}
\mathcal{G}(\mathbf{B}_0)\cup\mathcal{G}(\boldsymbol{\Gamma}_0)\in\mathrm{DAG}_p.
\end{equation*}
We denote the edge sets (i.e., nonzero components) of the matrices $\mathbf{B}_0$ and $\boldsymbol{\Gamma}_0$ by
\begin{equation*}
\mathcal{S}_\mathbf{B}:=\{(j,k):\beta_{jk,0}\neq0\}\quad\text{and}\quad\mathcal{S}_{\boldsymbol{\Gamma}}:=\{(j,k):\gamma_{jk,0}\neq0\}.
\end{equation*}
We assume $\mathcal{S}_\mathbf{B}\neq\emptyset$ and $\mathcal{S}_{\boldsymbol{\Gamma}}\neq\emptyset$ for simplicity, i.e., there is at least one edge in each graph.

To estimate $(\mathbf{B}_0,\boldsymbol{\Gamma}_0)$, we study the minimizer of the following optimization problem:
\begin{equation*}
(\hat{\mathbf{B}},\hat{\boldsymbol{\Gamma}})\in\underset{(\mathbf{B},\boldsymbol{\Gamma})\in\mathcal{F}}{\arg\min}\;F_m(\mathbf{B},\boldsymbol{\Gamma}).
\end{equation*}
Here, the objective function $F_m$ is the same penalized objective function in \eqref{eq:optcomb} but with the negative log-likelihood scaled by the number of clusters $m$ to ensure a nondegenerate limit:
\begin{equation*}
F_m(\mathbf{B},\boldsymbol{\Gamma}):=\ell_m(\mathbf{B},\boldsymbol{\Gamma})+\lambda_1\|\mathbf{B}\|_1+\lambda_2\|\boldsymbol{\Gamma}\|_1,\qquad\ell_m(\mathbf{B},\boldsymbol{\Gamma})=\frac{1}{m}l(\mathbf{B},\boldsymbol{\Gamma}).
\end{equation*}
The feasible set $\mathcal{F}$ is the same as the constraint set in \eqref{eq:optcomb} but with added box constraints:
\begin{equation}
\label{eq:feasiblebox}
\mathcal{F}:=\{(\mathbf{B},\boldsymbol{\Gamma})\in\mathbb{R}^{p\times p}\times\mathbb{R}_+^{p\times p}:\mathcal{G}(\mathbf{B})\cup\mathcal{G}(\boldsymbol{\Gamma})\in\mathrm{DAG}_p,\,\|\mathbf{B}\|_\infty\leq M_\mathbf{B},\,\|\boldsymbol{\Gamma}\|_\infty\leq M_{\boldsymbol{\Gamma}}\}.
\end{equation}
The box constraints $\|\mathbf{B}\|_\infty\leq M_\mathbf{B}$ and $\|\boldsymbol{\Gamma}\|_\infty\leq M_{\boldsymbol{\Gamma}}$ are included in the constraint set as a technical convenience to guarantee compactness of $\mathcal{F}$. The constants $M_\mathbf{B}>0$ and $M_{\boldsymbol{\Gamma}}>0$ may be set arbitrarily large so long as they satisfy $M_\mathbf{B}>\|\mathbf{B}_0\|_\infty$ and $M_{\boldsymbol{\Gamma}}>\|\boldsymbol{\Gamma}_0\|_\infty$.

\subsection{Assumptions}

The consistency results for the estimator rely on two assumptions which we set forth below. Because some true components $\gamma_{jk,0}$ may equal zero, derivatives with respect to $\gamma_{jk}$ at $\gamma_{jk}=0$ are interpreted as the corresponding one-sided derivatives from within $\mathbb{R}_+^{p\times p}$.

\begin{assumption}
\label{asmp:penalty}
As $m\to\infty$, the regularization parameters $\lambda_1$ and $\lambda_2$ satisfy $\lambda_1\to0$, $\lambda_2\to0$, $\sqrt{m}\lambda_1\to\infty$, $\sqrt{m}\lambda_2\to\infty$, and $\lambda_2/\lambda_1\to\kappa\in(0,\infty)$.
\end{assumption}
Assumption~\ref{asmp:penalty} specifies the scaling of the regularization parameters. The conditions $\lambda_1,\lambda_2\to0$ allow consistent parameter estimation by making the penalties asymptotically negligible. The additional conditions $\sqrt{m}\lambda_1,\sqrt{m}\lambda_2\to\infty$ are needed for consistent structure recovery by ensuring that the penalties dominate the sampling fluctuations. The ratio condition $\lambda_2/\lambda_1\to\kappa\in(0,\infty)$ implies that the two penalties shrink at the same asymptotic order.

\begin{assumption}
\label{asmp:hessian}
The population Hessian $\mathbf{H}:=\nabla^2 L(\mathbf{B}_0,\boldsymbol{\Gamma}_0)$ is positive definite. In addition, let $\mathcal{S}:=(\mathcal{S}_{\mathbf{B}},\mathcal{S}_{\boldsymbol{\Gamma}})$ and write $\mathbf{H}_{\mathcal{S}\mathcal{S}}$ for the principal submatrix of $\mathbf{H}$ indexed by $\mathcal{S}$. For $(j,k)$ indexing a component of $\mathbf{B}$ or $\boldsymbol{\Gamma}$, let $\mathbf{h}_{jk,\mathcal{S}}^{(\mathbf{B})}$ and $\mathbf{h}_{jk,\mathcal{S}}^{(\boldsymbol{\Gamma})}$ denote the row vectors of second derivatives between $\beta_{jk}$ or $\gamma_{jk}$ and the active components $(\mathbf{B}_{\mathcal{S}_\mathbf{B}},\boldsymbol{\Gamma}_{\mathcal{S}_{\boldsymbol{\Gamma}}})$, respectively. Then there exists an $\eta\in(0,1)$ such that
\begin{equation*}
\left|\mathbf{h}_{jk,\mathcal{S}}^{(\mathbf{B})}\mathbf{H}_{\mathcal{S}\mathcal{S}}^{-1}
\begin{pmatrix}
\mathbf{s}_\mathbf{B} \\
\kappa\mathbf{1}
\end{pmatrix}
\right|\leq1-\eta\quad\text{for all}\quad(j,k)\in\mathcal{S}_{\mathbf{B}}^c,
\end{equation*}
and
\begin{equation*}
\mathbf{h}_{jk,\mathcal{S}}^{(\boldsymbol{\Gamma})}\mathbf{H}_{\mathcal{S}\mathcal{S}}^{-1}
\begin{pmatrix}
\mathbf{s}_\mathbf{B} \\
\kappa\mathbf{1}
\end{pmatrix}
\leq\kappa(1-\eta)\quad\text{for all}\quad(j,k)\in\mathcal{S}_{\boldsymbol{\Gamma}}^c,
\end{equation*}
where $\mathbf{s}_\mathbf{B}:=\left(\operatorname{sign}(\beta_{jk,0})\right)_{(j,k)\in\mathcal{S}_\mathbf{B}}$ contains the signs of all active components of $\mathbf{B}$.
\end{assumption}
Assumption~\ref{asmp:hessian} imposes regularity conditions on the population loss. Positive definiteness of the Hessian at the true parameters makes the population loss locally strongly convex around $(\mathbf{B}_0,\boldsymbol{\Gamma}_0)$. The two inequalities constitute an irrepresentability condition, analogous to those used to analyze $\ell_1$-penalized linear models \citep{Zhao2006}, that limits the influence of inactive components through the active parameters, needed for consistent structure recovery.

\subsection{Results}

The following theorem establishes identifiability of the model described in the setup.
\begin{theorem}
\label{thrm:identifiability}
Under the data-generating process described above, if $\mathcal{G}(\mathbf{B}_0)\cup\mathcal{G}(\boldsymbol{\Gamma}_0)$ is a DAG, then the parameters $(\mathbf{B}_0,\boldsymbol{\Gamma}_0)$ are identifiable from the distribution of the observed clustered data. Consequently, the graphs $\mathcal{G}(\mathbf{B}_0)$, $\mathcal{G}(\boldsymbol{\Gamma}_0)$, and their union are identifiable.
\end{theorem}
\begin{proof}
See Appendix~\ref{app:identifiability}.
\end{proof}
Theorem~\ref{thrm:identifiability} implies that the data-generating parameters $(\mathbf{B}_0,\boldsymbol{\Gamma}_0)$ are uniquely determined by the distribution of the observed clustered data. Since $\ell_m(\mathbf{B},\boldsymbol{\Gamma})$ is the negative log-likelihood under our structural equation model, the population loss $L(\mathbf{B},\boldsymbol{\Gamma}):=\mathrm{E}[\ell_m(\mathbf{B},\boldsymbol{\Gamma})]$ coincides with the Kullback--Leibler divergence between the true distribution and the model up to an additive constant. Consequently, $L(\mathbf{B},\boldsymbol{\Gamma})$ is uniquely minimized at $(\mathbf{B}_0,\boldsymbol{\Gamma}_0)$. In the non-clustered setting, \citet{Peters2014} establish an analogous identifiability result for linear Gaussian structural equation models with equal noise variances.

The following theorem establishes asymptotic consistency of parameter estimation and structure recovery of the estimator $(\hat{\mathbf{B}},\hat{\boldsymbol{\Gamma}})$ with respect to the ground truth $(\mathbf{B}_0,\boldsymbol{\Gamma}_0)$.
\begin{theorem}
\label{thrm:consistency}
Suppose the clusters $\mathbf{X}^{(1)},\dots,\mathbf{X}^{(m)}$ are iid and have a common (deterministic) cluster size $n_i\equiv n$ with $1\leq n<\infty$, and Assumptions~\ref{asmp:penalty} and \ref{asmp:hessian} hold. Then the estimator $(\hat{\mathbf{B}},\hat{\boldsymbol{\Gamma}})$ is parameter estimation consistent as $m\to\infty$:
\begin{equation*}
\|(\hat{\mathbf{B}},\hat{\boldsymbol{\Gamma}})-(\mathbf{B}_0,\boldsymbol{\Gamma}_0)\|_F \overset{p}{\to}0.
\end{equation*}
Moreover, $(\hat{\mathbf{B}},\hat{\boldsymbol{\Gamma}})$ is structure recovery consistent as $m\to\infty$:
\begin{equation*}
\mathrm{Pr}\left(\mathcal{G}(\hat{\mathbf{B}})=\mathcal{G}(\mathbf{B}_0)\enspace\text{and}\enspace\mathcal{G}(\hat{\boldsymbol{\Gamma}})=\mathcal{G}(\boldsymbol{\Gamma}_0)\right)\to1.
\end{equation*}
\end{theorem}
\begin{proof}
See Appendix~\ref{app:consistency}.
\end{proof}
Theorem~\ref{thrm:consistency} is stated for equal cluster sizes for simplicity. The result extends to unequal but uniformly bounded cluster sizes $1\leq n_i\leq n<\infty$ with minor additional technicalities.

The proof of Theorem~\ref{thrm:consistency} involves four steps. We first establish consistency of the parameter estimates. Next, we show that this result rules out false negatives in the recovered structure. We then employ a primal--dual witness argument \citep{Wainwright2009} to establish the absence of false positives. Finally, we combine the results to get structure recovery consistency.

To the best of our knowledge, no prior work has established structure recovery consistency for mixed DAGs. Related selection consistency results are available for non-graphical mixed-effects models. For example, \citet{Fan2012} establish model selection consistency for penalized procedures that select both fixed and random effects in linear mixed models. Moreover, even in the classical homogeneous setting, structure recovery guarantees for DAG estimators are limited. One notable exception is the ordered-variable regime, where the DAG reduces to a lower-triangular adjacency and \citet{Shojaie2010} show $\ell_1$-penalized likelihood methods are structure recovery consistent under suitable conditions.

\section{Experiments}
\label{sec:experiments}

\subsection{Baselines and Metrics}

We experimentally evaluate five estimators, using the following labels in figures and tables.
\begin{itemize}
\item Mixed DAG: our proposed mixed-effects DAG estimator, which jointly learns fixed- and random-effects structures while simultaneously enforcing acyclicity of their union.
\item Fixed DAG: a fixed-effects-only DAG estimator following \citet{Bello2022}, which does not include random effects and therefore ignores cluster-level heterogeneity.
\item Mixed DAG (FM): a computationally simpler masked alternative to the mixed DAG that replaces the explicit acyclicity constraint with a fixed mask (FM) obtained from the fixed DAG estimate. Specifically, the fixed DAG estimate is converted into a topological ordering, and the resulting mask is applied to the fixed- and random-effects matrices, setting all entries that violate the ordering to zero. Thus, only edges from earlier to later variables in the ordering can be selected, thereby guaranteeing acyclicity.
\item Mixed DAG (OM): the same masked mixed-effects DAG estimator as outlined above, but using an oracle mask (OM) obtained from a topological ordering of the true DAG.
\item Mixed DG: a mixed-effects directed graph (DG) estimator that fits the same model as the mixed DAG but does not enforce acyclicity, and hence need not produce a DAG.
\end{itemize}
The oracle variant serves as a purely theoretical baseline, providing an idealized bound on performance. Among these estimators, those with an explicit acyclicity constraint use the log-determinant acyclicity characterization of \citet{Bello2022}, providing a uniform basis for comparison. Sparsity parameters are tuned on a validation set generated independently and identically to the training set. Appendix~\ref{app:implementation} provides implementation details.

We consider several metrics that characterize different attributes of the estimators. We report the estimation error, defined as the average of squared differences between the estimated and true weighted adjacency matrices over all clusters. We report the structural Hamming distance (SHD), defined as the number of edge additions, deletions, and reversals needed to transform the estimated structure into the true structure. We also report the F1 score, defined as the harmonic mean of precision and recall in edge selection. Finally, we report the sparsity, defined as the total number of edges in the estimated graph.

\subsection{Synthetic Data}

We generate synthetic datasets from the structural equation model \eqref{eq:sem} as follows. First, we sample an Erdős--Rényi graph \citep{Erdos1959} with $p$ nodes and $s$ edges, draw a random ordering of the nodes, and orient each edge from the earlier node to the later node in that ordering, thereby obtaining a DAG. Each edge in this graph is then assigned a weight independently from $\mathrm{Unif}([-0.3,-0.1]\cup[0.1,0.3])$, producing the fixed-effects matrix $\mathbf{B}$. We then draw the random-effects variances independently from $\mathrm{Unif}([0.1,0.3])$ to form the matrix $\boldsymbol{\Gamma}$, and use these to generate the random-effects edge weights $\mathbf{U}^{(i)}\sim\mathrm{N}(\mathbf{0},\boldsymbol{\Gamma})$ for $i=1,\dots,m$. Independently, the noise terms are sampled for each cluster as $\boldsymbol{\varepsilon}^{(i)}\sim\mathrm{N}(\mathbf{0},\mathbf{I})$ for $i=1,\dots,m$. Finally, the total sample size $\ndot$ is varied over five values logarithmically spaced between $30$ and $1{,}000$. We set the number of clusters to $m=\lceil\sqrt{\ndot}\rceil$ and assign observations randomly to clusters, resulting in clusters that vary in size with a mean of $\ndot/m$.

Figures~\ref{fig:synthetic_50_50_inf+inf_erdos_renyi} and \ref{fig:synthetic_200_200_inf+inf_erdos_renyi} present results for graphs with $p=50$ and $p=200$ nodes, respectively.
\begin{figure}[t]
\centering
\includegraphics[width=0.8\textwidth]{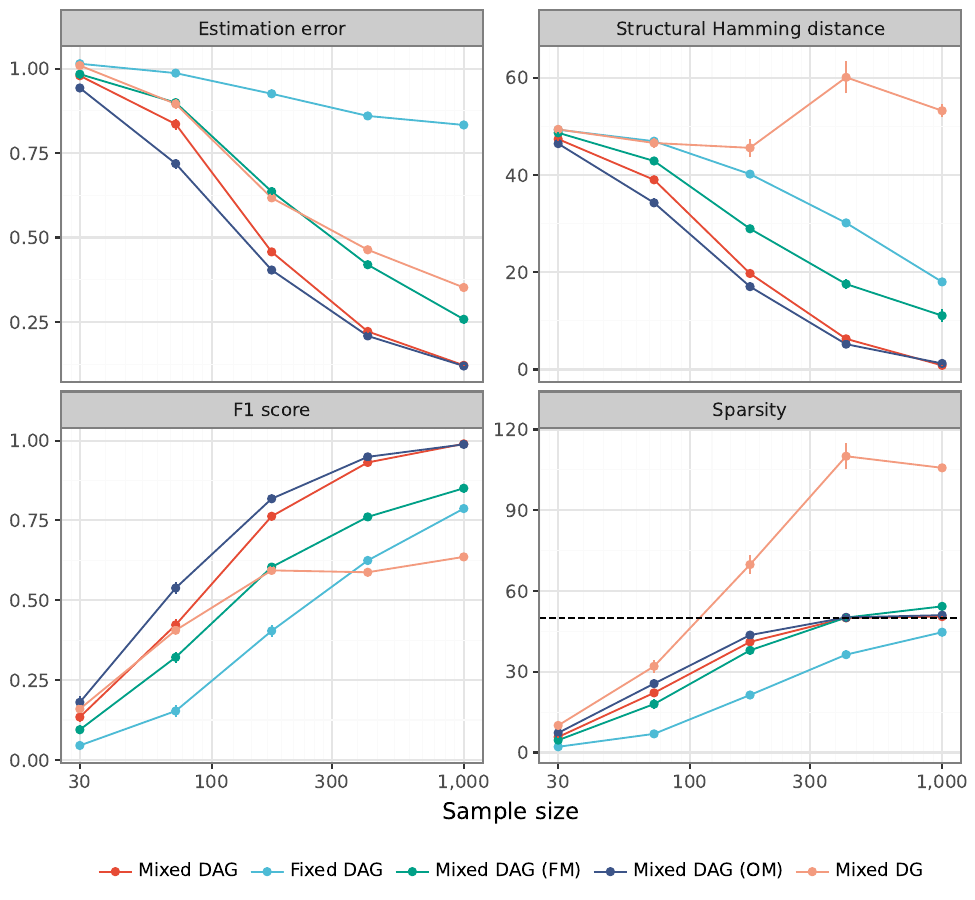}
\caption{Performance on synthetic data generated from Erdős--Rényi DAGs with $p=50$ nodes, $s=50$ edges, and $m=\lceil\sqrt{\ndot}\rceil$ clusters. Averages (solid points) and standard errors (error bars) are measured over 30 datasets. The dashed horizontal line in the bottom right panel indicates the number of edges in the true DAG.}
\label{fig:synthetic_50_50_inf+inf_erdos_renyi}
\end{figure}
\begin{figure}[t]
\centering
\includegraphics[width=0.8\textwidth]{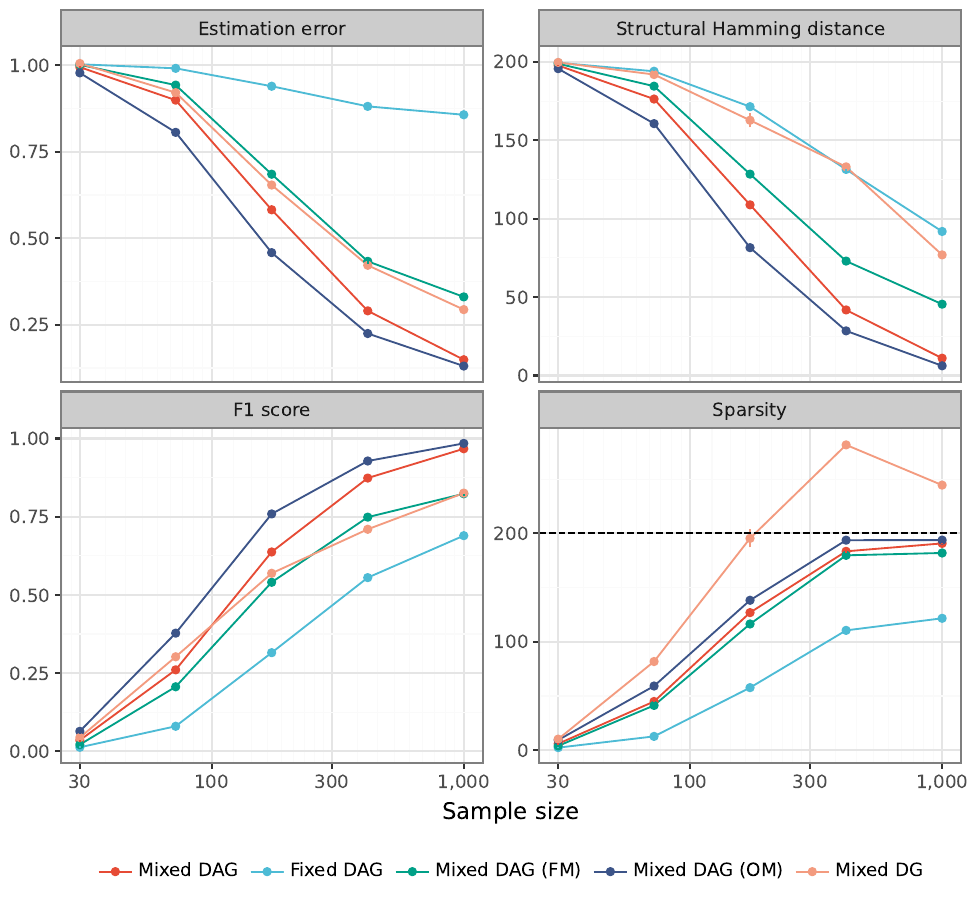}
\caption{Performance on synthetic data generated from Erdős--Rényi DAGs with $p=200$ nodes, $s=200$ edges, and $m=\lceil\sqrt{\ndot}\rceil$ clusters. Averages (solid points) and standard errors (error bars) are measured over 30 datasets. The dashed horizontal line in the bottom right panel indicates the number of edges in the true DAG.}
\label{fig:synthetic_200_200_inf+inf_erdos_renyi}
\end{figure}
In both settings, the mixed DAG exhibits excellent performance in both estimation and structure recovery. At very small sample sizes, where faithful recovery of the true DAG is intrinsically difficult, its results are comparable to those of the fixed DAG. As the sample size increases, however, a clear gap emerges, with the mixed DAG delivering substantially superior performance. Notably, the fixed DAG fails to recover the true graph structure even when $\ndot=1{,}000$. The FM variant improves on the fixed DAG by incorporating cluster-specific edge weights, but its performance is limited by the quality of the ordering estimated from the fixed DAG. In contrast, the mixed DAG reliably recovers sparse, accurate structures, making it a practical method for learning DAGs from clustered data at scale. Moreover, its performance approaches that of the OM variant, which uses a topological ordering of the true DAG as an idealized benchmark. The mixed DG performs worst overall because it does not respect the acyclic nature of the graphs, highlighting the importance of the DAG constraint.

Further experiments under alternative simulation designs are reported in Appendix~\ref{app:additional}. These additional results show that the mixed DAG continues to perform excellently under (i) alternative sparsity levels, (ii) alternative graph types, and (iii) alternative cluster regimes.

\subsection{Semi-synthetic Data}

To complement the fully synthetic graph designs in the previous experiments, we generate semi-synthetic datasets from the ANDES network, a well-known DAG from physics \citep{Conati1997}. It is publicly available from the Bayesian Network Repository (\url{https://www.bnlearn.com/bnrepository}). The network contains $p=223$ nodes and $s=338$ edges, representing a large-scale graph with realistic topologies and sparsity patterns. We generate the datasets by treating the network adjacency matrix as the ground-truth DAG and draw datasets by sampling fixed effects, random effects, and noise in the same fashion as before.

Table~\ref{tab:semisynthetic} reports results for total sample size $\ndot=1{,}000$, with the number of clusters $m=\lceil\sqrt{\ndot}\rceil=32$ as in the main synthetic experiments. On this large graph, the mixed DAG maintains strong performance, achieving the best metrics among all practical competitors. The improvement over the fixed DAG is particularly pronounced, with a sevenfold decrease in estimation error and a sizable gain in F1 score. Importantly, the gap between the mixed DAG and the oracle mask variant is relatively modest, indicating that jointly learning the ordering and mixed effects does not incur a large statistical penalty at this scale. In contrast, methods that rely on a fixed ordering or omit acyclicity constraints exhibit degraded performance.

\begin{table*}[ht]
\centering
\caption{Performance on semi-synthetic data generated from the ANDES network from the Bayesian Network Repository with $p=223$ nodes, $s=338$ edges, and $m=\lceil\sqrt{\ndot}\rceil=32$ clusters. Averages (standard errors) are measured over 30 datasets.}
\label{tab:semisynthetic}
\small
\begin{tabularx}{\linewidth}{lXXXX}
\toprule
 & Est. error & SHD & F1 score & Sparsity \\
\midrule
Mixed DAG & 0.13 (0.00) & 17.23 (0.67) & 0.97 (0.00) & 321.57 (0.64) \\
Fixed DAG & 0.90 (0.00) & 187.13 (2.12) & 0.60 (0.01) & 209.23 (2.39) \\
Mixed DAG (FM) & 0.38 (0.01) & 87.23 (1.13) & 0.80 (0.00) & 306.97 (1.29) \\
Mixed DAG (OM) & 0.12 (0.00) & 10.43 (0.50) & 0.98 (0.00) & 327.57 (0.50) \\
Mixed DG & 0.29 (0.00) & 118.97 (1.29) & 0.83 (0.00) & 364.20 (1.20) \\
\bottomrule
\end{tabularx}
\end{table*}

\subsection{Real Data}

As an illustrative real-world example, we consider the lung cancer multiplex immunofluorescence dataset from the \texttt{R} package \texttt{VectraPolarisData}. The dataset contains cell-level protein measurements from $m=153$ patients. The observational units are individual cells, and for each cell we observe the expression levels of $p=6$ protein markers of tumor and immune activity: \texttt{cd19}, \texttt{cd3}, \texttt{cd14}, \texttt{cd8}, \texttt{hladr}, and \texttt{ck}. This low-dimensional setting enables direct visualization of learned graphs that might be indicative of tumor-immune proteomic interactions. We treat cells as clustered by patient to allow marker interactions to vary across individuals and to incorporate patient-level tumor heterogeneity, which is common in cancers and especially in lung cancer. To balance patient contributions in this illustrative analysis, we randomly select 300 cells per patient and split them equally into training, validation, and testing sets, yielding $\ndot=153\times100=15{,}300$ cell-level observations in each split.

Figure~\ref{fig:real} shows that both methods recover a coherent ordering in which immune-associated markers \texttt{cd19}, \texttt{cd8}, \texttt{hladr}, \texttt{cd3}, and \texttt{cd14} feed into \texttt{ck}, consistent with \texttt{ck} as an epithelial/tumor marker downstream of immune-associated features.
\begin{figure*}[t]
\centering
\begin{subfigure}[t]{0.49\linewidth}
\centering
\includegraphics[width=0.8\textwidth]{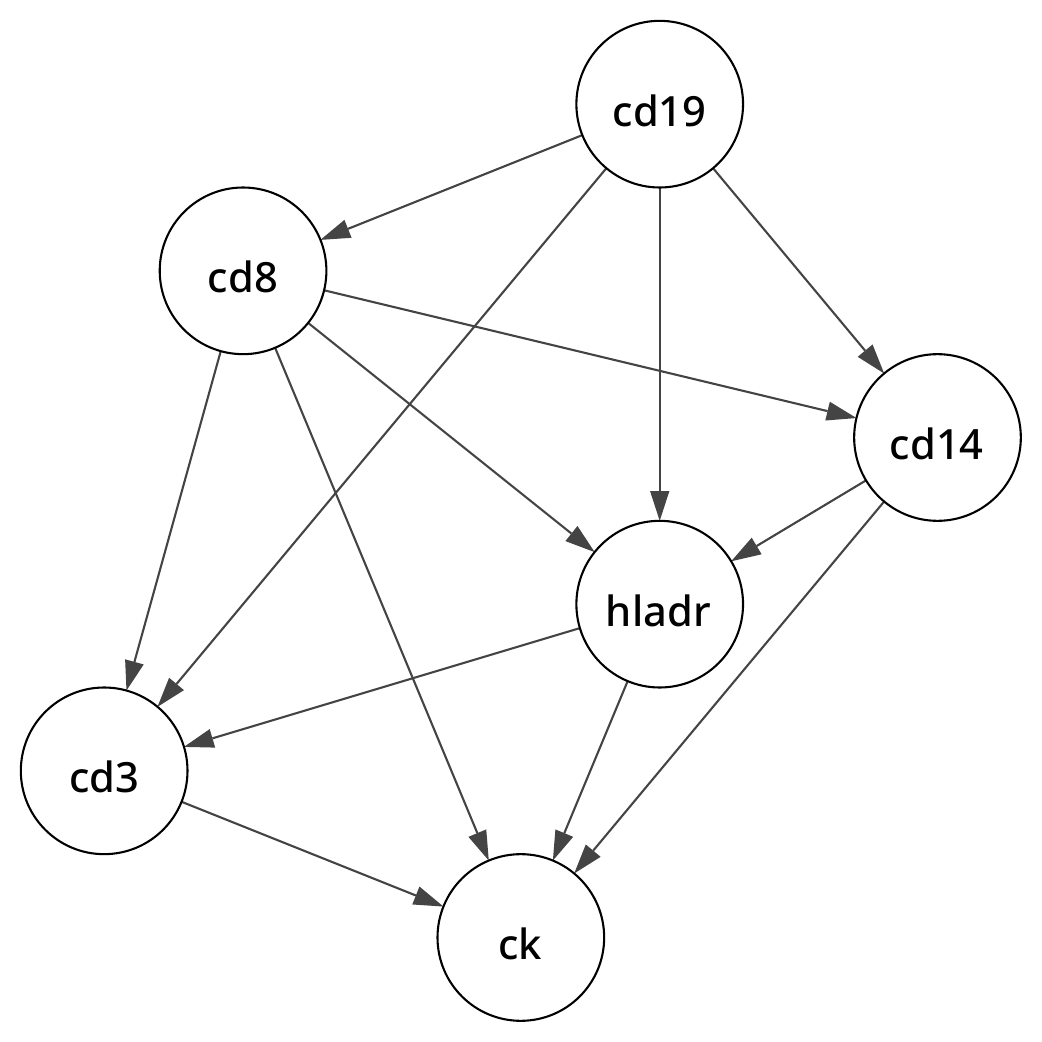}
\caption{Fixed DAG}
\end{subfigure}
\begin{subfigure}[t]{0.49\linewidth}
\centering
\includegraphics[width=0.8\textwidth]{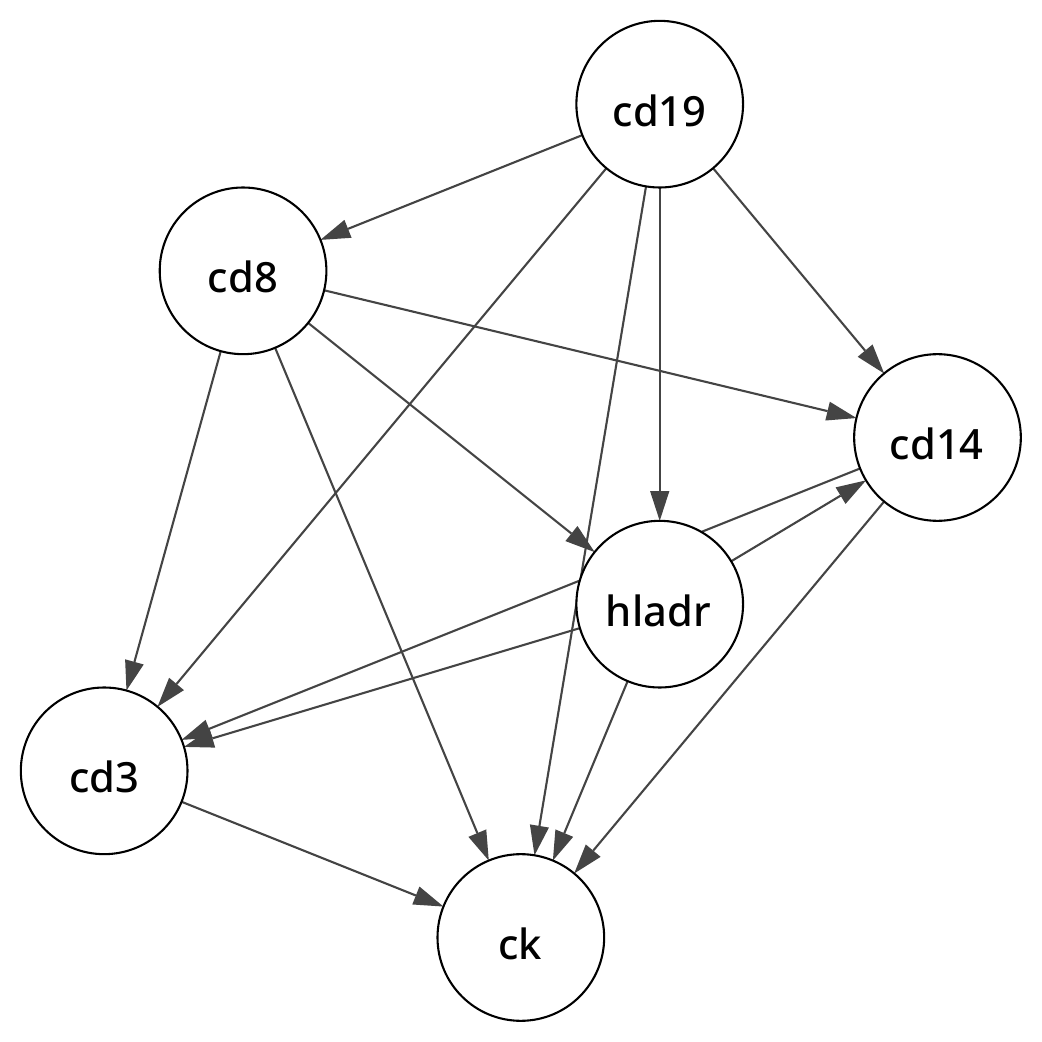}%
\caption{Mixed DAG}
\end{subfigure}
\caption{Graphs learned on the \texttt{VectraPolarisData}. All edges in the graph produced by the mixed DAG contain both fixed and random effects.}
\label{fig:real}
\end{figure*}
Relative to the fixed DAG, which selects 13 edges, the mixed DAG learns a slightly denser graph containing 15 edges. Compared with the fixed DAG, the mixed DAG reverses the direction of the edge between \texttt{hladr} and \texttt{cd14}, yielding \texttt{hladr}$\to$\texttt{cd14} rather than \texttt{cd14}$\to$\texttt{hladr}, and also adds the dependencies \texttt{cd14}$\to$\texttt{cd3} and \texttt{cd19}$\to$\texttt{ck}. The remaining structure is consistent across the two graphs, including edges such as \texttt{cd8}$\to$\texttt{cd14}, \texttt{cd3}$\to$\texttt{ck}, and \texttt{cd14}$\to$\texttt{ck}. These structural differences are reflected in a sizable gap in the mean squared reconstruction error on the testing set, which drops from 0.83 under the fixed DAG to 0.67 under the mixed DAG.

\section{Final Remarks}
\label{sec:final}

Many fields, including medicine and biology, give rise to clustered data with heterogeneous effects, making them unsuitable for existing approaches to causal discovery that assume a homogeneous population. Our approach addresses this issue by extending scalable structure learning techniques to clustered data through the integration of fixed- and random-effects principles within the framework of DAGs. We use a differentiable acyclicity constraint that guarantees the union of the fixed- and random-effects graphs remains acyclic, and develop a provably convergent first-order method that leverages efficient batched updates across clusters, making it feasible to learn graphs with hundreds of nodes. We further establish identifiability and show that our estimator recovers the true structure asymptotically. Synthetic, semi-synthetic, and real proteomics experiments demonstrate that the new estimator can recover both population-level and cluster-specific effects missed by alternative estimators. These results highlight its value as a principled tool for structure learning in clustered settings.

\acks{Thompson and Wand acknowledge financial support from the Australian Research Council under Discovery Project DP230101179. Baladandayuthapani was  partially supported by National Institutes of Health grants R01CA244845-
01A1 and P30 CA46592 for this work. }

\appendix

\section{Proof of Proposition~\ref{prop:union_acyclicity}}
\label{app:union_acyclicity}

\begin{proof}

To simplify exposition, define the matrix
\begin{equation*}
\mathbf{S}:=\mathbf{B}\odot\mathbf{B}+\boldsymbol{\Gamma}.
\end{equation*}
Since $(\mathbf{B},\boldsymbol{\Gamma})\in\mathbb{D}$, the matrix $\mathbf{S}$ is entrywise nonnegative and satisfies $\rho(\mathbf{S})<1$. Hence, Theorem~1 of \citep{Bello2022} applies to $\mathbf{S}$ and gives
\begin{equation*}
-\log\det(\mathbf{I}-\mathbf{S})=0\iff\mathcal{G}(\mathbf{S})\in\mathrm{DAG}_p.
\end{equation*}
Since the left-hand side is exactly $h(\mathbf{B},\boldsymbol{\Gamma})$, it follows
\begin{equation*}
h(\mathbf{B},\boldsymbol{\Gamma})=0\iff\mathcal{G}(\mathbf{S})\in\mathrm{DAG}_p.
\end{equation*}
For each pair $(j,k)$, the definition of $\mathbf{S}$ and the nonnegativity of $\beta_{jk}^2$ and $\gamma_{jk}$ imply
\begin{equation*}
s_{jk}=0\iff\beta_{jk}^2=0\enspace\text{and}\enspace\gamma_{jk}=0\iff\beta_{jk}=0\enspace\text{and}\enspace\gamma_{jk}=0.
\end{equation*}
Thus, $s_{jk}\neq0$ when $\beta_{jk}\neq0$ or $\gamma_{jk}\neq0$, and hence
\begin{equation*}
\mathcal{G}(\mathbf{S})=\mathcal{G}(\mathbf{B}\odot\mathbf{B}+\boldsymbol{\Gamma})=\mathcal{G}(\mathbf{B})\cup\mathcal{G}(\boldsymbol{\Gamma}).
\end{equation*}
Combining these equivalences gives
\begin{equation*}
h(\mathbf{B},\boldsymbol{\Gamma})=0\iff\mathcal{G}(\mathbf{B})\cup\mathcal{G}(\boldsymbol{\Gamma})\in\mathrm{DAG}_p,
\end{equation*}
thereby proving the claim of the proposition.

\end{proof}

\section{Proof of Proposition~\ref{prop:descent}}
\label{app:descent}

The proof of the proposition requires two technical lemmas, which establish Lipschitz continuity of the gradients of the smooth parts of the objective function on a compact set. Lemma~\ref{lemma:llipschitz} first shows this property for the negative log-likelihood.

\begin{lemma}
\label{lemma:llipschitz}
Let $l(\mathbf{B},\boldsymbol{\Gamma})$ be the negative log-likelihood \eqref{eq:likelihood} defined on the domain $\mathbb{R}^{p\times p}\times\mathbb{R}_+^{p\times p}$. Fix any compact set $\Omega\subset\mathbb{R}^{p\times p}\times\mathbb{R}_+^{p\times p}$. Then there exists a constant $L_l<\infty$ such that
\begin{equation*}
\|\nabla l(\mathbf{B}_1,\boldsymbol{\Gamma}_1)-\nabla l(\mathbf{B}_2,\boldsymbol{\Gamma}_2)\|_F\leq L_l\|(\mathbf{B}_1,\boldsymbol{\Gamma}_1)-(\mathbf{B}_2,\boldsymbol{\Gamma}_2)\|_F,
\end{equation*}
for all $(\mathbf{B}_1,\boldsymbol{\Gamma}_1),(\mathbf{B}_2,\boldsymbol{\Gamma}_2)\in\Omega$. In particular, $\nabla l(\mathbf{B},\boldsymbol{\Gamma})$ is $L_l$-Lipschitz continuous on $\Omega$.
\end{lemma}

\begin{proof}

For each cluster $i=1,\dots,m$ and node $k=1,\dots,p$, define
\begin{equation*}
\mathbf{r}_k^{(i)}(\boldsymbol{\beta}):=\mathbf{x}_k^{(i)}-\mathbf{X}^{(i)}\boldsymbol{\beta},
\end{equation*}
and
\begin{equation*}
\phi_{ik}(\boldsymbol{\beta},\boldsymbol{\gamma}):=\log\det\{\mathbf{V}^{(i)}(\boldsymbol{\gamma})\}+\mathbf{r}_k^{(i)}(\boldsymbol{\beta})^\top\mathbf{V}^{-(i)}(\boldsymbol{\gamma})\mathbf{r}_k^{(i)}(\boldsymbol{\beta}).
\end{equation*}
Thus, writing $\boldsymbol{\beta}_k$ and $\boldsymbol{\gamma}_k$ for the $k$th column of $\mathbf{B}$ and $\boldsymbol{\Gamma}$, respectively, we have
\begin{equation*}
l(\mathbf{B},\boldsymbol{\Gamma})=\sum_{k=1}^p\sum_{i=1}^m\phi_{ik}(\boldsymbol{\beta}_k,\boldsymbol{\gamma}_k).
\end{equation*}
Now, for any $\boldsymbol{\gamma}\in\mathbb{R}_+^p$, it holds
\begin{equation*}
\mathbf{V}^{(i)}(\boldsymbol{\gamma})=\mathbf{I}+\mathbf{X}^{(i)}\operatorname{diag}(\boldsymbol{\gamma})\mathbf{X}^{(i)\top}\succeq\mathbf{I},
\end{equation*}
where $\succeq$ denotes positive semidefinite ordering, and hence $\mathbf{V}^{(i)}(\boldsymbol{\gamma})$ is positive definite with $\lambda_{\min}\{\mathbf{V}^{(i)}(\boldsymbol{\gamma})\}\geq1$. In particular, $\log\det\{\mathbf{V}^{(i)}(\boldsymbol{\gamma})\}$ and $\mathbf{V}^{-(i)}(\boldsymbol{\gamma})$ are well-defined, and these mappings are $C^\infty$ on the open set
\begin{equation*}
\mathcal{U}:=\left\{(\mathbf{B},\boldsymbol{\Gamma})\in\mathbb{R}^{p\times p}\times\mathbb{R}^{p\times p}:\mathbf{V}^{(i)}(\boldsymbol{\gamma}_k)\succ\mathbf{0}\,\forall\,i=1,\dots,m,\,k=1,\dots,p\right\}.
\end{equation*}
Since $\mathbf{V}^{(i)}(\boldsymbol{\gamma}_k)\succeq\mathbf{I}$ whenever $\boldsymbol{\Gamma}\geq\mathbf{0}$, we have $\mathbb{R}^{p\times p}\times\mathbb{R}_+^{p\times p}\subset\mathcal{U}$. Consequently, $l(\mathbf{B},\boldsymbol{\Gamma})$ is twice continuously differentiable on $\mathcal{U}$. Now, since $\mathbb{R}^{p\times p}\times\mathbb{R}_+^{p\times p}$ is convex and $\Omega$ is compact, its convex hull $\operatorname{conv}(\Omega)$ is compact and satisfies
\begin{equation*}
\operatorname{conv}(\Omega)\subset\mathbb{R}^{p\times p}\times\mathbb{R}_+^{p\times p}\subset\mathcal{U}.
\end{equation*}
Therefore, by continuity of the Hessian and compactness of $\operatorname{conv}(\Omega)$, there exists a finite constant
\begin{equation*}
L_l:=\sup_{(\mathbf{B},\boldsymbol{\Gamma})\in\operatorname{conv}(\Omega)}\|\nabla^2l(\mathbf{B},\boldsymbol{\Gamma})\|_\mathrm{op}<\infty,
\end{equation*}
where $\|\cdot\|_\mathrm{op}$ denotes the operator norm. Next, take any $(\mathbf{B}_1,\boldsymbol{\Gamma}_1),(\mathbf{B}_2,\boldsymbol{\Gamma}_2)\in\Omega$ and for $t\in[0,1]$ define the line segment
\begin{equation*}
(\mathbf{B}(t),\boldsymbol{\Gamma}(t)):=(\mathbf{B}_2,\boldsymbol{\Gamma}_2)+t\{(\mathbf{B}_1,\boldsymbol{\Gamma}_1)-(\mathbf{B}_2,\boldsymbol{\Gamma}_2)\}.
\end{equation*}
By convexity of $\operatorname{conv}(\Omega)$ it holds $(\mathbf{B}(t),\boldsymbol{\Gamma}(t))\in\operatorname{conv}(\Omega)$ for all $t\in[0,1]$ and therefore
\begin{equation*}
\nabla l(\mathbf{B}_1,\boldsymbol{\Gamma}_1)-\nabla l(\mathbf{B}_2,\boldsymbol{\Gamma}_2)=\int_0^1\nabla^2 l(\mathbf{B}(t),\boldsymbol{\Gamma}(t))\{(\mathbf{B}_1,\boldsymbol{\Gamma}_1)-(\mathbf{B}_2,\boldsymbol{\Gamma}_2)\}\,dt.
\end{equation*}
Taking norms, applying the triangle inequality, and then using $\|\nabla^2 l(\mathbf{B}(t),\boldsymbol{\Gamma}(t))\|_\mathrm{op}\leq L_l$ for all $t\in[0,1]$, we have
\begin{equation*}
\begin{split}
\|\nabla l(\mathbf{B}_1,\boldsymbol{\Gamma}_1)-\nabla l(\mathbf{B}_2,\boldsymbol{\Gamma}_2)\|_F&\leq\left\|\int_0^1\nabla^2 l(\mathbf{B}(t),\boldsymbol{\Gamma}(t))\{(\mathbf{B}_1,\boldsymbol{\Gamma}_1)-(\mathbf{B}_2,\boldsymbol{\Gamma}_2)\}\,dt\right\|_F \\
&\leq\int_0^1\|\nabla^2 l(\mathbf{B}(t),\boldsymbol{\Gamma}(t))\|_\mathrm{op}\,dt\,\|(\mathbf{B}_1,\boldsymbol{\Gamma}_1)-(\mathbf{B}_2,\boldsymbol{\Gamma}_2)\|_F \\
&\leq L_l\|(\mathbf{B}_1,\boldsymbol{\Gamma}_1)-(\mathbf{B}_2,\boldsymbol{\Gamma}_2)\|_F.
\end{split}
\end{equation*}
We thus conclude that $\nabla l(\mathbf{B},\boldsymbol{\Gamma})$ is $L_l$-Lipschitz on $\Omega$.

\end{proof}

Lemma~\ref{lemma:hlipschitz} shows the same property for the log-determinant function.

\begin{lemma}
\label{lemma:hlipschitz}
Let $h(\mathbf{B},\boldsymbol{\Gamma})$ be the log-determinant function \eqref{eq:acyclicity} defined on the domain $\mathbb{D}$ given in \eqref{eq:acyclicitydomain}. Fix any compact set $\Omega\subset\mathbb{D}$. Then there exists a constant $L_h<\infty$ such that
\begin{equation*}
\|\nabla h(\mathbf{B}_1,\boldsymbol{\Gamma}_1)-\nabla h(\mathbf{B}_2,\boldsymbol{\Gamma}_2)\|_F\leq L_h\|(\mathbf{B}_1,\boldsymbol{\Gamma}_1)-(\mathbf{B}_2,\boldsymbol{\Gamma}_2)\|_F,
\end{equation*}
for all $(\mathbf{B}_1,\boldsymbol{\Gamma}_1),(\mathbf{B}_2,\boldsymbol{\Gamma}_2)\in\Omega$. In particular, $\nabla h(\mathbf{B},\boldsymbol{\Gamma})$ is $L_h$-Lipschitz continuous on $\Omega$.
\end{lemma}

\begin{proof}

Define the quantities
\begin{equation*}
\mathbf{S}(\mathbf{B},\boldsymbol{\Gamma}):=\mathbf{B}\odot\mathbf{B}+\boldsymbol{\Gamma},\qquad\mathbf{M}(\mathbf{B},\boldsymbol{\Gamma}):=\mathbf{I}-\mathbf{S}(\mathbf{B},\boldsymbol{\Gamma}),
\end{equation*}
so
\begin{equation*}
h(\mathbf{B},\boldsymbol{\Gamma})=-\log\det\{\mathbf{M}(\mathbf{B},\boldsymbol{\Gamma})\}.
\end{equation*}
Now, for any $(\mathbf{B},\boldsymbol{\Gamma})\in\mathbb{D}$, the matrix $\mathbf{S}(\mathbf{B},\boldsymbol{\Gamma})$ is entrywise nonnegative and satisfies $\rho\{\mathbf{S}(\mathbf{B},\boldsymbol{\Gamma})\}<1$, meaning $\det\{\mathbf{M}(\mathbf{B},\boldsymbol{\Gamma})\}>0$ \citep[since it is a nonsingular $M$-matrix; see][]{Meyer2000}. In particular, these properties mean $h(\mathbf{B},\boldsymbol{\Gamma})$ is $C^\infty$ on the open set
\begin{equation*}
\mathcal{U}:=\left\{(\mathbf{B},\boldsymbol{\Gamma})\in\mathbb{R}^{p\times p}\times\mathbb{R}^{p\times p}:\det\{\mathbf{M}(\mathbf{B},\boldsymbol{\Gamma})\}>0\right\}.
\end{equation*}
Since $\det\{\mathbf{M}(\mathbf{B},\boldsymbol{\Gamma})\}>0$ for all $(\mathbf{B},\boldsymbol{\Gamma})\in\mathbb{D}$, we have $\Omega\subset\mathcal{U}$. Consequently, $h(\mathbf{B},\boldsymbol{\Gamma})$ is twice continuously differentiable on $\mathcal{U}$, so the Hessian $\nabla^2 h(\mathbf{B},\boldsymbol{\Gamma})$ is continuous on $\mathcal{U}$. Now, since $\Omega$ is compact and $\mathcal{U}$ is open, there exists $\delta>0$ such that the closed $\delta$-neighborhood of $\Omega$ is contained in $\mathcal{U}$. Hence, by continuity of the Hessian and compactness of the closed $\delta$-neighborhood of $\Omega$, there exists a finite constant
\begin{equation*}
M_h:=\sup_{\inf_{(\tilde{\mathbf{B}},\tilde{\boldsymbol{\Gamma}})\in\Omega}\|(\mathbf{B},\boldsymbol{\Gamma})-(\tilde{\mathbf{B}},\tilde{\boldsymbol{\Gamma}})\|_F\leq\delta}\|\nabla^2 h(\mathbf{B},\boldsymbol{\Gamma})\|_\mathrm{op}<\infty.
\end{equation*}
Also, by continuity of the gradient and compactness of $\Omega$, it holds
\begin{equation*}
G_h:=\sup_{(\mathbf{B},\boldsymbol{\Gamma})\in\Omega}\|\nabla h(\mathbf{B},\boldsymbol{\Gamma})\|_F<\infty.
\end{equation*}
Define the quantity
\begin{equation*}
L_h:=\max\left\{M_h,\frac{2G_h}{\delta}\right\}.
\end{equation*}
Next, take any $(\mathbf{B}_1,\boldsymbol{\Gamma}_1),(\mathbf{B}_2,\boldsymbol{\Gamma}_2)\in\Omega$. First, suppose case (i) is true:
\begin{equation*}
\|(\mathbf{B}_1,\boldsymbol{\Gamma}_1)-(\mathbf{B}_2,\boldsymbol{\Gamma}_2)\|_F\leq\delta.
\end{equation*}
For $t\in[0,1]$ define the line segment
\begin{equation*}
(\mathbf{B}(t),\boldsymbol{\Gamma}(t)):=(\mathbf{B}_2,\boldsymbol{\Gamma}_2)+t\{(\mathbf{B}_1,\boldsymbol{\Gamma}_1)-(\mathbf{B}_2,\boldsymbol{\Gamma}_2)\}.
\end{equation*}
It then follows
\begin{equation*}
\inf_{(\tilde{\mathbf{B}},\tilde{\boldsymbol{\Gamma}})\in\Omega}\|(\mathbf{B}(t),\boldsymbol{\Gamma}(t))-(\tilde{\mathbf{B}},\tilde{\boldsymbol{\Gamma}})\|_F\leq\|(\mathbf{B}(t),\boldsymbol{\Gamma}(t))-(\mathbf{B}_2,\boldsymbol{\Gamma}_2)\|_F=t\|(\mathbf{B}_1,\boldsymbol{\Gamma}_1)-(\mathbf{B}_2,\boldsymbol{\Gamma}_2)\|_F\leq\delta
\end{equation*}
for all $t\in[0,1]$, and hence $(\mathbf{B}(t),\boldsymbol{\Gamma}(t))\in\mathcal{U}$ for all $t\in[0,1]$. Therefore, we have
\begin{equation*}
\nabla h(\mathbf{B}_1,\boldsymbol{\Gamma}_1)-\nabla h(\mathbf{B}_2,\boldsymbol{\Gamma}_2)=\int_0^1\nabla^2 h(\mathbf{B}(t),\boldsymbol{\Gamma}(t))\{(\mathbf{B}_1,\boldsymbol{\Gamma}_1)-(\mathbf{B}_2,\boldsymbol{\Gamma}_2)\}\,dt.
\end{equation*}
Taking norms, applying the triangle inequality, and then using
\begin{equation*}
\|\nabla^2 h(\mathbf{B}(t),\boldsymbol{\Gamma}(t))\|_\mathrm{op}\leq M_h\leq L_h
\end{equation*}
for all $t\in[0,1]$ gives
\begin{equation*}
\begin{split}
\|\nabla h(\mathbf{B}_1,\boldsymbol{\Gamma}_1)-\nabla h(\mathbf{B}_2,\boldsymbol{\Gamma}_2)\|_F&\leq\int_0^1\|\nabla^2 h(\mathbf{B}(t),\boldsymbol{\Gamma}(t))\|_\mathrm{op}\,dt\,\|(\mathbf{B}_1,\boldsymbol{\Gamma}_1)-(\mathbf{B}_2,\boldsymbol{\Gamma}_2)\|_F \\
&\leq L_h\|(\mathbf{B}_1,\boldsymbol{\Gamma}_1)-(\mathbf{B}_2,\boldsymbol{\Gamma}_2)\|_F.
\end{split}
\end{equation*}
Next, suppose case (ii) is true:
\begin{equation*}
\|(\mathbf{B}_1,\boldsymbol{\Gamma}_1)-(\mathbf{B}_2,\boldsymbol{\Gamma}_2)\|_F>\delta.
\end{equation*}
Then, using the triangle inequality, the definition of $G_h$, and the definition of $L_h$, it holds
\begin{equation*}
\begin{split}
\|\nabla h(\mathbf{B}_1,\boldsymbol{\Gamma}_1)-\nabla h(\mathbf{B}_2,\boldsymbol{\Gamma}_2)\|_F&\leq\|\nabla h(\mathbf{B}_1,\boldsymbol{\Gamma}_1)\|_F+\|\nabla h(\mathbf{B}_2,\boldsymbol{\Gamma}_2)\|_F \\
&\leq 2G_h \\
&\leq\frac{2G_h}{\delta}\|(\mathbf{B}_1,\boldsymbol{\Gamma}_1)-(\mathbf{B}_2,\boldsymbol{\Gamma}_2)\|_F \\
&\leq L_h\|(\mathbf{B}_1,\boldsymbol{\Gamma}_1)-(\mathbf{B}_2,\boldsymbol{\Gamma}_2)\|_F.
\end{split}
\end{equation*}
Combining cases (i) and (ii), we conclude that $\nabla h(\mathbf{B},\boldsymbol{\Gamma})$ is $L_h$-Lipschitz on $\Omega$.

\end{proof}

We are now ready to prove Proposition~\ref{prop:descent}.

\begin{proof}

Since $\mathbb{D}$ is open relative to $\mathbb{R}^{p\times p}\times\mathbb{R}_+^{p\times p}$ and $\Omega$ is compact, there exists $\delta>0$ such that the compact set
\begin{equation*}
\Omega_\delta:=\left\{(\mathbf{B},\boldsymbol{\Gamma})\in\mathbb{R}^{p\times p}\times\mathbb{R}_+^{p\times p}:\inf_{(\tilde{\mathbf{B}},\tilde{\boldsymbol{\Gamma}})\in\Omega}\|(\mathbf{B},\boldsymbol{\Gamma})-(\tilde{\mathbf{B}},\tilde{\boldsymbol{\Gamma}})\|_F\leq\delta\right\}\subset\mathbb{D}.
\end{equation*}
Applying Lemmas~\ref{lemma:llipschitz} and \ref{lemma:hlipschitz} to the compact set $\Omega_\delta$ and then using the triangle inequality, there exists a finite constant $L_\delta<\infty$ such that
\begin{equation*}
\|\nabla g_\mu(\mathbf{B}_1,\boldsymbol{\Gamma}_1)-\nabla g_\mu(\mathbf{B}_2,\boldsymbol{\Gamma}_2)\|_F\leq L_\delta\|(\mathbf{B}_1,\boldsymbol{\Gamma}_1)-(\mathbf{B}_2,\boldsymbol{\Gamma}_2)\|_F
\end{equation*}
for all $(\mathbf{B}_1,\boldsymbol{\Gamma}_1),(\mathbf{B}_2,\boldsymbol{\Gamma}_2)\in\Omega_\delta$. Also, by continuity of $g_\mu$ and $\nabla g_\mu$ and compactness of $\Omega\times\Omega$, it holds
\begin{equation*}
C:=\sup_{(\mathbf{B}_1,\boldsymbol{\Gamma}_1),(\mathbf{B}_2,\boldsymbol{\Gamma}_2)\in\Omega}\left|g_\mu(\mathbf{B}_1,\boldsymbol{\Gamma}_1)-g_\mu(\mathbf{B}_2,\boldsymbol{\Gamma}_2)-\left\langle\nabla g_\mu(\mathbf{B}_2,\boldsymbol{\Gamma}_2),(\mathbf{B}_1,\boldsymbol{\Gamma}_1)-(\mathbf{B}_2,\boldsymbol{\Gamma}_2)\right\rangle\right|<\infty.
\end{equation*}
Define the quantity
\begin{equation*}
L:=\max\left\{L_\delta,\frac{2C}{\delta^2}\right\}.
\end{equation*}
Next, take any $(\mathbf{B}_1,\boldsymbol{\Gamma}_1),(\mathbf{B}_2,\boldsymbol{\Gamma}_2)\in\Omega$. First, suppose case (i) is true:
\begin{equation*}
\|(\mathbf{B}_1,\boldsymbol{\Gamma}_1)-(\mathbf{B}_2,\boldsymbol{\Gamma}_2)\|_F\leq\delta.
\end{equation*}
For $t\in[0,1]$ define the line segment
\begin{equation*}
(\mathbf{B}(t),\boldsymbol{\Gamma}(t)):=(\mathbf{B}_2,\boldsymbol{\Gamma}_2)+t\{(\mathbf{B}_1,\boldsymbol{\Gamma}_1)-(\mathbf{B}_2,\boldsymbol{\Gamma}_2)\}.
\end{equation*}
It then follows
\begin{equation*}
\inf_{(\tilde{\mathbf{B}},\tilde{\boldsymbol{\Gamma}})\in\Omega}\|(\mathbf{B}(t),\boldsymbol{\Gamma}(t))-(\tilde{\mathbf{B}},\tilde{\boldsymbol{\Gamma}})\|_F\leq\|(\mathbf{B}(t),\boldsymbol{\Gamma}(t))-(\mathbf{B}_2,\boldsymbol{\Gamma}_2)\|_F=t\|(\mathbf{B}_1,\boldsymbol{\Gamma}_1)-(\mathbf{B}_2,\boldsymbol{\Gamma}_2)\|_F\leq\delta
\end{equation*}
for all $t\in[0,1]$, and hence $(\mathbf{B}(t),\boldsymbol{\Gamma}(t))\in\Omega_\delta$ for all $t\in[0,1]$. Therefore, it holds
\begin{equation*}
\begin{split}
&g_\mu(\mathbf{B}_1,\boldsymbol{\Gamma}_1)-g_\mu(\mathbf{B}_2,\boldsymbol{\Gamma}_2)-\left\langle\nabla g_\mu(\mathbf{B}_2,\boldsymbol{\Gamma}_2),(\mathbf{B}_1,\boldsymbol{\Gamma}_1)-(\mathbf{B}_2,\boldsymbol{\Gamma}_2)\right\rangle \\
&\hspace{1.5in}=\int_0^1\left\langle\nabla g_\mu(\mathbf{B}(t),\boldsymbol{\Gamma}(t))-\nabla g_\mu(\mathbf{B}_2,\boldsymbol{\Gamma}_2),(\mathbf{B}_1,\boldsymbol{\Gamma}_1)-(\mathbf{B}_2,\boldsymbol{\Gamma}_2)\right\rangle\,dt \\
&\hspace{1.5in}\leq\int_0^1\|\nabla g_\mu(\mathbf{B}(t),\boldsymbol{\Gamma}(t))-\nabla g_\mu(\mathbf{B}_2,\boldsymbol{\Gamma}_2)\|_F\,dt\,\|(\mathbf{B}_1,\boldsymbol{\Gamma}_1)-(\mathbf{B}_2,\boldsymbol{\Gamma}_2)\|_F \\
&\hspace{1.5in}\leq\int_0^1 L_\delta t\,dt\,\|(\mathbf{B}_1,\boldsymbol{\Gamma}_1)-(\mathbf{B}_2,\boldsymbol{\Gamma}_2)\|_F^2 \\
&\hspace{1.5in}=\frac{L_\delta}{2}\|(\mathbf{B}_1,\boldsymbol{\Gamma}_1)-(\mathbf{B}_2,\boldsymbol{\Gamma}_2)\|_F^2 \\
&\hspace{1.5in}\leq\frac{L}{2}\|(\mathbf{B}_1,\boldsymbol{\Gamma}_1)-(\mathbf{B}_2,\boldsymbol{\Gamma}_2)\|_F^2.
\end{split}
\end{equation*}
Rearranging the above inequality gives
\begin{equation}
\label{eq:case1descent}
g_\mu(\mathbf{B}_1,\boldsymbol{\Gamma}_1)\leq g_\mu(\mathbf{B}_2,\boldsymbol{\Gamma}_2)+\left\langle\nabla g_\mu(\mathbf{B}_2,\boldsymbol{\Gamma}_2),(\mathbf{B}_1,\boldsymbol{\Gamma}_1)-(\mathbf{B}_2,\boldsymbol{\Gamma}_2)\right\rangle+\frac{L}{2}\|(\mathbf{B}_1,\boldsymbol{\Gamma}_1)-(\mathbf{B}_2,\boldsymbol{\Gamma}_2)\|_F^2.
\end{equation}
Next, suppose case (ii) is true:
\begin{equation*}
\|(\mathbf{B}_1,\boldsymbol{\Gamma}_1)-(\mathbf{B}_2,\boldsymbol{\Gamma}_2)\|_F>\delta.
\end{equation*}
Then, by the definition of $C$ and the definition of $L$, it holds
\begin{equation*}
\begin{split}
g_\mu(\mathbf{B}_1,\boldsymbol{\Gamma}_1)-g_\mu(\mathbf{B}_2,\boldsymbol{\Gamma}_2)-\left\langle\nabla g_\mu(\mathbf{B}_2,\boldsymbol{\Gamma}_2),(\mathbf{B}_1,\boldsymbol{\Gamma}_1)-(\mathbf{B}_2,\boldsymbol{\Gamma}_2)\right\rangle&\leq C \\
&\leq\frac{C}{\delta^2}\|(\mathbf{B}_1,\boldsymbol{\Gamma}_1)-(\mathbf{B}_2,\boldsymbol{\Gamma}_2)\|_F^2 \\
&\leq\frac{L}{2}\|(\mathbf{B}_1,\boldsymbol{\Gamma}_1)-(\mathbf{B}_2,\boldsymbol{\Gamma}_2)\|_F^2.
\end{split}
\end{equation*}
Again, rearranging the above inequality gives
\begin{equation}
\label{eq:case2descent}
g_\mu(\mathbf{B}_1,\boldsymbol{\Gamma}_1)\leq g_\mu(\mathbf{B}_2,\boldsymbol{\Gamma}_2)+\left\langle\nabla g_\mu(\mathbf{B}_2,\boldsymbol{\Gamma}_2),(\mathbf{B}_1,\boldsymbol{\Gamma}_1)-(\mathbf{B}_2,\boldsymbol{\Gamma}_2)\right\rangle+\frac{L}{2}\|(\mathbf{B}_1,\boldsymbol{\Gamma}_1)-(\mathbf{B}_2,\boldsymbol{\Gamma}_2)\|_F^2.
\end{equation}
Combining \eqref{eq:case1descent} and \eqref{eq:case2descent} and taking $(\mathbf{B}_1,\boldsymbol{\Gamma}_1)=(\mathbf{B}^+,\boldsymbol{\Gamma}^+)$ and $(\mathbf{B}_2,\boldsymbol{\Gamma}_2)=(\mathbf{B},\boldsymbol{\Gamma})$ completes the proof.
\end{proof}

\section{Proof of Theorem~\ref{thrm:convergence}}
\label{app:convergence}

\begin{proof}

We break the proof into two parts, addressing the two claims of the theorem in turn.

\paragraph{Part 1: Convergence of objective values}

We begin by proving the first claim of the theorem that the objective values are decreasing and convergent. Our proof works by upper bounding $g_\mu$ and $r_\mu$ to attain a sufficient decrease property for $f_\mu$. First, recall that Proposition~\ref{prop:descent} provides an upper bound for $g_\mu$ as
\begin{equation}
\label{eq:gdecrease}
\begin{split}
&g_\mu(\mathbf{B}^{(k+1)},\boldsymbol{\Gamma}^{(k+1)})\leq g_\mu(\mathbf{B}^{(k)},\boldsymbol{\Gamma}^{(k)})+\left\langle\nabla g_\mu(\mathbf{B}^{(k)},\boldsymbol{\Gamma}^{(k)}),(\Delta\mathbf{B}^{(k)},\Delta\boldsymbol{\Gamma}^{(k)})\right\rangle \\
&\hspace{4in}+\frac{L}{2}\|(\Delta\mathbf{B}^{(k)},\Delta\boldsymbol{\Gamma}^{(k)})\|_F^2,
\end{split}
\end{equation}
where
\begin{equation*}
\Delta\mathbf{B}^{(k)}:=\mathbf{B}^{(k+1)}-\mathbf{B}^{(k)},\qquad\Delta\boldsymbol{\Gamma}^{(k)}:=\boldsymbol{\Gamma}^{(k+1)}-\boldsymbol{\Gamma}^{(k)}.
\end{equation*}
To bound $r_\mu$, recall that the proximal gradient updates
\begin{equation*}
\mathbf{B}^{(k+1)}\gets\operatorname{soft}_{\alpha\mu\lambda_1}\left(\mathbf{B}^{(k)}-\alpha\nabla_\mathbf{B}g_\mu(\mathbf{B}^{(k)},\boldsymbol{\Gamma}^{(k)})\right)
\end{equation*}
and
\begin{equation*}
\boldsymbol{\Gamma}^{(k+1)}\gets\operatorname{soft}_{\alpha\mu\lambda_2}^+\left(\boldsymbol{\Gamma}^{(k)}-\alpha\nabla_{\boldsymbol{\Gamma}}g_\mu(\mathbf{B}^{(k)},\boldsymbol{\Gamma}^{(k)})\right)
\end{equation*}
are the minimizers of the proximal quadratic model
\begin{equation*}
Q(\mathbf{B},\boldsymbol{\Gamma}):=\left\langle\nabla g_\mu(\mathbf{B}^{(k)},\boldsymbol{\Gamma}^{(k)}),(\mathbf{B},\boldsymbol{\Gamma})-(\mathbf{B}^{(k)},\boldsymbol{\Gamma}^{(k)})\right\rangle+\frac{1}{2\alpha}\|(\mathbf{B},\boldsymbol{\Gamma})-(\mathbf{B}^{(k)},\boldsymbol{\Gamma}^{(k)})\|_F^2+r_\mu(\mathbf{B},\boldsymbol{\Gamma}).
\end{equation*}
Hence, by virtue of the fact that $\mathbf{B}^{(k+1)}$ and $\boldsymbol{\Gamma}^{(k+1)}$ minimize $Q(\mathbf{B},\boldsymbol{\Gamma})$, it holds
\begin{equation}
\label{eq:Qinequality}
Q(\mathbf{B}^{(k+1)},\boldsymbol{\Gamma}^{(k+1)})\leq Q(\mathbf{B}^{(k)},\boldsymbol{\Gamma}^{(k)}).
\end{equation}
Directly evaluating the expression for $Q(\mathbf{B}^{(k)},\boldsymbol{\Gamma}^{(k)})$ on the right-hand side of \eqref{eq:Qinequality} yields
\begin{equation*}
Q(\mathbf{B}^{(k)},\boldsymbol{\Gamma}^{(k)})=r_\mu(\mathbf{B}^{(k)},\boldsymbol{\Gamma}^{(k)}).
\end{equation*}
Likewise, evaluating the expression for $Q(\mathbf{B}^{(k+1)},\boldsymbol{\Gamma}^{(k+1)})$ on the left-hand side of \eqref{eq:Qinequality} and subsequently rearranging terms yields
\begin{equation}
\label{eq:rdecrease}
\begin{split}
&r_\mu(\mathbf{B}^{(k+1)},\boldsymbol{\Gamma}^{(k+1)})\leq r_\mu(\mathbf{B}^{(k)},\boldsymbol{\Gamma}^{(k)})-\left\langle\nabla g_\mu(\mathbf{B}^{(k)},\boldsymbol{\Gamma}^{(k)}),(\Delta\mathbf{B}^{(k)},\Delta\boldsymbol{\Gamma}^{(k)})\right\rangle \\
&\hspace{4in}-\frac{1}{2\alpha}\|(\Delta\mathbf{B}^{(k)},\Delta\boldsymbol{\Gamma}^{(k)})\|_F^2.
\end{split}
\end{equation}
Finally, to bound the objective function $f_\mu$ we add \eqref{eq:gdecrease} and \eqref{eq:rdecrease} and cancel the inner product terms to get
\begin{equation}
\label{eq:diff}
f_\mu(\mathbf{B}^{(k+1)},\boldsymbol{\Gamma}^{(k+1)})\leq f_\mu(\mathbf{B}^{(k)},\boldsymbol{\Gamma}^{(k)})-\left(\frac{1}{2\alpha}-\frac{L}{2}\right)\|(\Delta\mathbf{B}^{(k)},\Delta\boldsymbol{\Gamma}^{(k)})\|_F^2.
\end{equation}
By assumption, the step size $\alpha<1/L$, so $1/(2\alpha)-L/2>0$. It immediately follows
\begin{equation*}
f_\mu(\mathbf{B}^{(k+1)},\boldsymbol{\Gamma}^{(k+1)})\leq f_\mu(\mathbf{B}^{(k)},\boldsymbol{\Gamma}^{(k)}).
\end{equation*}
Finally, since a continuous function on a compact set attains its minimum, $f_\mu$ is bounded below on $\Omega$. Moreover, because a monotone decreasing sequence bounded below converges, the sequence of objective values must converge, proving the first claim of the theorem.

\paragraph{Part 2: Vanishing difference in iterates}

We now prove the second claim of the theorem that the difference in iterates vanishes. Define the constant $c:=1/(2\alpha)-L/2$. As before, since $\alpha<1/L$, we have $c>0$. Hence, for every $k$, it holds from \eqref{eq:diff}
\begin{equation*}
c\|(\Delta\mathbf{B}^{(k)},\Delta\boldsymbol{\Gamma}^{(k)})\|_F^2\leq f_\mu(\mathbf{B}^{(k)},\boldsymbol{\Gamma}^{(k)})-f_\mu(\mathbf{B}^{(k+1)},\boldsymbol{\Gamma}^{(k+1)}).
\end{equation*}
Summing both sides from $k=0$ to $K-1$ yields the bound
\begin{equation*}
c\sum_{k=0}^{K-1}\|(\Delta\mathbf{B}^{(k)},\Delta\boldsymbol{\Gamma}^{(k)})\|_F^2\leq f_\mu(\mathbf{B}^{(0)},\boldsymbol{\Gamma}^{(0)})-f_\mu(\mathbf{B}^{(K)},\boldsymbol{\Gamma}^{(K)}).
\end{equation*}
Now, define the quantity
\begin{equation*}
f_\mu^\mathrm{inf}:=\inf_{(\mathbf{B},\boldsymbol{\Gamma})\in\Omega}f_\mu(\mathbf{B},\boldsymbol{\Gamma})>-\infty.
\end{equation*}
From the fact that $f_\mu(\mathbf{B}^{(K)},\boldsymbol{\Gamma}^{(K)})\geq f_\mu^\mathrm{inf}$, we get
\begin{equation*}
c\sum_{k=0}^{K-1}\|(\Delta\mathbf{B}^{(k)},\Delta\boldsymbol{\Gamma}^{(k)})\|_F^2\leq f_\mu(\mathbf{B}^{(0)},\boldsymbol{\Gamma}^{(0)})-f_\mu^\mathrm{inf}.
\end{equation*}
Thus, the sums on the left-hand side are uniformly bounded in $K$. Taking $K\to\infty$, we obtain
\begin{equation*}
\sum_{k=0}^\infty\|(\Delta\mathbf{B}^{(k)},\Delta\boldsymbol{\Gamma}^{(k)})\|_F^2<\infty.
\end{equation*}
Since this series is convergent and has nonnegative terms, its terms must converge to zero. Hence, it must hold
\begin{equation*}
\|(\Delta\mathbf{B}^{(k)},\Delta\boldsymbol{\Gamma}^{(k)})\|_F^2\to0,
\end{equation*}
from which the second claim of the theorem immediately follows.

\end{proof}

\section{Proof of Theorem~\ref{thrm:identifiability}}
\label{app:identifiability}

\begin{proof}
Fix one observation row from one cluster and suppress the cluster and row indices, writing the resulting random vector as $(x_1,\dots,x_p)^\top$. Since the joint distribution of the clustered data determines the marginal distribution of this vector, it suffices to recover $(\mathbf{B}_0,\boldsymbol{\Gamma}_0)$ from that marginal distribution.

Because $\mathcal{G}(\mathbf{B}_0)\cup\mathcal{G}(\boldsymbol{\Gamma}_0)$ is a DAG, there exists a topological ordering of the nodes. Fix one such ordering for notational convenience and index the variables accordingly. Then we may write
\begin{equation*}
x_k=\sum_{j<k}(\beta_{jk,0}+u_{jk})x_j+\varepsilon_k,\qquad k=1,\dots,p,
\end{equation*}
where the random-effects $u_{jk}\sim\mathrm{N}(0,\gamma_{jk,0})$ are independent, the noise terms $\varepsilon_k\sim\mathrm{N}(0,1)$ are independent, and all random effects are independent of all noise terms. Conditional on the random effects, $(x_1,\dots,x_p)$ is Gaussian with positive definite covariance, so its marginal distribution has a strictly positive density on $\mathbb{R}^p$.

A node is a source (i.e., it has no parents) if and only if its variance equals one. Indeed, if node $k$ is a source, then $x_k=\varepsilon_k$, so $\operatorname{Var}(x_k)=1$. Conversely, if node $k$ is not a source, then at least one coefficient $\beta_{jk,0}+u_{jk}$ with $j<k$ is nonzero almost surely. Since the random effects entering node $k$ do not appear in the structural equations for $x_1,\dots,x_{k-1}$, they are independent of $(x_1,\dots,x_{k-1})$. Therefore, the variance of $x_k$ satisfies
\begin{equation*}
\operatorname{Var}(x_k)=\operatorname{Var}\left(\sum_{j<k}(\beta_{jk,0}+u_{jk})x_j\right)+1>1,
\end{equation*}
because $(x_1,\dots,x_{k-1})$ has positive definite covariance and the coefficients are not zero almost surely. Thus, the source set is determined by the observed distribution.

Choose any source and relabel it as node one. Then $x_1=\varepsilon_1$, so $x_1$ is independent of all remaining noise terms and all random effects. For $k=2,\dots,p$, we have
\begin{equation*}
x_k=(\beta_{1k,0}+u_{1k})x_1+\sum_{1<j<k}(\beta_{jk,0}+u_{jk})x_j+\varepsilon_k.
\end{equation*}
Thus, conditional on $x_1=0$, the vector $(x_2,\dots,x_p)$ again follows a mixed-effects DAG model of the same form, with node one removed. Applying the previous variance characterization recursively to these conditional distributions identifies one source after another, and therefore allows us to construct a topological ordering of $\mathcal{G}(\mathbf{B}_0)\cup\mathcal{G}(\boldsymbol{\Gamma}_0)$.

After constructing a topological ordering in the manner above, the parameters are recoverable node by node. For each $k=1,\dots,p$ and any $x_1,\dots,x_{k-1}\in\mathbb{R}$, the conditional distribution of $x_k$ given $(x_1,\dots,x_{k-1})$ satisfies
\begin{equation*}
x_k\mid(x_1,\dots,x_{k-1})\sim\mathrm{N}\left(\sum_{j<k}\beta_{jk,0}x_j,\sum_{j<k}\gamma_{jk,0}x_j^2+1\right),
\end{equation*}
because $u_{1k},\dots,u_{k-1,k}$ and $\varepsilon_k$ are independent mean-zero Gaussians and the $u_{jk}$ are independent of $(x_1,\dots,x_{k-1})$. Since $(x_1,\dots,x_{k-1})$ has strictly positive density on $\mathbb{R}^{k-1}$, the conditional mean uniquely determines the fixed effects $\beta_{jk,0}$ for $j<k$, and the conditional variance uniquely determines the random-effect variances $\gamma_{jk,0}$ for $j<k$. Applying this argument for $k=1,\dots,p$ shows that every entry of $\mathbf{B}_0$ and $\boldsymbol{\Gamma}_0$ is uniquely determined by the observed distribution. Therefore, $(\mathbf{B}_0,\boldsymbol{\Gamma}_0)$ is identifiable, and so are $\mathcal{G}(\mathbf{B}_0)$, $\mathcal{G}(\boldsymbol{\Gamma}_0)$, and their union.
\end{proof}

\section{Proof of Theorem~\ref{thrm:consistency}}
\label{app:consistency}

The proof requires three technical lemmas. All lemmas are stated and proved under the assumptions of Theorem~\ref{thrm:consistency}. Lemma~\ref{lemma:unique} first shows that the population loss is uniquely minimized at the true parameters.

\begin{lemma}
\label{lemma:unique}
For all $(\mathbf{B},\boldsymbol{\Gamma})\in\mathcal{F}$, where $\mathcal{F}$ is defined in \eqref{eq:feasiblebox}, the population loss satisfies the inequality
\begin{equation*}
L(\mathbf{B},\boldsymbol{\Gamma})\geq L(\mathbf{B}_0,\boldsymbol{\Gamma}_0),
\end{equation*}
holding with equality if and only if $(\mathbf{B},\boldsymbol{\Gamma})=(\mathbf{B}_0,\boldsymbol{\Gamma}_0)$.
\end{lemma}
\begin{proof}
Let $p_{\mathbf{B},\boldsymbol{\Gamma}}(\mathbf{X}^{(1)})$ denote the cluster-level density under $(\mathbf{B},\boldsymbol{\Gamma})$. Since $\ell_m(\mathbf{B},\boldsymbol{\Gamma})$ is exactly the average negative log-likelihood (after restoring additive and multiplicative constants), taking expectation gives
\begin{equation*}
L(\mathbf{B},\boldsymbol{\Gamma})=-\mathrm{E}\left[\log\left\{p_{\mathbf{B},\boldsymbol{\Gamma}}(\mathbf{X}^{(1)})\right\}\right],\qquad L(\mathbf{B}_0,\boldsymbol{\Gamma}_0)=-\mathrm{E}\left[\log\left\{p_{\mathbf{B}_0,\boldsymbol{\Gamma}_0}(\mathbf{X}^{(1)})\right\}\right].
\end{equation*}
Taking the difference of the two terms above yields
\begin{equation*}
L(\mathbf{B},\boldsymbol{\Gamma})-L(\mathbf{B}_0,\boldsymbol{\Gamma}_0)=
\mathrm{E}\left[\log\left\{\frac{p_{\mathbf{B}_0,\boldsymbol{\Gamma}_0}(\mathbf{X}^{(1)})}{p_{\mathbf{B},\boldsymbol{\Gamma}}(\mathbf{X}^{(1)})}\right\}\right].
\end{equation*}
The right-hand side is the Kullback--Leibler divergence from $p_{\mathbf{B}_0,\boldsymbol{\Gamma}_0}$ to $p_{\mathbf{B},\boldsymbol{\Gamma}}$ and is therefore nonnegative. Equality holds if and only if $p_{\mathbf{B},\boldsymbol{\Gamma}}(\mathbf{X}^{(1)})=p_{\mathbf{B}_0,\boldsymbol{\Gamma}_0}(\mathbf{X}^{(1)})$ almost surely, which by Theorem~\ref{thrm:identifiability}, implies $(\mathbf{B},\boldsymbol{\Gamma})=(\mathbf{B}_0,\boldsymbol{\Gamma}_0)$.
\end{proof}

Lemma~\ref{lemma:uniformconvergence} provides a uniform convergence result for the empirical loss and its Hessian over compact parameter sets.

\begin{lemma}
\label{lemma:uniformconvergence}
Let $K\subset\mathbb{R}^{p\times p}\times\mathbb{R}_+^{p\times p}$ be compact. Then it holds
\begin{equation*}
\sup_{(\mathbf{B},\boldsymbol{\Gamma})\in K}\left|\ell_m(\mathbf{B},\boldsymbol{\Gamma})-L(\mathbf{B},\boldsymbol{\Gamma})\right|\overset{p}{\to}0
\end{equation*}
and
\begin{equation*}
\sup_{(\mathbf{B},\boldsymbol{\Gamma})\in K}\left\|\nabla^2\ell_m(\mathbf{B},\boldsymbol{\Gamma})-\nabla^2L(\mathbf{B},\boldsymbol{\Gamma})\right\|_{\mathrm{op}}\overset{p}{\to}0.
\end{equation*}
\end{lemma}
\begin{proof}
For each cluster $i$, write $a(\mathbf{X}^{(i)},\mathbf{B},\boldsymbol{\Gamma})$ for the contribution of cluster $i$ to the negative log-likelihood, so that
\begin{equation*}
\ell_m(\mathbf{B},\boldsymbol{\Gamma})=\frac{1}{m}\sum_{i=1}^ma(\mathbf{X}^{(i)},\mathbf{B},\boldsymbol{\Gamma}).
\end{equation*}
Because $\boldsymbol{\Gamma}\geq \mathbf{0}$, each covariance matrix
\begin{equation*}
\mathbf{V}^{(i)}(\boldsymbol{\gamma}_k)=\mathbf{I}+\mathbf{X}^{(i)}\operatorname{diag}(\boldsymbol{\gamma}_k)\mathbf{X}^{(i)\top}\succeq\mathbf{I},
\end{equation*}
so the log-determinant and inverse terms appearing in $a(\mathbf{X}^{(i)},\mathbf{B},\boldsymbol{\Gamma})$ are well-defined for every $(\mathbf{B},\boldsymbol{\Gamma})\in K$. Consequently, for almost every $\mathbf{X}^{(i)}$, the maps
\begin{equation*}
(\mathbf{B},\boldsymbol{\Gamma})\mapsto a(\mathbf{X}^{(i)},\mathbf{B},\boldsymbol{\Gamma}),\qquad(\mathbf{B},\boldsymbol{\Gamma})\mapsto\nabla^2 a(\mathbf{X}^{(i)},\mathbf{B},\boldsymbol{\Gamma}),
\end{equation*}
are continuous on $K$. Next, since $K$ is compact and $n_i\equiv n$, $a(\mathbf{X}^{(i)},\mathbf{B},\boldsymbol{\Gamma})$ and every scalar component of $\nabla^2 a(\mathbf{X}^{(i)},\mathbf{B},\boldsymbol{\Gamma})$ is bounded uniformly over $(\mathbf{B},\boldsymbol{\Gamma})\in K$ by a polynomial in \(\|\mathbf{X}^{(i)}\|_F\). Under the Gaussian structural equation model, $\mathbf{X}^{(i)}$ has finite moments of all orders, so these bounds have finite expectation. The same bounds justify differentiation under the expectation, and therefore
\begin{equation*}
\nabla^2L(\mathbf{B},\boldsymbol{\Gamma})=\mathrm{E}\left[\nabla^2 a(\mathbf{X}^{(1)},\mathbf{B},\boldsymbol{\Gamma})\right].
\end{equation*}
Thus, the uniform law of large numbers \citep[e.g.,][Lemma~2.4]{Newey1994}, applied on the compact set $K$ using the continuity and polynomial bounds above, gives
\begin{equation*}
\sup_{(\mathbf{B},\boldsymbol{\Gamma})\in K}\left|\ell_m(\mathbf{B},\boldsymbol{\Gamma})-L(\mathbf{B},\boldsymbol{\Gamma})\right|\overset{p}{\to}0.
\end{equation*}
The same result also gives entrywise uniform convergence of $\nabla^2\ell_m(\mathbf{B},\boldsymbol{\Gamma})$ to $\nabla^2L(\mathbf{B},\boldsymbol{\Gamma})$ on $K$. Since $p$ is fixed, the Hessian has finitely many entries, and it follows
\begin{equation*}
\sup_{(\mathbf{B},\boldsymbol{\Gamma})\in K}\left\|\nabla^2\ell_m(\mathbf{B},\boldsymbol{\Gamma})-\nabla^2L(\mathbf{B},\boldsymbol{\Gamma})\right\|_{\mathrm{op}}\overset{p}{\to}0.
\end{equation*}
\end{proof}

Finally, Lemma~\ref{lemma:localstrongconvexity} establishes local strong convexity of the empirical loss in a neighborhood of the true parameter values.

\begin{lemma}
\label{lemma:localstrongconvexity}
There exist constants $c>0$ and $\delta>0$ such that, with probability tending to one, the empirical loss $\ell_m(\mathbf{B},\boldsymbol{\Gamma})$ satisfies
\begin{equation*}
\ell_m(\mathbf{B}_1,\boldsymbol{\Gamma}_1)\geq\ell_m(\mathbf{B}_2,\boldsymbol{\Gamma}_2)+\left\langle\nabla\ell_m(\mathbf{B}_2,\boldsymbol{\Gamma}_2),(\mathbf{B}_1,\boldsymbol{\Gamma}_1)-(\mathbf{B}_2,\boldsymbol{\Gamma}_2)\right\rangle+\frac{c}{2}\|(\mathbf{B}_1,\boldsymbol{\Gamma}_1)-(\mathbf{B}_2,\boldsymbol{\Gamma}_2)\|_F^2
\end{equation*}
for all $(\mathbf{B}_1,\boldsymbol{\Gamma}_1),(\mathbf{B}_2,\boldsymbol{\Gamma}_2)\in\mathcal{F}$ such that
\begin{equation*}
\|(\mathbf{B}_1,\boldsymbol{\Gamma}_1)-(\mathbf{B}_0,\boldsymbol{\Gamma}_0)\|_F\leq\delta\quad\text{and}\quad\|(\mathbf{B}_2,\boldsymbol{\Gamma}_2)-(\mathbf{B}_0,\boldsymbol{\Gamma}_0)\|_F\leq\delta.
\end{equation*}
\end{lemma}
\begin{proof}
By Assumption~\ref{asmp:hessian}, $\nabla^2L(\mathbf{B}_0,\boldsymbol{\Gamma}_0)$ is positive definite. Let
\begin{equation*}
\lambda_0:=\lambda_{\min}\{\nabla^2L(\mathbf{B}_0,\boldsymbol{\Gamma}_0)\}>0.
\end{equation*}
Since $\nabla^2L(\mathbf{B},\boldsymbol{\Gamma})$ is continuous in a neighborhood of $(\mathbf{B}_0,\boldsymbol{\Gamma}_0)$ relative to $\mathbb{R}^{p\times p}\times\mathbb{R}_+^{p\times p}$, there exists $\delta>0$ such that
\begin{equation*}
\lambda_{\min}\{\nabla^2 L(\mathbf{B},\boldsymbol{\Gamma})\}\geq\frac{\lambda_0}{2}
\end{equation*}
whenever
\begin{equation*}
\|(\mathbf{B},\boldsymbol{\Gamma})-(\mathbf{B}_0,\boldsymbol{\Gamma}_0)\|_F\leq\delta.
\end{equation*}
Define the compact set
\begin{equation*}
K_\delta:=\left\{(\mathbf{B},\boldsymbol{\Gamma})\in\mathbb{R}^{p\times p}\times\mathbb{R}_+^{p\times p}:\|(\mathbf{B},\boldsymbol{\Gamma})-(\mathbf{B}_0,\boldsymbol{\Gamma}_0)\|_F\leq\delta\right\},
\end{equation*}
and apply Lemma~\ref{lemma:uniformconvergence} to get
\begin{equation*}
\sup_{(\mathbf{B},\boldsymbol{\Gamma})\in K_\delta}\left\|\nabla^2\ell_m(\mathbf{B},\boldsymbol{\Gamma})-\nabla^2L(\mathbf{B},\boldsymbol{\Gamma})\right\|_{\mathrm{op}}\overset{p}{\to}0.
\end{equation*}
Therefore, with probability tending to one, it holds
\begin{equation*}
\sup_{\|(\mathbf{B},\boldsymbol{\Gamma})-(\mathbf{B}_0,\boldsymbol{\Gamma}_0)\|_F\leq\delta}\left\|\nabla^2\ell_m(\mathbf{B},\boldsymbol{\Gamma})-\nabla^2L(\mathbf{B},\boldsymbol{\Gamma})\right\|_{\mathrm{op}}\leq\frac{\lambda_0}{4}.
\end{equation*}
On this event, we have
\begin{equation*}
\lambda_{\min}\{\nabla^2\ell_m(\mathbf{B},\boldsymbol{\Gamma})\}\geq\lambda_{\min}\{\nabla^2L(\mathbf{B},\boldsymbol{\Gamma})\}-\left\|\nabla^2\ell_m(\mathbf{B},\boldsymbol{\Gamma})-\nabla^2L(\mathbf{B},\boldsymbol{\Gamma})\right\|_{\mathrm{op}}\geq\frac{\lambda_0}{4}
\end{equation*}
for all $(\mathbf{B},\boldsymbol{\Gamma})$ in the neighborhood. Define $c:=\lambda_0/4$ and for any two points $(\mathbf{B}_1,\boldsymbol{\Gamma}_1)$ and $(\mathbf{B}_2,\boldsymbol{\Gamma}_2)$ in this neighborhood, define
\begin{equation*}
(\mathbf{B}(t),\boldsymbol{\Gamma}(t)):=(\mathbf{B}_2,\boldsymbol{\Gamma}_2)+t\{(\mathbf{B}_1,\boldsymbol{\Gamma}_1)-(\mathbf{B}_2,\boldsymbol{\Gamma}_2)\}
\end{equation*}
for $t\in[0,1]$. Since $K_\delta$ is convex, it holds $(\mathbf{B}(t),\boldsymbol{\Gamma}(t))\in K_\delta$ for all $t\in[0,1]$. Moreover, because $\boldsymbol{\Gamma}(t)\in\mathbb{R}_+^{p\times p}$, each covariance matrix appearing in $\ell_m(\mathbf{B}(t),\boldsymbol{\Gamma}(t))$ is positive definite, so $\ell_m(\mathbf{B}(t),\boldsymbol{\Gamma}(t))$ is twice continuously differentiable on $[0,1]$. Therefore, Taylor's theorem with the integral form of the remainder yields
\begin{equation*}
\ell_m(\mathbf{B}_1,\boldsymbol{\Gamma}_1)=\ell_m(\mathbf{B}_2,\boldsymbol{\Gamma}_2)+\left\langle\nabla \ell_m(\mathbf{B}_2,\boldsymbol{\Gamma}_2),(\mathbf{B}_1,\boldsymbol{\Gamma}_1)-(\mathbf{B}_2,\boldsymbol{\Gamma}_2)\right\rangle+\int_0^1(1-t)Q(t)\,dt,
\end{equation*}
where
\begin{equation*}
\begin{split}
Q(t)&:= \\
&\left\langle(\mathbf{B}_1,\boldsymbol{\Gamma}_1)-(\mathbf{B}_2,\boldsymbol{\Gamma}_2),\nabla^2\ell_m((\mathbf{B}_2,\boldsymbol{\Gamma}_2)+t\{(\mathbf{B}_1,\boldsymbol{\Gamma}_1)-(\mathbf{B}_2,\boldsymbol{\Gamma}_2)\})\{(\mathbf{B}_1,\boldsymbol{\Gamma}_1)-(\mathbf{B}_2,\boldsymbol{\Gamma}_2)\}\right\rangle.
\end{split}
\end{equation*}
Using the uniform lower bound on the Hessian gives
\begin{equation*}
Q(t)\geq c\|(\mathbf{B}_1,\boldsymbol{\Gamma}_1)-(\mathbf{B}_2,\boldsymbol{\Gamma}_2)\|_F^2
\end{equation*}
for all $t\in[0,1]$. Therefore, it holds
\begin{equation*}
\ell_m(\mathbf{B}_1,\boldsymbol{\Gamma}_1)\geq\ell_m(\mathbf{B}_2,\boldsymbol{\Gamma}_2)+\left\langle\nabla\ell_m(\mathbf{B}_2,\boldsymbol{\Gamma}_2),(\mathbf{B}_1,\boldsymbol{\Gamma}_1)-(\mathbf{B}_2,\boldsymbol{\Gamma}_2)\right\rangle+\frac{c}{2}\|(\mathbf{B}_1,\boldsymbol{\Gamma}_1)-(\mathbf{B}_2,\boldsymbol{\Gamma}_2)\|_F^2,
\end{equation*}
which completes the proof.
\end{proof}

With these results in place, we are now ready to prove Theorem~\ref{thrm:consistency}.

\begin{proof}

\paragraph{Part 1: Consistency of parameter estimation}

By the definition of $F_m$ and the fact that $\|\mathbf{B}\|_\infty\leq M_\mathbf{B}$, $\|\boldsymbol{\Gamma}\|_\infty\leq M_{\boldsymbol{\Gamma}}$, and $p$ is fixed, we have
\begin{equation*}
\sup_{(\mathbf{B},\boldsymbol{\Gamma})\in\mathcal{F}}|F_m(\mathbf{B},\boldsymbol{\Gamma})-L(\mathbf{B},\boldsymbol{\Gamma})|\leq\sup_{(\mathbf{B},\boldsymbol{\Gamma})\in\mathcal{F}}|\ell_m(\mathbf{B},\boldsymbol{\Gamma})-L(\mathbf{B},\boldsymbol{\Gamma})|+\lambda_1p^2M_\mathbf{B}+\lambda_2p^2M_{\boldsymbol{\Gamma}}.
\end{equation*}
We now show the two terms on the right-hand side converge to zero. Since $\mathcal{F}$ is compact, Lemma~\ref{lemma:uniformconvergence} implies
\begin{equation*}
\sup_{(\mathbf{B},\boldsymbol{\Gamma})\in\mathcal{F}}|\ell_m(\mathbf{B},\boldsymbol{\Gamma})-L(\mathbf{B},\boldsymbol{\Gamma})|\overset{p}{\to}0.
\end{equation*}
Moreover, by Assumption~\ref{asmp:penalty} the penalty terms on the right-hand side vanish, and hence
\begin{equation*}
z_m:=\sup_{(\mathbf{B},\boldsymbol{\Gamma})\in\mathcal{F}}|F_m(\mathbf{B},\boldsymbol{\Gamma})-L(\mathbf{B},\boldsymbol{\Gamma})|\overset{p}{\to}0.
\end{equation*}
Now, fix any $\epsilon>0$ and define the set
\begin{equation*}
\mathcal{A}_\epsilon:=\{(\mathbf{B},\boldsymbol{\Gamma})\in\mathcal{F}:\|(\mathbf{B},\boldsymbol{\Gamma})-(\mathbf{B}_0,\boldsymbol{\Gamma}_0)\|_F\geq\epsilon\}.
\end{equation*}
Since $\mathcal{F}$ is closed and bounded (since the set of acyclic matrices is closed), it is compact and hence $\mathcal{A}_\epsilon$ is also compact. Moreover, since $L$ is continuous on $\mathcal{F}$ and, by Lemma~\ref{lemma:unique}, uniquely minimized at $(\mathbf{B}_0,\boldsymbol{\Gamma}_0)$, it follows
\begin{equation*}
\Delta_\epsilon:=\inf_{(\mathbf{B},\boldsymbol{\Gamma})\in\mathcal{A}_\epsilon}\{L(\mathbf{B},\boldsymbol{\Gamma})-L(\mathbf{B}_0,\boldsymbol{\Gamma}_0)\}>0.
\end{equation*}
Now, using that $(\hat{\mathbf{B}},\hat{\boldsymbol{\Gamma}})$ is the global minimizer of $F_m$, we have
\begin{equation*}
L(\hat{\mathbf{B}},\hat{\boldsymbol{\Gamma}})\leq F_m(\hat{\mathbf{B}},\hat{\boldsymbol{\Gamma}})+z_m\leq F_m(\mathbf{B}_0,\boldsymbol{\Gamma}_0)+z_m\leq L(\mathbf{B}_0,\boldsymbol{\Gamma}_0)+2z_m.
\end{equation*}
Therefore, on the event $z_m<\Delta_\epsilon/2$, it holds
\begin{equation}
\label{eq:step1inequal}
L(\hat{\mathbf{B}},\hat{\boldsymbol{\Gamma}})<L(\mathbf{B}_0,\boldsymbol{\Gamma}_0)+\Delta_\epsilon.
\end{equation}
Suppose for the sake of contradiction
\begin{equation*}
\|(\hat{\mathbf{B}},\hat{\boldsymbol{\Gamma}})-(\mathbf{B}_0,\boldsymbol{\Gamma}_0)\|_F\geq\epsilon,
\end{equation*}
then $(\hat{\mathbf{B}},\hat{\boldsymbol{\Gamma}})\in\mathcal{A}_\epsilon$, implying
\begin{equation*}
L(\hat{\mathbf{B}},\hat{\boldsymbol{\Gamma}})\geq L(\mathbf{B}_0,\boldsymbol{\Gamma}_0)+\Delta_\epsilon,
\end{equation*}
which contradicts \eqref{eq:step1inequal}. Hence, whenever $z_m<\Delta_\epsilon/2$, it must hold
\begin{equation*}
\|(\hat{\mathbf{B}},\hat{\boldsymbol{\Gamma}})-(\mathbf{B}_0,\boldsymbol{\Gamma}_0)\|_F<\epsilon.
\end{equation*}
Since $z_m\overset{p}{\to}0$, we have
\begin{equation*}
\mathrm{Pr}\left(\|(\hat{\mathbf{B}},\hat{\boldsymbol{\Gamma}})-(\mathbf{B}_0,\boldsymbol{\Gamma}_0)\|_F\geq\epsilon\right)\leq\mathrm{Pr}(z_m\geq\Delta_\epsilon/2)\to0,
\end{equation*}
and hence
\begin{equation*}
(\hat{\mathbf{B}},\hat{\boldsymbol{\Gamma}})\overset{p}{\to}(\mathbf{B}_0,\boldsymbol{\Gamma}_0),
\end{equation*}
thereby proving the first claim of the theorem.

\paragraph{Part 2: No false negatives}

Let $(j,k)\in\mathcal{S}_\mathbf{B}$. Since $p$ is fixed, the set $\mathcal{S}_\mathbf{B}$ is finite. Define
\begin{equation*}
c_\mathbf{B}:=\min_{(a,b)\in\mathcal{S}_\mathbf{B}}|\beta_{ab,0}|>0.
\end{equation*}
If $\hat{\beta}_{jk}=0$, then $|\hat{\beta}_{jk}-\beta_{jk,0}|=|\beta_{jk,0}|\geq c_\mathbf{B}$. Hence, by Part~1, we have
\begin{equation*}
\mathrm{Pr}(\hat{\beta}_{jk}=0)\leq\mathrm{Pr}(|\hat{\beta}_{jk}-\beta_{jk,0}|\geq c_\mathbf{B})\to0.
\end{equation*}
Therefore, $\mathrm{Pr}(\hat{\beta}_{jk}\neq0)\to1$ for all $(j,k)\in\mathcal{S}_\mathbf{B}$. Since $p$ is fixed, $\mathcal{S}_\mathbf{B}$ is finite and the union bound gives
\begin{equation}
\label{eq:unionbeta}
\mathrm{Pr}\left(\exists\,(j,k)\in\mathcal{S}_\mathbf{B}:\hat{\beta}_{jk}=0\right)\leq\sum_{(j,k)\in\mathcal{S}_\mathbf{B}}\mathrm{Pr}(\hat{\beta}_{jk}=0)\to0.
\end{equation}
Similarly, let $(j,k)\in\mathcal{S}_{\boldsymbol{\Gamma}}$. Since $p$ is fixed, the set $\mathcal{S}_{\boldsymbol{\Gamma}}$ is finite. Define
\begin{equation*}
c_{\boldsymbol{\Gamma}}:=\min_{(a,b)\in\mathcal{S}_{\boldsymbol{\Gamma}}}\gamma_{ab,0}>0.
\end{equation*}
If $\hat{\gamma}_{jk}=0$, then $|\hat{\gamma}_{jk}-\gamma_{jk,0}|=\gamma_{jk,0}\geq c_{\boldsymbol{\Gamma}}$. Hence, again by Part~1, we have
\begin{equation*}
\mathrm{Pr}\left(\hat{\gamma}_{jk}=0\right)\leq\mathrm{Pr}(|\hat{\gamma}_{jk}-\gamma_{jk,0}|\geq c_{\boldsymbol{\Gamma}})\to0,
\end{equation*}
which implies $\mathrm{Pr}(\hat{\gamma}_{jk}\neq0)\to1$ for all $(j,k)\in\mathcal{S}_{\boldsymbol{\Gamma}}$. The union bound therefore gives
\begin{equation}
\label{eq:uniongamma}
\mathrm{Pr}(\exists\,(j,k)\in\mathcal{S}_{\boldsymbol{\Gamma}}:\hat{\gamma}_{jk}=0)\leq\sum_{(j,k)\in\mathcal{S}_{\boldsymbol{\Gamma}}}\mathrm{Pr}(\hat{\gamma}_{jk}=0)\to0.
\end{equation}
Combining \eqref{eq:unionbeta} and \eqref{eq:uniongamma} yields
\begin{equation*}
\mathrm{Pr}\left(\hat{\beta}_{jk}\neq0\,\forall\,(j,k)\in\mathcal{S}_\mathbf{B}\land\hat{\gamma}_{jk}\neq0\,\forall\,(j,k)\in\mathcal{S}_{\boldsymbol{\Gamma}}\right)\to1,
\end{equation*}
proving $(\hat{\mathbf{B}},\hat{\boldsymbol{\Gamma}})$ contains no false negatives.

\paragraph{Part 3: No false positives}

\subparagraph{Part 3(a)}

Define the structure-restricted feasible set
\begin{equation*}
\mathcal{F}_\mathcal{S}:=\{(\mathbf{B},\boldsymbol{\Gamma})\in\mathcal{F}:\beta_{jk}=0\,\forall\,(j,k)\in\mathcal{S}_\mathbf{B}^c,\,\gamma_{jk}=0\,\forall\,(j,k)\in\mathcal{S}_{\boldsymbol{\Gamma}}^c\}.
\end{equation*}
This set is nonempty because $(\mathbf{B}_0,\boldsymbol{\Gamma}_0)\in\mathcal{F}_\mathcal{S}$. Since $\mathcal{G}(\mathbf{B}_0)\cup\mathcal{G}(\boldsymbol{\Gamma}_0)$ is a DAG, any spanning subgraph (i.e., a graph obtained by deleting edges) is also a DAG. Therefore, every element of $\mathcal{F}_\mathcal{S}$ is acyclic. In particular, for any $(\mathbf{B},\boldsymbol{\Gamma})\in\mathcal{F}_\mathcal{S}$, it holds
\begin{equation*}
\mathcal{G}(\mathbf{B})\cup\mathcal{G}(\boldsymbol{\Gamma})\subseteq\mathcal{G}(\mathbf{B}_0)\cup\mathcal{G}(\boldsymbol{\Gamma}_0).
\end{equation*}

Let $(\tilde{\mathbf{B}},\tilde{\boldsymbol{\Gamma}})$ be a global minimizer of $F_m$ over $\mathcal{F}_\mathcal{S}$. We work on the event that $(\tilde{\mathbf{B}},\tilde{\boldsymbol{\Gamma}})$ is close enough to $(\mathbf{B}_0,\boldsymbol{\Gamma}_0)$ such that:
\begin{itemize}
\item $\operatorname{sign}(\tilde{\beta}_{jk})=\operatorname{sign}(\beta_{jk,0})$ and $|
\tilde{\beta}_{jk}|<M_\mathbf{B}$ for all $(j,k)\in\mathcal{S}_\mathbf{B}$; and
\item $0<\tilde{\gamma}_{jk}<M_{\boldsymbol{\Gamma}}$ for all $(j,k)\in\mathcal{S}_{\boldsymbol{\Gamma}}$.
\end{itemize}
This event occurs with probability tending to one due to the same line of reasoning as in Part~1, but on $\mathcal{F}_\mathcal{S}$, together with the inequalities $\|\mathbf{B}_0\|_\infty<M_\mathbf{B}$ and $\|\boldsymbol{\Gamma}_0\|_\infty<M_{\boldsymbol{\Gamma}}$.

Define the gradients on the nonzero components as
\begin{equation*}
\nabla_{\mathbf{B},\mathcal{S}_\mathbf{B}}\ell_m(\mathbf{B},\boldsymbol{\Gamma}):=\left(\frac{\partial}{\partial\beta_{jk}}\ell_m(\mathbf{B},\boldsymbol{\Gamma})\right)_{(j,k)\in\mathcal{S}_\mathbf{B}}
\end{equation*}
and
\begin{equation*}
\nabla_{\boldsymbol{\Gamma},\mathcal{S}_{\boldsymbol{\Gamma}}}\ell_m(\mathbf{B},\boldsymbol{\Gamma}):=\left(\frac{\partial}{\partial\gamma_{jk}}\ell_m(\mathbf{B},\boldsymbol{\Gamma})\right)_{(j,k)\in\mathcal{S}_{\boldsymbol{\Gamma}}},
\end{equation*}
which can be stacked together as
\begin{equation*}
\nabla_\mathcal{S}\ell_m(\mathbf{B},\boldsymbol{\Gamma}):=
\begin{pmatrix}
\nabla_{\mathbf{B},\mathcal{S}_\mathbf{B}}\ell_m(\mathbf{B},\boldsymbol{\Gamma}) \\
\nabla_{\boldsymbol{\Gamma},\mathcal{S}_{\boldsymbol{\Gamma}}}\ell_m(\mathbf{B},\boldsymbol{\Gamma})
\end{pmatrix}.
\end{equation*}
On the event above, the acyclicity constraint is automatic on $\mathcal{F}_\mathcal{S}$ (each such union graph is a spanning subgraph of the true union DAG), the box constraints are inactive on the nonzero components, and the nonnegative constraint on $\boldsymbol{\Gamma}$ is inactive on $\mathcal{S}_{\boldsymbol{\Gamma}}$. The first-order conditions for $(\tilde{\mathbf{B}},\tilde{\boldsymbol{\Gamma}})$ on the nonzero components are
\begin{equation}
\label{eq:focbeta}
\nabla_{\mathbf{B},\mathcal{S}_\mathbf{B}}\ell_m(\tilde{\mathbf{B}},\tilde{\boldsymbol{\Gamma}})+\lambda_1 \mathbf{s}_\mathbf{B}=\mathbf{0}
\end{equation}
and, since $\gamma_{jk}\geq0$ and nonzero $\tilde{\gamma}_{jk}>0$,
\begin{equation}
\label{eq:focgamma}
\nabla_{\boldsymbol{\Gamma},\mathcal{S}_{\boldsymbol{\Gamma}}}\ell_m(\tilde{\mathbf{B}},\tilde{\boldsymbol{\Gamma}})+\lambda_2\mathbf{1}=\mathbf{0}.
\end{equation}
Stacking these equations together gives
\begin{equation}
\label{eq:foc}
\nabla_{\mathcal{S}}\ell_m(\tilde{\mathbf{B}},\tilde{\boldsymbol{\Gamma}})+\lambda_1
\begin{pmatrix}
\mathbf{s}_\mathbf{B} \\
\lambda_2/\lambda_1\mathbf{1}
\end{pmatrix}
=\mathbf{0}.
\end{equation}

\subparagraph{Part 3(b)}

Define the sets of acyclicity-preserving inactive indices
\begin{equation*}
\mathcal{T}_{\mathbf{B}}:=\left\{(j,k)\in\mathcal{S}_{\mathbf{B}}^c:\mathcal{G}(\mathbf{B}_0)\cup\mathcal{G}(\boldsymbol{\Gamma}_0)\cup\{j\to k\}\in\mathrm{DAG}_p\right\}
\end{equation*}
and
\begin{equation*}
\mathcal{T}_{\boldsymbol{\Gamma}}:=\left\{(j,k)\in\mathcal{S}_{\boldsymbol{\Gamma}}^c:\mathcal{G}(\mathbf{B}_0)\cup\mathcal{G}(\boldsymbol{\Gamma}_0)\cup\{j\to k\}\in\mathrm{DAG}_p\right\},
\end{equation*}
where $\{j\to k\}$ denotes the directed graph with node set $\{1,\dots,p\}$ and edge set $\{(j,k)\}$. We show that the strict dual feasibility conditions hold with probability tending to one on the acyclicity-preserving inactive sets defined above. For indices $\mathcal{S}_\mathbf{B}^c\setminus\mathcal{T}_\mathbf{B}$ and $\mathcal{S}_{\boldsymbol{\Gamma}}^c\setminus\mathcal{T}_{\boldsymbol{\Gamma}}$, any nonzero value would create a cycle in the union graph and hence violate feasibility of $(\mathbf{B},\boldsymbol{\Gamma})\in\mathcal{F}$ on the event 
that all true edges are present (i.e., the no false negatives event established 
in Part~2). It therefore suffices to establish dual feasibility on $\mathcal{T}_\mathbf{B}$ and $\mathcal{T}_{\boldsymbol{\Gamma}}$. Specifically, we show the following inequalities with probability tending to one:
\begin{itemize}
\item For every $(j,k)\in\mathcal{T}_\mathbf{B}$, it holds
\begin{equation}
\label{eq:gradbetabound}
\left|\frac{\partial}{\partial\beta_{jk}}\ell_m(\tilde{\mathbf{B}},\tilde{\boldsymbol{\Gamma}})\right|\leq\left(1-\frac{\eta}{2}\right)\lambda_1;
\end{equation}
and
\item For every $(j,k)\in\mathcal{T}_{\boldsymbol{\Gamma}}$, it holds
\begin{equation}
\label{eq:gradgammabound}
\frac{\partial}{\partial\gamma_{jk}}\ell_m(\tilde{\mathbf{B}},\tilde{\boldsymbol{\Gamma}})>-\left(1-\frac{\eta}{2}\right)\lambda_2.
\end{equation}
\end{itemize}
These inequalities ensure strict dual feasibility for all inactive coordinates that admit feasible perturbations.

We begin by performing a Taylor expansion on the nonzero entries. Since $\boldsymbol{\Gamma}\in\mathbb{R}_+^{p\times p}$ on $\mathcal{F}$, each covariance matrix $\mathbf{V}^{(i)}(\boldsymbol{\gamma}_k)$ is positive definite, so $\log\det\{\mathbf{V}^{(i)}(\boldsymbol{\gamma}_k)\}$ and $\mathbf{V}^{-(i)}(\boldsymbol{\gamma}_k)$ are $C^2$ (indeed $C^\infty$) on an open neighborhood of $\mathcal{F}$. Hence, $\ell_m(\mathbf{B},\boldsymbol{\Gamma})$ is $C^2$ on an open neighborhood of $\mathcal{F}$, and we may expand $\nabla_\mathcal{S}\ell_m(\mathbf{B},\boldsymbol{\Gamma})$ around $(\mathbf{B}_0,\boldsymbol{\Gamma}_0)$ as
\begin{equation}
\label{eq:expansion1}
\nabla_\mathcal{S}\ell_m(\tilde{\mathbf{B}},\tilde{\boldsymbol{\Gamma}})=\nabla_\mathcal{S}\ell_m(\mathbf{B}_0,\boldsymbol{\Gamma}_0)+\mathbf{H}_{\mathcal{S}\mathcal{S}}^{(m)}
\begin{pmatrix}
\tilde{\mathbf{B}}_{\mathcal{S}_\mathbf{B}}-\mathbf{B}_{\mathcal{S}_\mathbf{B},0} \\
\tilde{\boldsymbol{\Gamma}}_{\mathcal{S}_{\boldsymbol{\Gamma}}}-\boldsymbol{\Gamma}_{\mathcal{S}_{\boldsymbol{\Gamma}},0}
\end{pmatrix},
\end{equation}
where $\mathbf{H}_{\mathcal{S}\mathcal{S}}^{(m)}$ is the mean Hessian along the line segment between $(\mathbf{B}_0,\boldsymbol{\Gamma}_0)$ and $(\tilde{\mathbf{B}},\tilde{\boldsymbol{\Gamma}})$ restricted to the nonzero components. Combining \eqref{eq:foc} and \eqref{eq:expansion1} gives
\begin{equation*}
\mathbf{H}_{\mathcal{S}\mathcal{S}}^{(m)}
\begin{pmatrix}
\tilde{\mathbf{B}}_{\mathcal{S}_\mathbf{B}}-\mathbf{B}_{\mathcal{S}_\mathbf{B},0} \\
\tilde{\boldsymbol{\Gamma}}_{\mathcal{S}_{\boldsymbol{\Gamma}}}-\boldsymbol{\Gamma}_{\mathcal{S}_{\boldsymbol{\Gamma}},0}
\end{pmatrix}
=-\nabla_\mathcal{S}\ell_m(\mathbf{B}_0,\boldsymbol{\Gamma}_0)-\lambda_1
\begin{pmatrix}
\mathbf{s}_\mathbf{B} \\
\lambda_2/\lambda_1\mathbf{1}
\end{pmatrix}.
\end{equation*}
Since $\mathcal{F}_{\mathcal{S}}$ is compact, Lemma~\ref{lemma:uniformconvergence} gives
\begin{equation*}
\sup_{(\mathbf{B},\boldsymbol{\Gamma})\in\mathcal{F}_{\mathcal{S}}} \left\|\nabla^2\ell_m(\mathbf{B},\boldsymbol{\Gamma})-\nabla^2L(\mathbf{B},\boldsymbol{\Gamma})\right\|_{\mathrm{op}}\overset{p}{\to}0.
\end{equation*}
Moreover, the same argument as in Part~1, but with minimization carried out over $\mathcal{F}_{\mathcal{S}}$, gives $(\tilde{\mathbf{B}},\tilde{\boldsymbol{\Gamma}})\overset{p}{\to}(\mathbf{B}_0,\boldsymbol{\Gamma}_0)$. Together with continuity of $\nabla^2 L(\mathbf{B},\boldsymbol{\Gamma})$, these two results imply that $\mathbf{H}_{\mathcal{S}\mathcal{S}}^{(m)}\overset{p}{\to}\mathbf{H}_{\mathcal{S}\mathcal{S}}$. Since $\mathbf{H}_{\mathcal{S}\mathcal{S}}$ is nonsingular (because $\mathbf{H}$ is positive definite by Assumption~\ref{asmp:hessian}), $\mathbf{H}_{\mathcal{S}\mathcal{S}}^{(m)}$ is invertible with probability tending to one, and so
\begin{equation*}
\begin{pmatrix}
\tilde{\mathbf{B}}_{\mathcal{S}_\mathbf{B}}-\mathbf{B}_{\mathcal{S}_\mathbf{B},0} \\
\tilde{\boldsymbol{\Gamma}}_{\mathcal{S}_{\boldsymbol{\Gamma}}}-\boldsymbol{\Gamma}_{\mathcal{S}_{\boldsymbol{\Gamma}},0}
\end{pmatrix}
=-\mathbf{H}_{\mathcal{S}\mathcal{S}}^{-(m)}\left[\nabla_\mathcal{S}\ell_m(\mathbf{B}_0,\boldsymbol{\Gamma}_0)+\lambda_1
\begin{pmatrix}
\mathbf{s}_\mathbf{B} \\
\lambda_2/\lambda_1\mathbf{1}
\end{pmatrix}
\right].
\end{equation*}
Because the clusters are iid and satisfy $n_i\equiv n$, each coordinate of $\nabla\ell_m(\mathbf{B}_0,\boldsymbol{\Gamma}_0)$ can be written as an average of iid mean-zero clusterwise contributions with finite variance under the Gaussian structural equation model. Therefore, by Chebyshev's inequality, $\nabla_\mathcal{S}\ell_m(\mathbf{B}_0,\boldsymbol{\Gamma}_0)=O_p(m^{-1/2})$ coordinatewise. Further, $\mathbf{H}_{\mathcal{S}\mathcal{S}}^{-(m)}=O_p(1)$ and, by Assumption~\ref{asmp:penalty}, $m^{-1/2}=o(\lambda_1)$. Hence, it holds
\begin{equation}
\label{eq:expansion2}
\begin{pmatrix}
\tilde{\mathbf{B}}_{\mathcal{S}_\mathbf{B}}-\mathbf{B}_{\mathcal{S}_\mathbf{B},0} \\
\tilde{\boldsymbol{\Gamma}}_{\mathcal{S}_{\boldsymbol{\Gamma}}}-\boldsymbol{\Gamma}_{\mathcal{S}_{\boldsymbol{\Gamma}},0}
\end{pmatrix}
=-\lambda_1\mathbf{H}_{\mathcal{S}\mathcal{S}}^{-(m)}
\begin{pmatrix}
\mathbf{s}_\mathbf{B} \\
\lambda_2/\lambda_1\mathbf{1}
\end{pmatrix}
+o_p(\lambda_1).
\end{equation}

We now perform a Taylor expansion on the zero entries. Take a zero fixed-effect entry $(j,k)\in\mathcal{T}_\mathbf{B}$. Expand its partial derivative around $(\mathbf{B}_0,\boldsymbol{\Gamma}_0)$ as
\begin{equation}
\label{eq:expansion3}
\frac{\partial}{\partial\beta_{jk}}\ell_m(\tilde{\mathbf{B}},\tilde{\boldsymbol{\Gamma}})=\frac{\partial}{\partial\beta_{jk}}\ell_m(\mathbf{B}_0,\boldsymbol{\Gamma}_0)+\mathbf{h}_{jk,\mathcal{S}}^{(m,\mathbf{B})}
\begin{pmatrix}
\tilde{\mathbf{B}}_{\mathcal{S}_\mathbf{B}}-\mathbf{B}_{\mathcal{S}_\mathbf{B},0} \\
\tilde{\boldsymbol{\Gamma}}_{\mathcal{S}_{\boldsymbol{\Gamma}}}-\boldsymbol{\Gamma}_{\mathcal{S}_{\boldsymbol{\Gamma}},0}
\end{pmatrix},
\end{equation}
where $\mathbf{h}_{jk,\mathcal{S}}^{(m,\mathbf{B})}$ is the mean cross-Hessian row between $\beta_{jk}$ and the active entries. Substitute \eqref{eq:expansion2} into \eqref{eq:expansion3} to get
\begin{equation*}
\frac{\partial}{\partial\beta_{jk}}\ell_m(\tilde{\mathbf{B}},\tilde{\boldsymbol{\Gamma}})=\frac{\partial}{\partial\beta_{jk}}\ell_m(\mathbf{B}_0,\boldsymbol{\Gamma}_0)-\lambda_1\mathbf{h}_{jk,\mathcal{S}}^{(m,\mathbf{B})}
\mathbf{H}_{\mathcal{S}\mathcal{S}}^{-(m)}
\begin{pmatrix}
\mathbf{s}_\mathbf{B} \\
\lambda_2/\lambda_1\mathbf{1}
\end{pmatrix}
+o_p(\lambda_1).
\end{equation*}
The first term on the right-hand side can be written as
\begin{equation*}
\frac{\partial}{\partial\beta_{jk}}\ell_m(\mathbf{B}_0,\boldsymbol{\Gamma}_0)=\frac{\partial}{\partial\beta_{jk}}L(\mathbf{B}_0,\boldsymbol{\Gamma}_0)+\left\{\frac{\partial}{\partial\beta_{jk}}\ell_m(\mathbf{B}_0,\boldsymbol{\Gamma}_0)-\frac{\partial}{\partial\beta_{jk}}L(\mathbf{B}_0,\boldsymbol{\Gamma}_0)\right\}.
\end{equation*}
Since $(\mathbf{B}_0,\boldsymbol{\Gamma}_0)$ uniquely minimizes $L$ over $\mathcal{F}$ by Lemma~\ref{lemma:unique}, the population derivative vanishes along any direction that preserves acyclicity of the union graph. The term in brackets is an average of iid mean-zero clusterwise contributions with finite variance and is therefore $O_p(m^{-1/2})$ by Chebyshev's inequality. Hence, we have
\begin{equation*}
\frac{\partial}{\partial\beta_{jk}}\ell_m(\mathbf{B}_0,\boldsymbol{\Gamma}_0)=O_p(m^{-1/2})=o_p(\lambda_1),
\end{equation*}
where the final equality follows from Assumption~\ref{asmp:penalty}. Also $\lambda_2/\lambda_1\to\kappa$ by Assumption~\ref{asmp:penalty}. Then, since $\mathbf{h}_{jk,\mathcal{S}}^{(m,\mathbf{B})}\overset{p}{\to}\mathbf{h}_{jk,\mathcal{S}}^{(\mathbf{B})}$, we have
\begin{equation*}
\mathbf{h}_{jk,\mathcal{S}}^{(m,\mathbf{B})}\mathbf{H}_{\mathcal{S}\mathcal{S}}^{-(m)}
\begin{pmatrix}
\mathbf{s}_\mathbf{B} \\
\lambda_2/\lambda_1\mathbf{1}
\end{pmatrix}
\overset{p}{\to}\mathbf{h}_{jk,\mathcal{S}}^{(\mathbf{B})}\mathbf{H}_{\mathcal{S}\mathcal{S}}^{-1}
\begin{pmatrix}
\mathbf{s}_\mathbf{B} \\
\kappa\mathbf{1}
\end{pmatrix}.
\end{equation*}
By Assumption~\ref{asmp:hessian}, the right-hand side has absolute value at most $1-\eta$, therefore
\begin{equation*}
\mathrm{Pr}\left(\left|\frac{\partial}{\partial\beta_{jk}}\ell_m(\tilde{\mathbf{B}},\tilde{\boldsymbol{\Gamma}})\right|\leq\left(1-\frac{\eta}{2}\right)\lambda_1\right)\to1
\end{equation*}
for each $(j,k)\in\mathcal{T}_\mathbf{B}$. Since $p$ is fixed, $\mathcal{T}_\mathbf{B}$ is finite, so the above convergence holds uniformly over all such $(j,k)\in\mathcal{T}_\mathbf{B}$ by the union bound, proving \eqref{eq:gradbetabound}.

The same line of argument as above applies to a zero random-effect variance entry $(j,k)\in\mathcal{T}_{\boldsymbol{\Gamma}}$, yielding
\begin{equation*}
\frac{\partial}{\partial\gamma_{jk}}\ell_m(\tilde{\mathbf{B}},\tilde{\boldsymbol{\Gamma}})=\frac{\partial}{\partial\gamma_{jk}}\ell_m(\mathbf{B}_0,\boldsymbol{\Gamma}_0)-\lambda_1\mathbf{h}_{jk,\mathcal{S}}^{(m,\boldsymbol{\Gamma})}
\mathbf{H}_{\mathcal{S}\mathcal{S}}^{-(m)}
\begin{pmatrix}
\mathbf{s}_\mathbf{B} \\
\lambda_2/\lambda_1\mathbf{1}
\end{pmatrix}
+o_p(\lambda_1).
\end{equation*}
The first term on the right-hand side can be written as
\begin{equation*}
\frac{\partial}{\partial\gamma_{jk}}\ell_m(\mathbf{B}_0,\boldsymbol{\Gamma}_0)=\frac{\partial}{\partial\gamma_{jk}}L(\mathbf{B}_0,\boldsymbol{\Gamma}_0)+\left\{\frac{\partial}{\partial\gamma_{jk}}\ell_m(\mathbf{B}_0,\boldsymbol{\Gamma}_0)-\frac{\partial}{\partial\gamma_{jk}}L(\mathbf{B}_0,\boldsymbol{\Gamma}_0)\right\}.
\end{equation*}
Since $(\mathbf{B}_0,\boldsymbol{\Gamma}_0)$ uniquely minimizes $L$ over $\mathcal{F}$ by Lemma~\ref{lemma:unique}, the right derivative of $L$ is nonnegative along any direction that preserves acyclicity of the union graph. The term in brackets is an average of iid clusterwise contributions with finite variance and is therefore $O_p(m^{-1/2})$ by Chebyshev's inequality. Hence, we have
\begin{equation*}
\frac{\partial}{\partial\gamma_{jk}}\ell_m(\mathbf{B}_0,\boldsymbol{\Gamma}_0)=\frac{\partial}{\partial\gamma_{jk}}L(\mathbf{B}_0,\boldsymbol{\Gamma}_0)+O_p(m^{-1/2})\geq-o_p(\lambda_2),
\end{equation*}
where the final inequality follows by Assumption~\ref{asmp:penalty}. Then, since $\mathbf{h}_{jk,\mathcal{S}}^{(m,\boldsymbol{\Gamma})}\overset{p}{\to}\mathbf{h}_{jk,\mathcal{S}}^{(\boldsymbol{\Gamma})}$, we have
\begin{equation*}
\mathbf{h}_{jk,\mathcal{S}}^{(m,\boldsymbol{\Gamma})}
\mathbf{H}_{\mathcal{S}\mathcal{S}}^{-(m)}
\begin{pmatrix}
\mathbf{s}_\mathbf{B} \\
\lambda_2/\lambda_1\mathbf{1}
\end{pmatrix}
\overset{p}{\to}
\mathbf{h}_{jk,\mathcal{S}}^{(\boldsymbol{\Gamma})}
\mathbf{H}_{\mathcal{S}\mathcal{S}}^{-1}
\begin{pmatrix}
\mathbf{s}_\mathbf{B} \\
\kappa\mathbf{1}
\end{pmatrix}.
\end{equation*}
Again, Assumption~\ref{asmp:hessian} upper bounds the right-hand side by $\kappa(1-\eta)$, therefore
\begin{equation*}
\mathrm{Pr}\left(\frac{\partial}{\partial\gamma_{jk}}\ell_m(\tilde{\mathbf{B}},\tilde{\boldsymbol{\Gamma}})>-\left(1-\frac{\eta}{2}\right)\lambda_2\right)\to1
\end{equation*}
for each $(j,k)\in\mathcal{T}_{\boldsymbol{\Gamma}}$. As before, the above convergence holds uniformly over all $(j,k)\in\mathcal{T}_{\boldsymbol{\Gamma}}$ by the union bound, proving \eqref{eq:gradgammabound}.

\subparagraph{Part 3(c)}

By Part~1, $(\hat{\mathbf{B}},\hat{\boldsymbol{\Gamma}})\overset{p}{\to}(\mathbf{B}_0,\boldsymbol{\Gamma}_0)$, and as explained in Part~3(b), $(\tilde{\mathbf{B}},\tilde{\boldsymbol{\Gamma}})\overset{p}{\to}(\mathbf{B}_0,\boldsymbol{\Gamma}_0)$. Since $(\mathbf{B}_0,\boldsymbol{\Gamma}_0)$ is the unique minimizer of $L$ over $\mathcal{F}$ by Lemma~\ref{lemma:unique}, it is also the unique minimizer of $L$ over the subset $\mathcal{F}_{\mathcal{S}}$. Let $\delta>0$ and $c>0$ be the constants from Lemma~\ref{lemma:localstrongconvexity}. Hence, with probability tending to one, we have
\begin{equation*}
\|(\hat{\mathbf{B}},\hat{\boldsymbol{\Gamma}})-(\mathbf{B}_0,\boldsymbol{\Gamma}_0)\|_F\leq\delta\quad\text{and}\quad\|(\tilde{\mathbf{B}},\tilde{\boldsymbol{\Gamma}})-(\mathbf{B}_0,\boldsymbol{\Gamma}_0)\|_F\leq\delta.
\end{equation*}
On this event, Lemma~\ref{lemma:localstrongconvexity} with $(\mathbf{B}_1,\boldsymbol{\Gamma}_1)=(\hat{\mathbf{B}},\hat{\boldsymbol{\Gamma}})$ and $(\mathbf{B}_2,\boldsymbol{\Gamma}_2)=(\tilde{\mathbf{B}},\tilde{\boldsymbol{\Gamma}})$ yields
\begin{equation*}
\ell_m(\hat{\mathbf{B}},\hat{\boldsymbol{\Gamma}})-\ell_m(\tilde{\mathbf{B}},\tilde{\boldsymbol{\Gamma}})\geq\left\langle\nabla\ell_m(\tilde{\mathbf{B}},\tilde{\boldsymbol{\Gamma}}),(\hat{\mathbf{B}},\hat{\boldsymbol{\Gamma}})-(\tilde{\mathbf{B}},\tilde{\boldsymbol{\Gamma}})\right\rangle+\frac{c}{2}\|(\hat{\mathbf{B}},\hat{\boldsymbol{\Gamma}})-(\tilde{\mathbf{B}},\tilde{\boldsymbol{\Gamma}})\|_F^2.
\end{equation*}
Adding the difference of $\ell_1$ penalties to both sides gives
\begin{equation*}
\begin{split}
F_m(\hat{\mathbf{B}},\hat{\boldsymbol{\Gamma}})-F_m(\tilde{\mathbf{B}},\tilde{\boldsymbol{\Gamma}})&\geq\sum_{(j,k)\in\mathcal{S}_\mathbf{B}}
\left[\frac{\partial}{\partial\beta_{jk}}\ell_m(\tilde{\mathbf{B}},\tilde{\boldsymbol{\Gamma}})
(\hat{\beta}_{jk}-\tilde{\beta}_{jk})+\lambda_1(|\hat{\beta}_{jk}|-|\tilde{\beta}_{jk}|)
\right] \\
&\quad+\sum_{(j,k)\in\mathcal{S}_\mathbf{B}^c}\left[\frac{\partial}{\partial\beta_{jk}}\ell_m(\tilde{\mathbf{B}},\tilde{\boldsymbol{\Gamma}})\hat{\beta}_{jk}+\lambda_1|\hat{\beta}_{jk}|\right] \\
&\quad+\sum_{(j,k)\in\mathcal{S}_{\boldsymbol{\Gamma}}}\left[\frac{\partial}{\partial\gamma_{jk}}\ell_m(\tilde{\mathbf{B}},\tilde{\boldsymbol{\Gamma}})(\hat{\gamma}_{jk}-\tilde{\gamma}_{jk})+\lambda_2(\hat{\gamma}_{jk}-\tilde{\gamma}_{jk})\right] \\
&\quad+\sum_{(j,k)\in\mathcal{S}_{\boldsymbol{\Gamma}}^c}\left[\frac{\partial}{\partial\gamma_{jk}}\ell_m(\tilde{\mathbf{B}},\tilde{\boldsymbol{\Gamma}})\hat{\gamma}_{jk}+\lambda_2\hat{\gamma}_{jk}\right] \\
&\quad+\frac{c}{2}\|(\hat{\mathbf{B}},\hat{\boldsymbol{\Gamma}})-(\tilde{\mathbf{B}},\tilde{\boldsymbol{\Gamma}})\|_F^2.
\end{split}
\end{equation*}
By \eqref{eq:focbeta} and the inequality
\begin{equation*}
|a|-|b|\geq\operatorname{sign}(b)(a-b),
\end{equation*}
the sum over $\mathcal{S}_\mathbf{B}$ is nonnegative. By \eqref{eq:focgamma}, the sum over $\mathcal{S}_{\boldsymbol{\Gamma}}$ is equal to zero. By \eqref{eq:gradbetabound}, we have
\begin{equation*}
\frac{\partial}{\partial\beta_{jk}}\ell_m(\tilde{\mathbf{B}},\tilde{\boldsymbol{\Gamma}})\hat{\beta}_{jk}+\lambda_1|\hat{\beta}_{jk}|\geq\frac{\eta}{2}\lambda_1|\hat{\beta}_{jk}|
\end{equation*}
for all $(j,k)\in\mathcal{T}_\mathbf{B}$. Likewise, since $\hat{\gamma}_{jk}\geq0$ and \eqref{eq:gradgammabound} holds, we have
\begin{equation*}
\frac{\partial}{\partial\gamma_{jk}}\ell_m(\tilde{\mathbf{B}},\tilde{\boldsymbol{\Gamma}})\hat{\gamma}_{jk}+\lambda_2\hat{\gamma}_{jk}\geq\frac{\eta}{2}\lambda_2\hat{\gamma}_{jk}
\end{equation*}
for all $(j,k)\in\mathcal{T}_{\boldsymbol{\Gamma}}$. On the event from Part~2 (no false negatives), feasibility of $(\hat{\mathbf{B}},\hat{\boldsymbol{\Gamma}})\in\mathcal{F}$ implies
\begin{equation*}
\hat{\beta}_{jk}=0\,\forall\,(j,k)\in\mathcal{S}_\mathbf{B}^c\setminus\mathcal{T}_\mathbf{B}\quad\text{and}\quad\hat{\gamma}_{jk}=0\,\forall\,(j,k)\in\mathcal{S}_{\boldsymbol{\Gamma}}^c\setminus\mathcal{T}_{\boldsymbol{\Gamma}},
\end{equation*}
since turning on any such coordinate would create a directed cycle in the union graph. It follows
\begin{equation*}
F_m(\hat{\mathbf{B}},\hat{\boldsymbol{\Gamma}})-F_m(\tilde{\mathbf{B}},\tilde{\boldsymbol{\Gamma}})\geq\frac{\eta}{2}\lambda_1\sum_{(j,k)\in\mathcal{T}_\mathbf{B}}|\hat{\beta}_{jk}|+\frac{\eta}{2}\lambda_2\sum_{(j,k)\in\mathcal{T}_{\boldsymbol{\Gamma}}}\hat{\gamma}_{jk}+\frac{c}{2}\|(\hat{\mathbf{B}},\hat{\boldsymbol{\Gamma}})-(\tilde{\mathbf{B}},\tilde{\boldsymbol{\Gamma}})\|_F^2.
\end{equation*}
Since $(\hat{\mathbf{B}},\hat{\boldsymbol{\Gamma}})$ is the global minimizer of $F_m$ over $\mathcal{F}$ and $(\tilde{\mathbf{B}},\tilde{\boldsymbol{\Gamma}})\in\mathcal{F}_{\mathcal{S}}\subseteq\mathcal{F}$, it holds
\begin{equation*}
F_m(\hat{\mathbf{B}},\hat{\boldsymbol{\Gamma}})\leq F_m(\tilde{\mathbf{B}},\tilde{\boldsymbol{\Gamma}}).
\end{equation*}
Combining the previous inequalities, we conclude
\begin{equation*}
\sum_{(j,k)\in\mathcal{T}_\mathbf{B}}|\hat{\beta}_{jk}|=0\quad\text{and}\quad\sum_{(j,k)\in\mathcal{T}_{\boldsymbol{\Gamma}}}\hat{\gamma}_{jk}=0,
\end{equation*}
with probability tending to one. Together with the feasibility implications above, this result gives
\begin{equation*}
\hat{\beta}_{jk}=0\,\forall\,(j,k)\in\mathcal{S}_\mathbf{B}^c\quad\text{and}\quad\hat{\gamma}_{jk}=0\,\forall\,(j,k)\in\mathcal{S}_{\boldsymbol{\Gamma}}^c,
\end{equation*}
proving $(\hat{\mathbf{B}},\hat{\boldsymbol{\Gamma}})$ contains no false positives.

\paragraph{Part 4: Consistency of structure recovery} 

Part~2 guarantees $(\hat{\mathbf{B}},\hat{\boldsymbol{\Gamma}})$ contain no false negatives. Part~3 guarantees $(\hat{\mathbf{B}},\hat{\boldsymbol{\Gamma}})$ contain no false positives. It immediately follows
\begin{equation*}
\mathrm{Pr}\left(\mathcal{G}(\hat{\mathbf{B}})=\mathcal{G}(\mathbf{B}_0)\enspace\text{and}\enspace\mathcal{G}(\hat{\boldsymbol{\Gamma}})=\mathcal{G}(\boldsymbol{\Gamma}_0)\right)\to1,
\end{equation*}
thereby proving the second claim of the theorem.

\end{proof}

\section{Implementation Details}
\label{app:implementation}

\subsection{Sparsity parameters}

The sparsity parameters $\lambda_1$ and $\lambda_2$ are taken equal and swept over a grid of 30 values equally spaced on a logarithmic grid between 1 and $10^{-3}$. As with all differentiable approaches, the weighted adjacency matrices produced by \texttt{MixDAG} can contain small nonzero values that violate acyclicity. These violations arise because gradient descent does not produce exact zeros when finite stopping rules are used. We thus follow the common practice of hard thresholding \citep[e.g.,][]{Zheng2018,Bello2022} and set any elements of $\mathbf{B}$ and $\boldsymbol{\Gamma}$ that are below 0.05 in absolute value to zero. All baselines are thresholded similarly.

\subsection{Polishing}

A drawback to $\ell_1$-induced sparsity is that it biases the estimates of $\mathbf{B}$ and $\boldsymbol{\Gamma}$ towards zero. This bias is a consequence of the shrinkage inherent to the $\ell_1$ norm and is a well-documented side effect of $\ell_1$ penalties \citep{Hastie2020}. We correct for this bias in all baselines through a ``polishing'' step. Specifically, after fitting the model, we retain the selected edges and reoptimize over this fixed structure, removing the $\ell_1$ penalties (and the DAG constraint since the solution is already acyclic). Let $\mathbf{M}^\mathbf{B},\mathbf{M}^{\boldsymbol{\Gamma}}\in\{0,1\}^{p\times p}$ denote binary masks with entries
\begin{equation*}
M^\mathbf{B}_{jk}:=1(\hat{\beta}_{jk}\neq0),\qquad M^{\boldsymbol{\Gamma}}_{jk}:=1(\hat{\gamma}_{jk}\neq0).
\end{equation*}
We then solve
\begin{equation*}
\underset{\mathbf{B}\in\mathbb{R}^{p\times p},\,\boldsymbol{\Gamma}\in\mathbb{R}_+^{p\times p}}{\min}\;l(\mathbf{M}^\mathbf{B}\odot\mathbf{B},\mathbf{M}^{\boldsymbol{\Gamma}}\odot\boldsymbol{\Gamma}).
\end{equation*}
This polishing step is typically fast, as it optimizes over a reduced parameter space and involves the negative log-likelihood only and not the acyclicity constraint or $\ell_1$ penalties.

\subsection{Computational Environment}

The experiments are run on a Linux machine with eight NVIDIA RTX 4090 graphics cards. Each experimental run used a single graphics card, with multiple runs performed in parallel. Our implementations are in Python and built on PyTorch \citep{Paszke2019}.

\section{Additional Experiments}
\label{app:additional}

We assess the robustness of \texttt{MixDAG} to a range of alternative simulation designs. In particular, we vary the graph sparsity, the graph type, and the cluster regime. Figures~\ref{fig:synthetic_50_100_inf+inf_erdos_renyi} and \ref{fig:synthetic_50_150_inf+inf_erdos_renyi} show results for denser graphs with $s=2p$ and $s=3p$ edges, respectively, rather than the $s=p$ edges used in the main experiments.
\begin{figure}[t]
\centering
\includegraphics[width=0.8\textwidth]{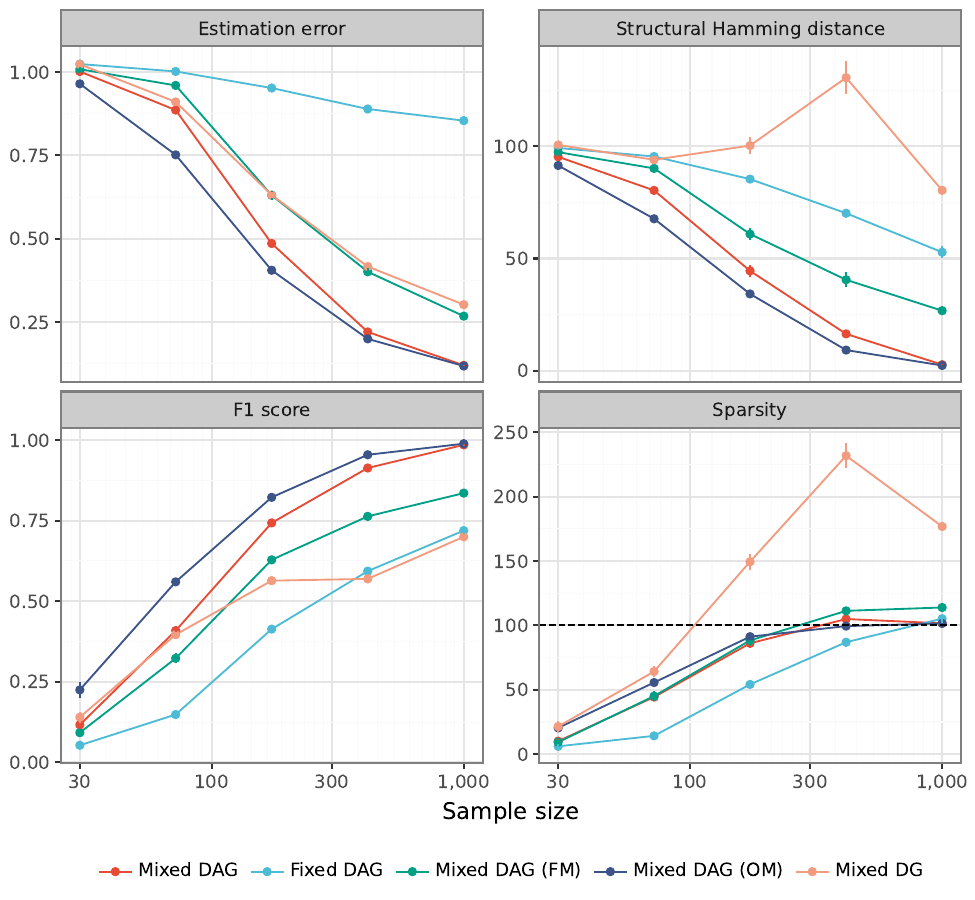}
\caption{Performance on synthetic data generated from Erdős--Rényi DAGs with $p=50$ nodes, $s=100$ edges, and $m=\lceil\sqrt{\ndot}\rceil$ clusters. Averages (solid points) and standard errors (error bars) are measured over 30 datasets. The dashed horizontal line in the bottom right panel indicates the number of edges in the true DAG.}
\label{fig:synthetic_50_100_inf+inf_erdos_renyi}
\end{figure}
\begin{figure}[t]
\centering
\includegraphics[width=0.8\textwidth]{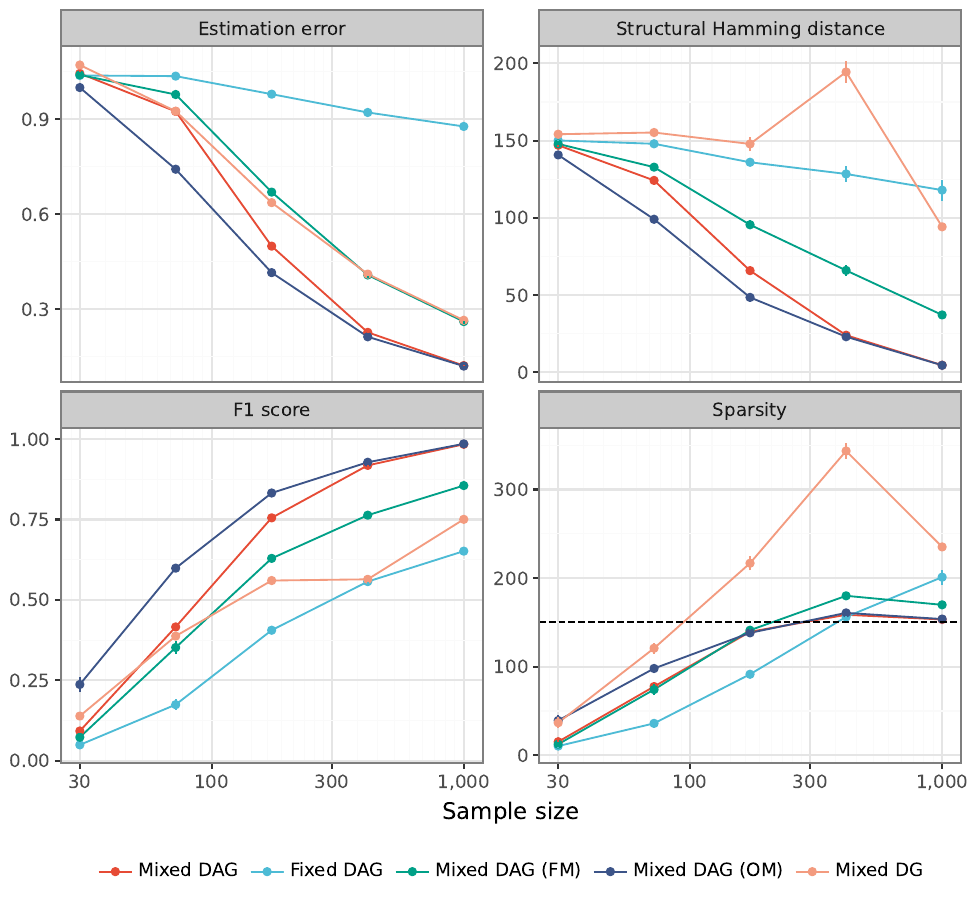}
\caption{Performance on synthetic data generated from Erdős--Rényi DAGs with $p=50$ nodes, $s=150$ edges, and $m=\lceil\sqrt{\ndot}\rceil$ clusters. Averages (solid points) and standard errors (error bars) are measured over 30 datasets. The dashed horizontal line in the bottom right panel indicates the number of edges in the true DAG.}
\label{fig:synthetic_50_150_inf+inf_erdos_renyi}
\end{figure}
Figures~\ref{fig:synthetic_50_50_inf+inf_scale_free_2} and \ref{fig:synthetic_50_50_inf+inf_scale_free_3} show results for scale-free graphs \citep{Barabasi1999} with degree exponents 2 and 3, respectively, in place of the Erdős--Rényi graphs used in the main experiments.
\begin{figure}[t]
\centering
\includegraphics[width=0.8\textwidth]{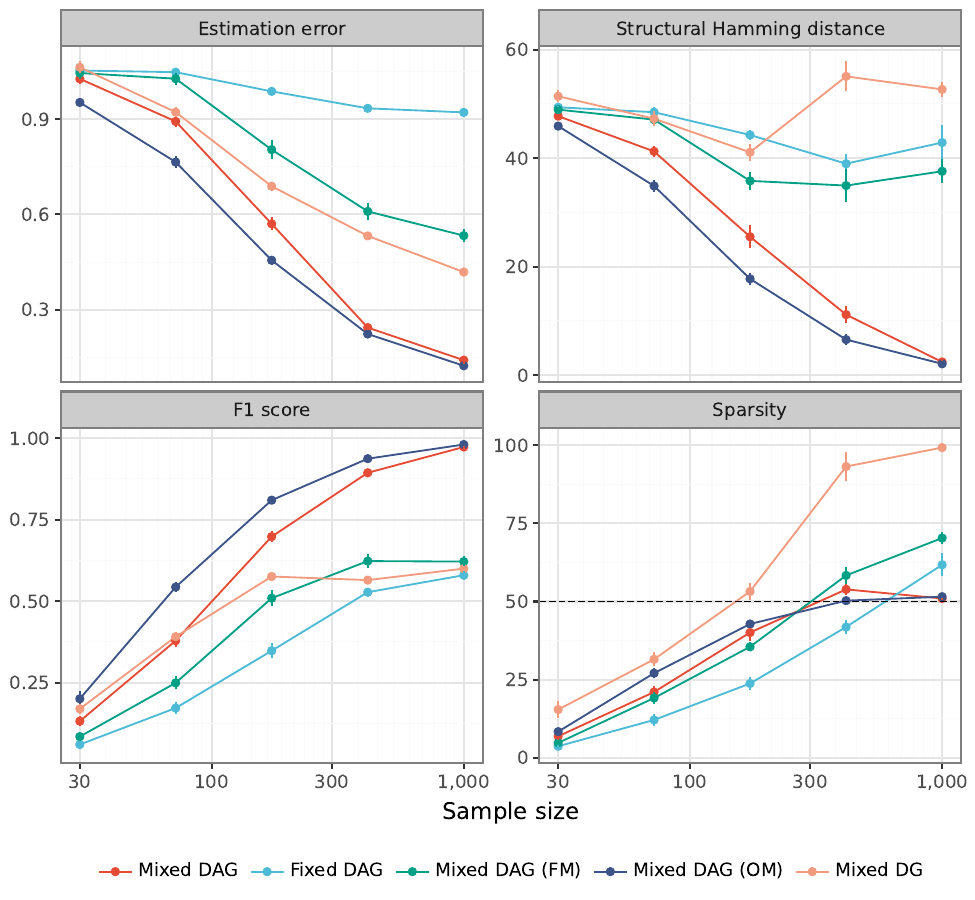}
\caption{Performance on synthetic data generated from scale-free (degree exponent 2) DAGs with $p=50$ nodes, $s=50$ edges, and $m=\lceil\sqrt{\ndot}\rceil$ clusters. Averages (solid points) and standard errors (error bars) are measured over 30 datasets. The dashed horizontal line in the bottom right panel indicates the number of edges in the true DAG.}
\label{fig:synthetic_50_50_inf+inf_scale_free_2}
\end{figure}
\begin{figure}[t]
\centering
\includegraphics[width=0.8\textwidth]{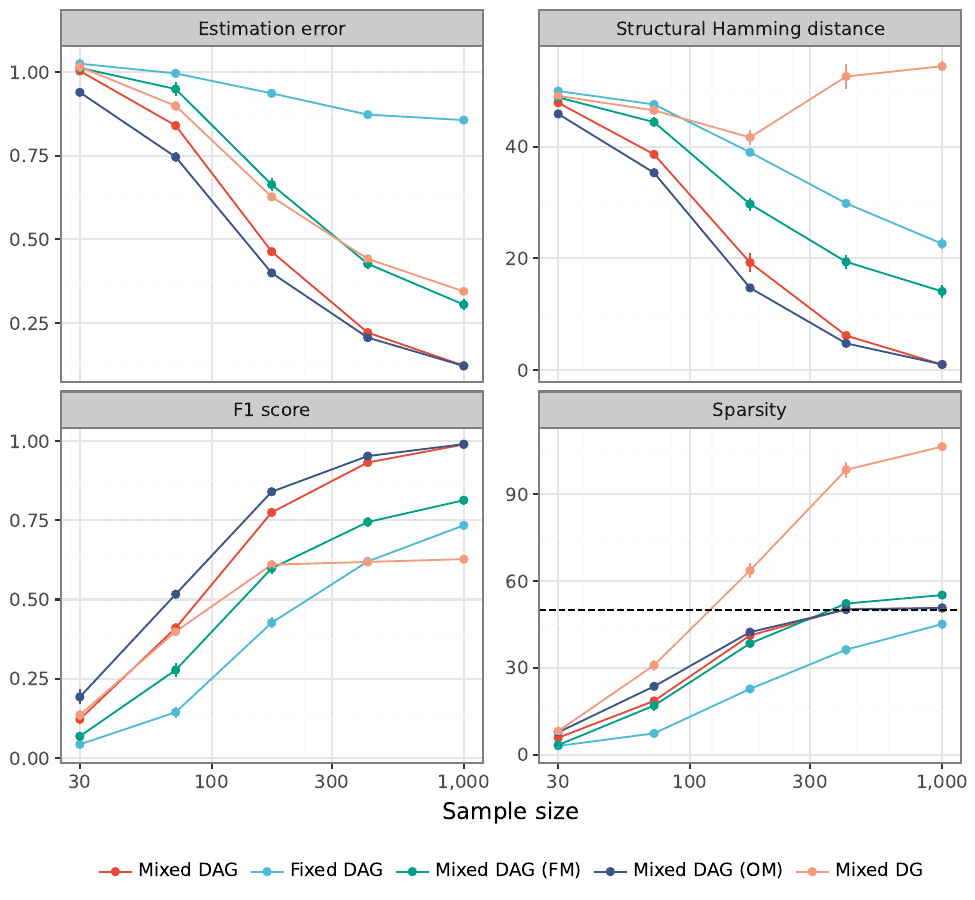}
\caption{Performance on synthetic data generated from scale-free (degree exponent 3) DAGs with $p=50$ nodes, $s=50$ edges, and $m=\lceil\sqrt{\ndot}\rceil$ clusters. Averages (solid points) and standard errors (error bars) are measured over 30 datasets. The dashed horizontal line in the bottom right panel indicates the number of edges in the true DAG.}
\label{fig:synthetic_50_50_inf+inf_scale_free_3}
\end{figure}
Finally, Figures~\ref{fig:synthetic_50_50_inf+fix_erdos_renyi} and \ref{fig:synthetic_50_50_fix+inf_erdos_renyi} show results for alternative cluster regimes with fixed average cluster size ($m=\lceil\ndot/10\rceil$) and fixed number of clusters ($m=10$), respectively, instead of the regime in the main experiments where both the number of clusters and cluster sizes increase with $\ndot$.
\begin{figure}[t]
\centering
\includegraphics[width=0.8\textwidth]{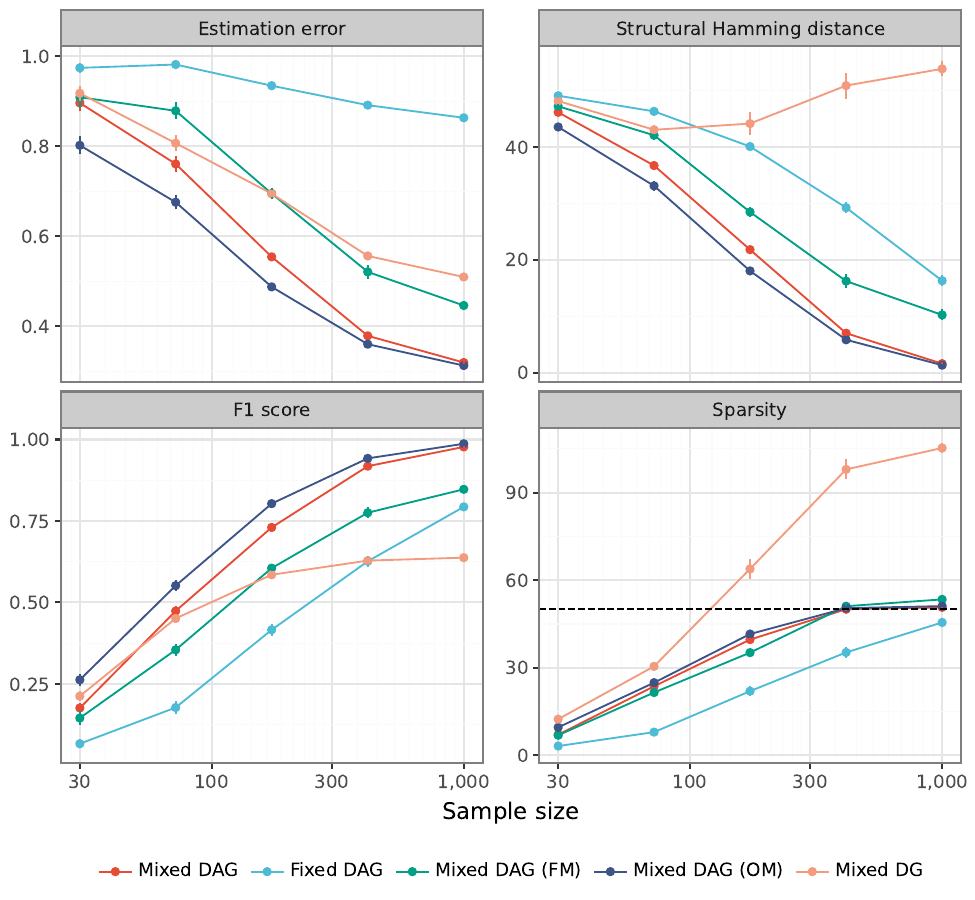}
\caption{Performance on synthetic data generated from Erdős--Rényi DAGs with $p=50$ nodes, $s=50$ edges, and $m=\lceil\ndot/10\rceil$ clusters. Averages (solid points) and standard errors (error bars) are measured over 30 datasets. The dashed horizontal line in the bottom right panel indicates the number of edges in the true DAG.}
\label{fig:synthetic_50_50_inf+fix_erdos_renyi}
\end{figure}
\begin{figure}[t]
\centering
\includegraphics[width=0.8\textwidth]{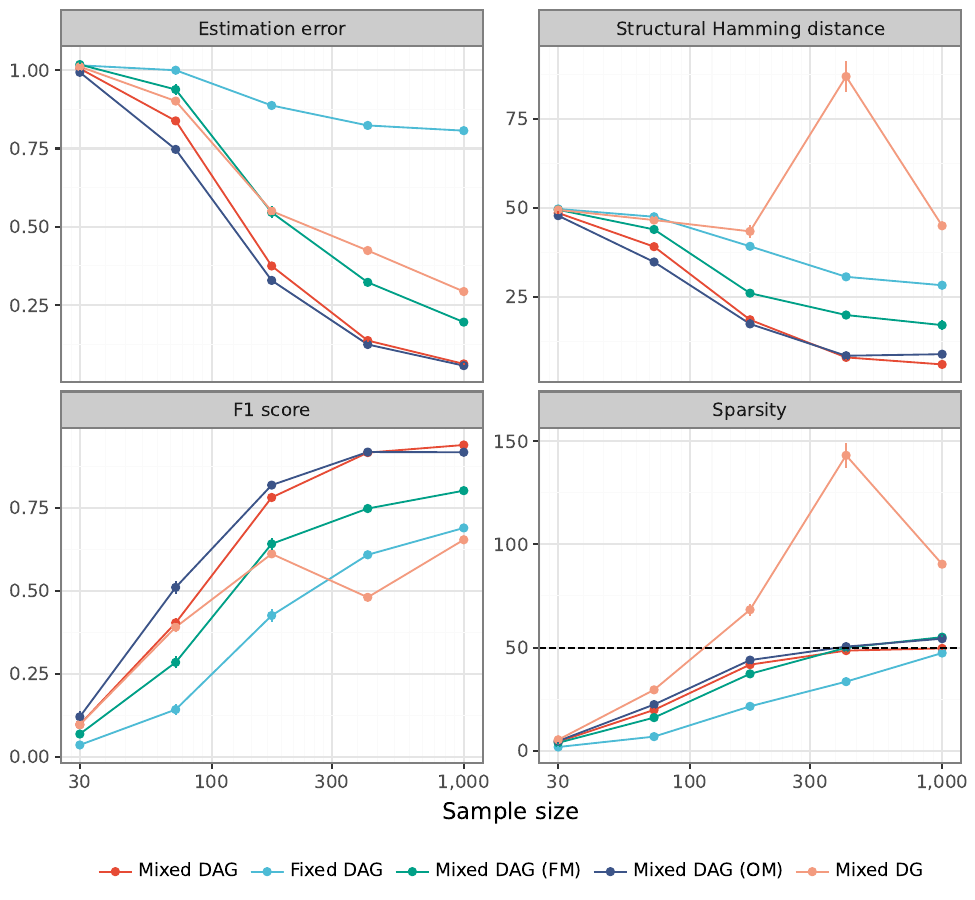}
\caption{Performance on synthetic data generated from Erdős--Rényi DAGs with $p=50$ nodes, $s=50$ edges, and $m=10$ clusters. Averages (solid points) and standard errors (error bars) are measured over 30 datasets. The dashed horizontal line in the bottom right panel indicates the number of edges in the true DAG.}
\label{fig:synthetic_50_50_fix+inf_erdos_renyi}
\end{figure}
Across all designs, \texttt{MixDAG} continues to perform well relative to the competing methods, indicating the gains in the main experiments are not specific to a particular graph sparsity, graph type, or cluster regime.

\clearpage

\bibliography{library}

\end{document}